\documentclass{article}
\usepackage[margin=1in]{geometry}

\usepackage{amsmath,graphicx,psfrag,verbatim}
\usepackage[small,bf]{caption}
\usepackage{amssymb}
\usepackage{float}

\usepackage[numbers,sort]{natbib}

\usepackage[x11names]{xcolor}
\usepackage{listings}
\usepackage{longtable}
\usepackage{subcaption}
\usepackage{mathtools}
\usepackage{adjustbox}
\usepackage{csquotes}

\usepackage{hyperref}

\definecolor{mydarkblue}{rgb}{0,0.08,0.45}
\hypersetup{%
    colorlinks = true,
    urlcolor   = mydarkblue,
    linkcolor  = mydarkblue,
    citecolor  = mydarkblue%
}

\usepackage{amsthm, thmtools, thm-restate}

\theoremstyle{plain}
\newtheorem{theorem}{Theorem}[section]
\newtheorem{proposition}[theorem]{Proposition}
\newtheorem{lemma}[theorem]{Lemma}

\theoremstyle{definition}

\theoremstyle{remark}
\newtheorem{remark}[theorem]{Remark}
\newtheorem*{remark*}{Remark}



\makeatother

\usepackage[capitalise]{cleveref}
\crefname{definition}{Definition}{Definitions}
\crefname{assumption}{Assumption}{Assumptions}
\usepackage{amsmath,amsfonts,amssymb}
\usepackage{bm,bbm,pifont}

\usepackage{xparse}
\DeclareFontFamily{U}{ntxmia}{}
\DeclareFontShape{U}{ntxmia}{m}{it}{<-> ntxmia }{}
\DeclareFontShape{U}{ntxmia}{b}{it}{<-> ntxbmia }{}
\DeclareSymbolFont{lettersA}{U}{ntxmia}{m}{it}
\SetSymbolFont{lettersA}{bold}{U}{ntxmia}{b}{it}

\ExplSyntaxOn
\NewDocumentCommand{\varmathbb}{m}
 {
  \tl_map_inline:nn { #1 }
   {
    \use:c { varbb##1 }
   }
 }
\cs_new_protected:Nn \__mathbb_define:Nn
 {
  \DeclareMathSymbol{#1}{\mathord}{lettersA}{#2}
 }
\cs_generate_variant:Nn \__mathbb_define:Nn {ce}
\tl_map_inline:nn { ABCDEFGHIJKLMNOPQRSTUVWXYZ }
 {
  \__mathbb_define:ce { varbb#1 } { \int_eval:n { `#1+67 } }
 }
\tl_map_inline:nn { abcdefghijklmnopqrstuvwxyz }
 {
  \__mathbb_define:ce { varbb#1 } { \int_eval:n { `#1+61 } }
 }
\DeclareMathSymbol{\varbbimath}{\mathord}{lettersA}{'270}
\DeclareMathSymbol{\varbbjmath}{\mathord}{lettersA}{'271}
\ExplSyntaxOff

\def\norm#1{\lVert#1\rVert}

\def\bignorm#1{\left\lVert#1\right\rVert}

\def\bigopen#1{\left(#1\right)}

\newcommand{\vertiii}[1]{{\left\vert\kern-0.25ex\left\vert\kern-0.25ex\left\vert #1 
    \right\vert\kern-0.25ex\right\vert\kern-0.25ex\right\vert}}

\newcommand{\PS}{\mathrm{PS}}
\newcommand{\PO}{\mathrm{PO}}
\newcommand{\PR}{\mathrm{PR}}
\newcommand{\Reg}{\mathrm{Reg}}

\newcommand{\Unif}{\mathrm{Unif}}

\newcommand{\ARGD}{\mathcal{A}_{\mathrm{RGD}}}
\newcommand{\ARRM}{\mathcal{A}_{\mathrm{RRM}}}
\newcommand{\ARSGDG}{\mathcal{A}_{\mathrm{SGD\text{-}greedy}}}
\newcommand{\ARSGDL}{\mathcal{A}_{\mathrm{SGD\text{-}lazy}}}
\newcommand{\AZGD}{\mathcal{A}_{\mathrm{ZGD}}}
\newcommand{\AZGDtwo}{\mathcal{A}_{\mathrm{ZGD}2}}

\def\1{\bm{1}}

\DeclareMathAlphabet{\mathsfit}{\encodingdefault}{\sfdefault}{m}{sl}
\SetMathAlphabet{\mathsfit}{bold}{\encodingdefault}{\sfdefault}{bx}{n}

\def\gA{{\mathcal{A}}}

\def\gC{{\mathcal{C}}}
\def\gD{{\mathcal{D}}}

\def\gG{{\mathcal{G}}}

\def\gN{{\mathcal{N}}}
\def\gO{{\mathcal{O}}}
\def\gP{{\mathcal{P}}}

\def\gW{{\mathcal{W}}}

\def\gZ{{\mathcal{Z}}}

\def\sB{{\mathbb{B}}}

\def\sN{{\mathbb{N}}}

\def\sS{{\mathbb{S}}}

\newcommand{\E}{\mathbb{E}}

\newcommand{\R}{\mathbb{R}}

\DeclareMathOperator*{\argmin}{arg\,min}

\newcommand{\RRM}{\mathrm{RRM}}

\usepackage{kpfonts}

\title{Partially Performative Prediction}
\author{Jaewook Lee%
\thanks{Department of Electrical Engineering, Stanford University
(\texttt{jwl99@stanford.edu}).}
\and Tijana Zrnic%
\thanks{Departments of Statistics and Management Science \& Engineering, Stanford University
(\texttt{tijana.zrnic@stanford.edu}).}
}
\date{%
June 5, 2026%
}

\begin{document}
\maketitle

\begin{abstract}
Performative prediction studies feedback loops that arise when predictive models are deployed in consequential domains. In these settings, deploying a model can change the population whose patterns the model aims to predict, inducing a distribution shift that is \emph{endogenous} to the learning system. This perspective departs from classical treatments of distribution shift, where shifts are typically modeled as \emph{exogenous} changes in the data-generating process. Yet, in practice, distribution shift is rarely one or the other. Predictive models may influence future data through the decisions they support, while the world itself continues to drift for reasons beyond the learner’s control.
We study \emph{partially performative prediction}, a framework that captures both endogenous and exogenous sources of distribution shift. The framework generalizes performative prediction by allowing the data distribution to evolve both in response to the deployed model and according to an external, time-varying process. We extend the central notions of performative stability and performative optimality to this setting by defining their online analogues that track the evolving partially performative environment. We analyze practical learning heuristics, including repeated retraining, and characterize when they successfully adapt to partially performative environments.
\end{abstract}


\section{Introduction}
\label{sec:1}

When prediction informs decision-making, predictive models do not merely describe a population; they can also shape it. Disease forecasts guide interventions that slow disease spread; recommendation systems influence user preferences; stock and home price forecasts affect the prices that ultimately materialize. Performative prediction \cite{perdomo20} formalizes this phenomenon by studying learning problems in which deploying a model changes the distribution of future observations.

The core idea is that deploying a model $\theta$ induces a model-dependent distribution over observations, denoted by $\gD(\theta)$. The learner's objective is to minimize the expected loss under the induced distribution, called the \emph{performative risk}:
\begin{align}
    \label{eqn:PP_objective}
    \PR(\theta) = \E_{Z \sim \gD(\theta)} [\ell(Z, \theta)],
\end{align}
where $Z$ denotes an observation, typically a feature--label pair $Z=(X,Y)$. The central object $\gD(\theta)$ is called the \emph{distribution map}. Thus, the goal is not simply to choose a model that performs well on the population as currently observed, but one that performs well after the model's impact on the population has been realized. Existing work studies two central solution concepts: \emph{performative stability} and \emph{performative optimality}, and analyzes when retraining procedures converge to such equilibria. In particular, these procedures generate a sequence of model updates $\theta_1,\theta_2,\dots$, whose limit is either a performatively stable or a performatively optimal point.

Performativity, therefore, surfaces as distribution shift: each model update may induce a new data distribution. What distinguishes performative distribution shift is that it is \emph{endogenous}: it arises from the learner's own actions. This contrasts with classical models of distribution shift in machine learning, which typically treat the shift as driven by \emph{exogenous} factors. For example, online learning studies algorithms for adapting a sequence of actions $\theta_t$ to an arbitrary sequence of distributions $P_t$.

The starting point of this work is the observation that, in realistic performative environments, distribution shift is rarely purely endogenous. Deployed predictive systems are embedded in changing populations, markets, and platforms. Their predictions may influence future data, but so too may seasonal effects, policy changes, market forces, and strategic behavior outside the learner's control. 

To capture this interplay, we introduce \emph{partially performative prediction}, an online model that combines endogenous and exogenous sources of distribution shift. At round $t$, the data distribution is
\begin{align*}
    \gD_t(\theta)
    &= (1 - \alpha_t)\gD(\theta) + \alpha_t P_t,
\end{align*}
where $\gD(\theta)$ captures the performative response to the deployed model $\theta$, $P_t$ captures an exogenous distribution shift, and $\alpha_t\in[0,1]$ controls their relative strength. This framework strictly generalizes the classical performative setting, recovered when $\alpha_t=0$, and approaches a purely exogenous-shift setting when $\alpha_t=1$. Note also that this model allows for a fixed non-performative component, obtained when $P_t$ is constant across $t$.

Partially performative models require a reconsideration of the learner's objective. Since the environment is fundamentally non-stationary, the goal can no longer simply be convergence to a fixed model, as in performative stability or performative optimality. Instead, the learner must contend with a moving target whose moves are only partly induced by the learner's own decisions. At the same time, the model highlights several new questions about the interaction between endogenous and exogenous shifts. Is a smaller value of $\alpha_t$ always more favorable to the learner, or can a large exogenous component help? How does the rate of change of $P_t$ affect the performance of learning algorithms? And how do the conditions under which retraining heuristics converge, or fail to converge, change relative to the classical performative setting ($\alpha_t=0$)?

\subsection{Our contributions}

We study all the above questions and more in the partially performative setting. Our first contribution is to extend the classical solution concepts of performative prediction to non-stationary environments. In the standard setting, performative stability and performative optimality are fixed models. Under partial performativity, the relevant stable and optimal points may change over time. We therefore introduce regret-based analogues, which we call \emph{performative stability regret} and \emph{performative optimality regret}. These quantities compare the learner's sequence of decisions to the time-varying sequence of performatively stable points
$\theta^\PS_1,\theta^\PS_2,\dots$
and performatively optimal points
$\theta^\PO_1,\theta^\PO_2,\dots$,
respectively. We then analyze several natural learning algorithms, including repeated risk minimization \cite{perdomo20} and gradient-based methods, under these regret notions.

A key quantity in our bounds is the \emph{path length}: the cumulative movement of the comparator sequence. For performative stability and performative optimality, we denote the one-step changes~by
\begin{align*}
    \Delta_{t}^{\PS} = \|\theta^\PS_{t} - \theta^\PS_{t+1}\| \quad \text{ and } \quad \Delta_{t}^{\PO} = \|\theta^\PO_t - \theta^\PO_{t+1}\|.
\end{align*}
The path lengths $\sum_t \Delta_{t}^{(\cdot)}$ measure the degree of non-stationarity in the partially performative problem. Slowly varying exogenous distributions $P_t$ lead to small path length, while rapidly changing $P_t$ can lead to large path length.

Our first main result bounds the stability regret of repeated risk minimization (RRM), an idealized form of retraining given by
\[\theta_{t+1} = \argmin_{\theta} \E_{Z\sim \gD_t(\theta_t)}[\ell(Z,\theta)].\]
Let $\gA_{\RRM}$ denote the RRM algorithm.

\begin{theorem}[Informal]
\label{thm:informalthm1}
    Suppose that the loss $\ell(z,\theta)$ is $\beta_{z}$-smooth in $z$ and $\mu$-strongly convex in $\theta$, and that the distribution map $\gD$ is $\epsilon$-Lispchitz with $\epsilon < \frac{\mu}{(1-\alpha_{\min})\beta_{z}}$, where $\alpha_{\min} = \min_{t\in[T]}\alpha_t$. Then, the performative stability regret of repeated risk minimization after $T$ steps satisfies
    \begin{align}
        \mathrm{Reg}_{T}^{ \PS }(\gA_{\RRM}) &= \gO\left(1 + \sum_{t=1}^{T-1} \Delta^\PS_{t} \right).
    \end{align}
\end{theorem}

This result makes explicit a tradeoff between endogenous and exogenous shifts. When the shifts are mostly performative ($\alpha_t \approx 0, \forall\, t$), the regret has no dependence on $T$ because the path length vanishes. The condition on performativity reduces to the standard condition $\epsilon < \frac{\mu}{\beta_{z}}$ from \cite{perdomo20}. Conversely, when the shifts are mostly exogenous ($\alpha_t \approx 1, \forall\, t$), there is no condition on the strength of performativity, but the regret depends on the non-stationarity of $P_t$. If $P_t$ is near-constant, the regret is constant too; if $P_t$ fluctuates rapidly, the regret could be large. Thus, retraining on the most recent data is effective under mild distribution shift, but ineffective when shifts are unpredictable.

We also prove analogous stability regret bounds for gradient-based algorithms, including those with ``greedy'' and ``lazy'' deployment schedules \cite{mendlerdunner20}.

Our second main contribution concerns performative optimality. We study algorithms that explicitly anticipate performativity by using zeroth-order gradient estimation techniques. Such methods perturb deployed models to estimate how changes in the model affect the induced distribution, and prior work has shown that related procedures can converge to performatively optimal points in the classical setting \cite{miller21, izzo2021learn, liu2024two, lin2024plug, chen24}. We show that, in partially performative environments, these ideas yield regret guarantees relative to the time-varying sequence of performatively optimal points.

\begin{theorem}[Informal]
\label{thm:informalthm2}
    Suppose that the performative risk at step $t$ is convex for all $t$. Then, there exists an algorithm $\gA$ that after $T$ steps achieves performative optimality regret of
    \begin{align*}
        \Reg^\PO_T(\gA)
        &= \gO \left( d\, \sqrt{T} \cdot \left( 1 + \sum_{t=1}^{T-1} \Delta^\PO_{t} \right) \right),
    \end{align*}
where $d$ is the dimension of $\theta$.
\end{theorem}

Again, we see path length as a key quantity in the bound. Like in the previous result, we can again observe a tradeoff between having a large endogenous and a large exogenous component. A large exogenous component might make the path length large. At the same time, a larger endogenous component can make the performative risk nonconvex, even when the loss $\ell(z,\theta)$ is convex. If $\alpha_t \approx 1$, the performative component is weak; consequently, convexity of the loss directly implies convexity of the performative risk. Thus, exogenous variation can make the target harder to track, while endogenous feedback can make the target harder to optimize.

\subsection{Related work}

Our work contributes to the growing literature on performative prediction \cite{perdomo20}, which formalizes settings in which deploying a model changes the future data distribution. A large body of work studies algorithms for finding performatively stable points \cite{mendlerdunner20, drusvyatskiy2023stochastic, mofakhami2023performative, wang2023network, narang2023multiplayer, perdomo2025revisiting, farina2026stability} and performatively optimal points \cite{miller21, izzo2021learn, kim2023making, narang2023multiplayer, xue2024distributionally, lin2024plug}, under progressively weaker assumptions. See \citet{hardt2025performative} and \citet{kehrenberg2026dissecting} for recent surveys of the area.

Most relevant to our work are extensions with stateful observation models \cite{brown2022performative, ray2022decision, izzo2022learn, li2022state, liu2024two}, which allow prior deployments to influence the current data distribution. These models capture memory, delayed responses, and other forms of endogenous dynamics. However, they still pursue the classical notions of performative stability and performative optimality; the underlying environment is fully determined by the history of deployments. In contrast, our model allows exogenous components of the data-generating process to evolve over time, so the relevant stable and optimal points may drift.

Another closely related thread studies performative regret minimization \cite{jagadeesan22, chen24, yan2023zero, park2024parameter}, where the learner deploys models sequentially and seeks regret bounds on the accumulated performative risk. In these works, the performative risk is static: the distribution map $\gD(\theta)$ is fixed, and regret is measured against a fixed baseline. In our setting, the distribution map $\gD_t(\theta)$ may vary with time. Accordingly, we study dynamic notions of performative regret, benchmarking the learner against time-varying sequences of performatively stable or performatively optimal models.

At a technical level, our model of partially performative prediction resembles the model of algorithmic collective action \cite{hardt2023algorithmic} in the presence of selfish agents by \citet{zhu2025look}. They similarly show a faster convergence rate of repeated risk minimization in the presence of a fixed, non-performative distributional component, which corresponds to the behavior of a coordinated collective in their case. They do not, however, analyze regret guarantees or non-stationary environments.

Our analysis builds on tools from online learning, in particular online convex optimization \cite{zinkevich03, hazan2016introduction} and zeroth-order optimization \cite{flaxman05, agarwal10}. We study dynamic regret, which measures performance relative to a changing sequence of optima rather than a single best baseline in hindsight. Dynamic regret  originated in tracking changing experts and comparators \cite{herbster98,zinkevich03} and was later sharpened through path and loss variation \cite{jadbabaie15,mokhtari16,yang16,zhang17}. Some of our results rely on zeroth-order optimization, where path-length-dependent dynamic regret guarantees are known \cite{zhao20}, though existing analyses do not accommodate performative feedback.

\section{Preliminaries}
\label{sec:2}

We begin by reviewing the formal setup of performative prediction, with a focus on the partially performative setting studied in this work. The learner chooses model parameters $\theta \in \Theta$, where $\Theta$ is a convex, compact subset of $\R^d$ of diameter $D_{\Theta}$. We write $\Pi_{\Theta}$ for the Euclidean projection onto $\Theta$. Performance is measured by a loss function $\ell(z,\theta)$, where $z$ denotes a data point; thus, $\ell(z,\theta)$ is the loss incurred by model $\theta$ on instance $z$.

In partially performative prediction, the data distribution after deploying model $\theta$ at time $t$ is
\begin{align}
    \gD_t(\theta)
    =
    (1-\alpha_t)\gD(\theta) + \alpha_t P_t,
    \label{eq:pp}
\end{align}
where $\alpha_t\in[0,1]$ is a fixed sequence of mixture coefficients and $P_t$ is an exogenous distribution that is not affected by the deployed model. The component $\gD(\theta)$ captures the performative response to model $\theta$, while $P_t$ captures time-varying external shift. Setting $\alpha_t=0$ recovers the standard performative prediction model \cite{perdomo20}. When the deployed model at time $t$ is $\theta_t$, we sometimes use the shorthand $\gD_t \coloneqq \gD_t(\theta_t)$.

The partially performative setting admits natural analogues of the central concepts from performative prediction. We define the \emph{performative risk} and \emph{decoupled performative risk} at time $t$ as
\begin{align*}
    \mathrm{PR}_t(\theta)
    &\coloneqq
    \E_{Z\sim \gD_t(\theta)}[\ell(Z,\theta)],\quad 
    \mathrm{DPR}_t(\theta,\theta')
    \coloneqq
    \E_{Z\sim \gD_t(\theta')}[\ell(Z,\theta)].
\end{align*}
The performative risk evaluates a model on the distribution that it induces. The decoupled performative risk separates the model being evaluated, $\theta$, from the model that induces the distribution,~$\theta'$.

The \emph{performatively stable point} at time $t$ is defined as the point $\theta_t^{\PS}\in\Theta$ satisfying
\begin{align*}
    \theta_t^{\PS}
    \coloneqq
    \argmin_{\theta\in\Theta}
    \mathrm{DPR}_t(\theta,\theta_t^{\PS}).
\end{align*}
Thus, once $\theta_t^{\PS}$ is deployed and induces the distribution $\gD_t(\theta_t^{\PS})$, it minimizes the expected loss on that induced distribution. The \emph{performatively optimal point} at time $t$ is the point
\begin{align*}
    \theta_t^{\PO}
    \coloneqq
    \argmin_{\theta\in\Theta}
    \mathrm{PR}_t(\theta).
\end{align*}
Performative optimality jointly accounts for the loss incurred by predicting with $\theta$ and the distributional response induced by deploying $\theta$. In general, stable and optimal points do not coincide~\cite{perdomo20}.

Because the stable and optimal points may vary with time, a natural objective is to achieve low regret relative to the sequences $\{\theta_t^\PS\}_{t\in[T]}$ and $\{\theta_t^\PO\}_{t\in[T]}$. This leads to our definitions of \emph{performative stability regret} and \emph{performative optimality regret}.
Let $\gA$ be an algorithm that outputs a possibly stochastic sequence of iterates $\{\theta_t\}_{t\in[T]}$. We define the performative stability regret and the performative optimality regret of $\gA$ as
\begin{align*}
    \Reg_T^{\PS}(\gA)
    &\coloneqq
    \sum_{t=1}^{T}
    \left(
        \E[\mathrm{PR}_t(\theta_t)]
        -
        \mathrm{PR}_t(\theta_t^\PS)
    \right), \\ 
    \Reg_T^{\PO}(\gA)
    &\coloneqq
    \sum_{t=1}^{T}
    \left(
        \E[\mathrm{PR}_t(\theta_t)]
        -
        \mathrm{PR}_t(\theta_t^\PO)
    \right).
\end{align*}
The expectation is taken over any randomness in the algorithm.

Several previous works have studied regret guarantees in performative prediction \cite{jagadeesan22, yan2023zero, chen24, farina2026stability}. Their guarantees are given in terms of \emph{static} regret, which compares the learner's realized loss to that of a single fixed $\theta\in\Theta$. In the standard performative prediction model, this is natural, since the stable and optimal points do not change over time. In the partially performative setting, however, the exogenous component can make these points time-varying. We therefore study \emph{dynamic} regret, which compares the learner's performance to a sequence of models. This distinction parallels the classical difference between static and dynamic regret in online learning; see, for example, \cite{hazan2016introduction}.

Finally, we invoke several common assumptions from the performative prediction literature, including (strong) convexity and smoothness conditions on the loss $\ell(z,\theta)$.
We say that a continuously differentiable function $\ell : \gZ \times \Theta \to \R$ is 
$\mu$-strongly convex in $\theta$ if
\begin{align}
    \ell(z, \theta')
    &\ge \ell(z, \theta) + \langle \nabla_{\theta} \ell(z, \theta), \theta' - \theta \rangle + \frac{\mu}{2} \| \theta' - \theta \|^2, \qquad \forall \theta, \theta' \in \Theta, \, z \in \gZ,
    \label{eq:strong_convexity}
\end{align}
and convex in $\theta$ if the above holds with $\mu=0$.

We say that $\ell$ is $\beta_{\theta}$-smooth in $\theta$ if
\begin{align}
    \| \nabla_{\theta} \ell(z, \theta) - \nabla_{\theta} \ell(z, \theta') \|
    &\le \beta_{\theta} \| \theta - \theta' \|, \qquad \forall \, \theta, \theta' \in \Theta \, , z \in \mathcal{Z},
    \label{eq:smoothness_theta}
\end{align}
and $\beta_{z}$-smooth in $z$ if\footnote{Note that $\beta_{z}$-smoothness is also defined using the gradient with respect to $\theta$.}
\begin{align}
    \| \nabla_{\theta} \ell(z, \theta) - \nabla_{\theta} \ell(z', \theta) \|
    &\le \beta_{z} \| z - z' \|, \qquad \forall \, z, z' \in \mathcal{Z}, \, \theta \in \Theta.
    \label{eq:smoothness_z}
\end{align}
We say that $\PR_t(\theta)$ is $L$-Lipschitz if
\begin{align}
    | \PR_t(\theta) - \PR_t(\theta') |
    &\le L \| \theta - \theta' \|, \qquad \forall \, \theta, \theta' \in \Theta.
    \label{eq:lipschitz}
\end{align}

Here we also state one assumption specific to performative settings: sensitivity of the distribution map.
We say that $\gD$ is $\epsilon$-sensitive if
\begin{align}
\label{eq:sensitivity}
    \gW_1\bigl(\gD(\theta), \gD(\theta')\bigr)
    \le
    \epsilon \|\theta-\theta'\|,
    \qquad
    \forall\, \theta,\theta'\in\Theta,
\end{align}
where $\gW_1$ denotes the Wasserstein-$1$ distance. This condition requires the performative response to vary smoothly with the deployed model. Intuitively,  $\epsilon$ measures the strength of performative effects.

\begin{remark}
    \label{rem:admissible}
    When we make assumptions on the loss as a function of $z$, we require the assumptions to hold for all $\theta$; for example, $z \mapsto \ell(z, \theta)$ is $\beta_{z}$-smooth for all $\theta$.
    For the assumptions on the loss as a function of $\theta$, however, our results hold under a weaker assumption that the \emph{expected} risk $\mathrm{R}_{D} (\theta) \coloneqq \E_{Z \sim D} [\ell(Z, \theta)]$ is $\mu$-strongly convex, $\beta_{\theta}$-smooth, etc, for all admissible distributions that could be generated by $\gD_t(\theta_t)$. For simplicity we state the assumptions for all $z\in \mathcal{Z}$, as in \cref{eq:strong_convexity,eq:smoothness_theta,eq:lipschitz}, however all the proofs go through if we consider the expected risk.
    
    Making assumptions on the expected risk is strictly weaker, especially if we have additional structure on the distributions that can be generated by $\gD_t(\theta_t)$.
    For example, consider a simple one-dimensional coin-flip setting with a single feature where the random variable $Z = (X, Y) \in \{-1, +1\}^{2}$ has a fixed distribution of $X \sim \Unif(\{-1, +1\})$ and only the conditional law of $Y \mid X$ can change due to performative and exogenous factors.
    Consider
    \begin{align*}
        \ell(z, \theta)
        &= \frac{1}{2} (\theta - y)^2 - x \theta^2.
    \end{align*}
    We have that $\theta \mapsto \ell(z, \theta)$ is nonconvex in $\theta$ for $x=1$.
    However, if we average the loss over the distribution of $(X,Y)$, we have
    \begin{align*}
        \mathrm{R}_{D} (\theta)
        &= \E_{Z \sim D} \left[ \frac{1}{2} (\theta - Y)^2 - X \theta^2 \right] = \E_{Z \sim D} \left[ \frac{1}{2} (\theta - Y)^2 \right],
    \end{align*}
    which is $1$-strongly convex in $\theta$.
    Similarly, we can also consider
    \begin{align*}
        \ell(z, \theta)
        &= \frac{1}{2} (\theta - y)^2 - x |\theta|,
    \end{align*}
    which is non-smooth at $\theta = 0$, but admits a $1$-smooth $\mathrm{R}_{D} (\theta)$ under the same setup.
\end{remark}

\section{Repeated risk minimization under partial performativity}
\label{sec:3}

We begin by studying repeated risk minimization (RRM) \citep{perdomo20}. Given an initial $\theta_1 \in \Theta$, RRM sets
\begin{align}
\tag{RRM}
    \theta_{t+1} & = \argmin_{\theta \in \Theta} \E_{Z \sim \gD_t(\theta_t)} [\ell (Z, \theta)].
\end{align}
We use $\ARRM$ to denote this algorithm.

We start from a key technical lemma that characterizes one-step RRM updates.

\begin{restatable}{lemma}{lempislipschitz}
    \label{lem:pislipschitz}
    Suppose that $\gD(\theta)$ is $\epsilon$-sensitive \eqref{eq:sensitivity}
    and $\ell(z, \theta)$ is $\mu$-strongly convex in $\theta$ \eqref{eq:strong_convexity} and $\beta_{z}$-smooth in $z$ \eqref{eq:smoothness_z}.
    Then, the iterates of $\ARRM$ satisfy
    \begin{align*}
        \norm{\theta_{t+1} - \theta_{t}^{\PS}} &\le \frac{(1 - \alpha_{t}) \epsilon \beta_{z}}{\mu} \| \theta_{t} - \theta_{t}^{\PS} \|. 
    \end{align*}
\end{restatable}
We defer the proof of \cref{lem:pislipschitz} to \cref{subsec:pislipschitz}.
We note that \cref{lem:pislipschitz} is a strict generalization of Theorem 3.5 by Perdomo et al.~\cite{perdomo20}, recovered when $\alpha_t = 0$ and $\theta_t^\PS \equiv \theta^\PS$ for all $t$. Building on this lemma, we obtain the following stability regret for RRM.

\begin{restatable}{theorem}{thmrrm}
    \label{thm:rrm}
    Suppose that $\gD(\theta)$ is $\epsilon$-sensitive \eqref{eq:sensitivity}, $\ell(z, \theta)$ is $\mu$-strongly convex in $\theta$ \eqref{eq:strong_convexity}
    and $\beta_{z}$-smooth in $z$ \eqref{eq:smoothness_z},
    and $\PR_t (\theta)$ is $L$-Lipschitz \eqref{eq:lipschitz} in $\theta$.
    Then, if $\epsilon < \frac{\mu}{(1 - \alpha_{\min}) \beta_{z}}$, where $\alpha_{\min} = \min_{t \in [T]} \alpha_{t}$, the iterates of $\ARRM$ satisfy
    \begin{align}
        \mathrm{Reg}_{T}^{ \PS } (\gA_{\mathrm{RRM}}) &\le \frac{L}{1 - \gamma} \cdot \bigopen{ \norm{ \theta_{1} - \theta_{1}^{\PS} } + \sum_{t=1}^{T-1} \Delta_{t}^{\PS} },
        \label{eq:rrm}
    \end{align}
    for $\gamma = \frac{(1 - \alpha_{\min}) \epsilon \beta_{z}}{\mu}$.
\end{restatable}

We defer the proof of \cref{thm:rrm} to \cref{subsec:rrm}.

\begin{remark}
    Although we state all (strong) convexity and smoothness conditions directly in terms of the loss $\ell$ for simplicity, all results in the paper also hold under the corresponding conditions on the \emph{expected} loss; see \cref{rem:admissible} for details.
    Also, we show in \cref{lem:rgd_lislip} that $\ell$ being Lipschitz in both $z$ and $\theta$ implies that $\PR_t (\theta)$ is Lipschitz.
\end{remark}

In the standard performative setting with $\alpha_t \equiv 0$, convergence of RRM to stability required that the contraction factor $\gamma=\frac{\epsilon \beta_{z}}{\mu}$ be strictly less than one.
Under the partially performative setting with possible $\alpha_t > 0$, the contraction factor $\gamma$ gets smaller by a factor of $(1 - \alpha_t)$, and thus we only need a weaker sensitivity condition $\gamma=\frac{(1-\alpha_{\min})\epsilon \beta_{z}}{\mu} <1$, or equivalently
$\epsilon < \frac{\mu}{(1 - \alpha_{\min}) \beta_{z}}$. At the same time, large exogenous shifts may lead to large gaps between subsequent stable points $\Delta_{t}^{\PS}$.
To provide more intuition for $\Delta_{t}^{\PS}$, we state an upper bound on $\Delta_{t}^{\PS}$ as a function of the mixture coefficient $\alpha_t$, how fast the distribution $P_t$ varies, and the properties of the loss function.

\begin{restatable}{proposition}{propwassbound}
\label{prop:wassbound}
    Fix $t\in \mathbb{N}$.
    Suppose that $\gD(\theta)$ is $\epsilon$-sensitive \eqref{eq:sensitivity}
    and $\ell(z, \theta)$ is $\mu$-strongly convex in $\theta$ \eqref{eq:strong_convexity} and $\beta_{z}$-smooth in $z$ \eqref{eq:smoothness_z}.
    Further, suppose that $\gW_{1} (\gD (\theta_{t+1}^{\mathrm{PS}}), P_{t+1}) \leq C$ for some $C<\infty$. Then, if $\epsilon < \frac{\mu}{(1 - \alpha_t) \beta_{z}}$, we have
    \begin{align*}
        \Delta_{t}^{\PS} \le \frac{\beta_{z}}{\mu-\beta_{z}(1-\alpha_t)\epsilon} \left(\alpha_{t} \gW_{1} (P_{t}, P_{t+1}) + 
        \lvert \alpha_{t} - \alpha_{t+1} \rvert C\right).
    \end{align*}
\end{restatable}

We defer the proof of \cref{prop:wassbound} to \cref{subsec:wassbound}.
We observe that $\Delta_t^\PS$ is driven by (i) how fast the exogenous component drifts, $\gW_{1} (P_{t}, P_{t+1})$; (ii) the contribution of the exogenous component $\alpha_t$; and (iii) the local fluctuations in the mixture coefficient, $|\alpha_t - \alpha_{t+1}|$. In particular, $\Delta_t^\PS$ is small whenever $\alpha_t$ is consistently small or the distributions $P_t$ do not vary significantly.

As a special case, suppose that $\alpha_{t} \le c t^{-b}$, for some $c, b>0$, and that $\gW_{1} (P_{t}, P_{t+1})$ is uniformly bounded for all $t$. Then, if $\epsilon < \mu/\beta_{z}$, we can conclude
\begin{align*}
    \mathrm{Reg}_{T}^{ \PS } (\gA_{\mathrm{RRM}})
    = \mathcal{O} \bigopen{ 1 + \sum_{t=1}^{T-1} \alpha_{t} }.
\end{align*}

For the case $b\in (0,1)$, we have $\sum_{t=1}^{T-1} \alpha_{t} = \gO(T^{1-b})$, and therefore $\mathrm{Reg}_{T}^{\PS} (\ARRM) = \gO(T^{1-b})$.
Similarly, for $b = 1$ we have $\sum_{t=1}^{T-1} \alpha_{t} = \mathcal{O} (\log T)$, and for $b > 1$ we have $\sum_{t=1}^{T-1} \alpha_{t} = \mathcal{O} (1)$, each implying a regret upper bound of the same rate, respectively.

\section{Gradient descent under partial performativity}
\label{sec:4}

Next, we consider several variants of gradient descent and prove their stability regret bounds. We begin by considering repeated gradient descent (RGD) with access to the full expected gradient, after which we move on to stochastic gradient descent. In the standard performative prediction setting, the former was studied by \cite{perdomo20, mendlerdunner20} and the latter was studied by \cite{mendlerdunner20, drusvyatskiy2023stochastic}, among others.

\subsection{Repeated gradient descent}
\label{sec:rgd_main}

First, we consider repeated gradient descent (RGD), 
\begin{align*}
\tag{RGD}
    \theta_{t+1} &= \Pi_{\Theta} \big( \theta_{t} - \eta_{t} \E_{Z\sim \gD_t(\theta_t)}[\nabla \ell(Z,\theta_t)] \big),
\end{align*}
where $\eta_t>0$ is a step size sequence. We denote this algorithm by $\gA_{\mathrm{RGD}}$.

We obtain a similar rate for the stability regret bound as in the case of RRM.

\begin{restatable}{theorem}{thmrgd}%
    \label{thm:rgd}
    Suppose that $\gD(\theta)$ is $\epsilon$-sensitive \eqref{eq:sensitivity}, $\ell(z, \theta)$ is $\mu$-strongly convex \eqref{eq:strong_convexity} and $\beta_{\theta}$-smooth \eqref{eq:smoothness_theta} in $\theta$ and $\beta_z$-smooth \eqref{eq:smoothness_z} in $z$, and $\PR_t (\theta)$ is $L$-Lipschitz \eqref{eq:lipschitz} in $\theta$.
    Then, if $\epsilon < \frac{\mu}{(1 - \alpha_{\min}) \beta_{z}}$, where $\alpha_{\min} = \min_{t \in [T]} \alpha_{t}$,
    the iterates of $\gA_{\mathrm{RGD}}$ with step size
    \begin{align*}
        \eta_{t} &= \frac{\mu - (1 - \alpha_{t}) \epsilon \beta_{z}}{2 ( \beta_{\theta}^{2} + (1 - \alpha_{t})^{2} \epsilon^{2} \beta_{z}^{2} )}
    \end{align*}
    satisfy
    \begin{align}
        \mathrm{Reg}_{T}^{ \PS } (\gA_{\mathrm{RGD}}) &\le \frac{L}{1 - \overline{\gamma}} \cdot \bigopen{ \norm{ \theta_{1} - \theta_{1}^{\PS} } + \sum_{t=1}^{T-1} \Delta_{t}^{\PS} },
        \label{eq:rgd}
    \end{align}
    for $\overline{\gamma} = 1 - \frac{(\mu - (1 - \alpha_{\min}) \epsilon \beta_{z})^2}{4 ( \beta_{\theta}^{2} + (1 - \alpha_{\min})^{2} \epsilon^{2} \beta_{z}^{2} )}$.
\end{restatable}

We defer the proof of \cref{thm:rgd} to \cref{subsec:rgd}. As in the case of RRM, we observe a regret upper bound of $\gO (1 + \sum_{t=1}^{T-1} \Delta_{t}^{\PS})$. The dependence on the path length mirrors existing results for online gradient descent \cite{mokhtari16}.

Interestingly, \citet{mendlerdunner20} assume that the stable point lies in the interior of $\Theta$. Our proof of \Cref{thm:rgd} bypasses such an assumption, and thus weakens the conditions for convergence to stability in fully performative environments, obtained by taking $\alpha_t \equiv 0$.


\subsection{Stochastic gradient descent: greedy and lazy deploy}
\label{sec:rsgdg_main}

We consider two variants of stochastic gradient descent. Both sequentially collect samples from the currently deployed model and take a single gradient step after each. They differ in their data collection schedule: \emph{greedy deploy} collects a single data point between two deployments, while \emph{lazy deploy} collects multiple data points and takes multiple offline steps between two deployments.

Before formally stating the algorithms and their guarantees, we introduce a standard assumption in stochastic optimization. We assume that the stochastic gradients have a quadratic-bounded variance: for all possible $\gD_t$, for some $C_V$ we have
\begin{align}
\label{ass:bv}
        \E_{Z \sim \gD_t} [\norm{\nabla_{\theta} \ell (Z, \theta)}^2]
        &\le \sigma^2 + C_V^2 \norm{\theta - \theta^{\star}_{\gD_t}}^2,
\end{align}
where $\theta^{\star}_{\gD_t} \coloneqq \argmin_\theta \E_{Z\sim \gD_t}[\ell(Z, \theta)]$.

\paragraph{Greedy deploy.}
Consider the \textit{greedy deploy} \citep{mendlerdunner20} variant of stochastic gradient descent (SGD):
\begin{align*}
\tag{greedy SGD}
    \theta_{t+1} &= \Pi_{\Theta} ( \theta_{t} - \eta_{t} \nabla_{\theta} \ell (Z_{t}, \theta_{t}) ), \quad \text{where} \ Z_{t} \sim \gD_{t},
\end{align*}
which uses a single sample per each deployment.
We denote the algorithm by $\ARSGDG$.

\begin{restatable}{proposition}{thmrsgdg}
    \label{thm:rsgdg}
    Let the same assumptions as in \cref{thm:rgd} and the quadratic variance bound \eqref{ass:bv} hold.
    Then, the iterates of $\ARSGDG$ with step size $\eta_{t} \propto \frac{1}{t + t_0}$, for some constant $t_0 > 0$, satisfy
    \begin{align}
        \mathrm{Reg}_{T}^{\PS} (\ARSGDG)
        &\le \gO \left( T^{1/2} + T^{1/4} \left( \sum_{t=1}^{T-1} (t + t_0 + 1)^{5/2} (\Delta_{t}^{\PS})^2 \right)^{\!\!1/2} \, \right).
        \label{eq:rsgdg}
    \end{align}
\end{restatable}

We defer the proof of \cref{thm:rsgdg} to \cref{subsec:rsgdg}.
We can observe from \eqref{eq:rsgdg} that $\ARSGDG$ enjoys sublinear regret whenever we have $\Delta_{t}^{\PS} \le Mt^{-b}$ for $b > 1$,
\begin{align*}
    \mathrm{Reg}_{T}^{\PS} (\ARSGDG) &= \begin{cases}
        \gO (T^{2-b}), & \text{if } b\in(1, \frac{3}{2}), \\
        \gO (T^{1/2}), & \text{if } b\geq \frac{3}{2}.
    \end{cases}
\end{align*}
As discussed in Section \ref{sec:3}, we can have $\Delta_{t}^{\PS} \le Mt^{-b}$ when $\alpha_t \le ct^{-b}$ for some constant $c > 0$.

\begin{remark}
    Unlike the other results in \cref{sec:3} and \cref{sec:4}, the upper bound in \cref{thm:rsgdg} does not guarantee sublinear regret of $\ARSGDG$ when $\Delta_{t}^{\PS} \le Mt^{-b}$ decays with a rate of $b\in (0,1)$.
    We leave the search for a tighter upper bound or a matching lower bound to future work.
\end{remark}


\paragraph{Lazy deploy.}
Consider the \textit{lazy deploy} \citep{mendlerdunner20} variant of stochastic gradient descent (SGD):
\begin{align*}
\tag{lazy SGD}
    \varphi_{t, j+1} = \Pi_{\Theta} ( \varphi_{t, j} - \eta_{t, j} \nabla_{\theta} \ell (Z_{t, j}, \varphi_{t, j}) ),
\end{align*}
where $Z_{t, j} \stackrel{\text{i.i.d.}}{\sim} \gD_{t}$ for $j \in [n(t)]$, with $\varphi_{t, 1} = \theta_{t}$ and $\varphi_{t, n(t)+1} = \theta_{t+1}$.
In other words, $\varphi_{t,j}$ are internal iterates computed between two deployments $\theta_t$ and $\theta_{t+1}$. Within each deployment, we collect $n(t)$ samples from $\gD_{t}$, and we assume $n(t) \ge n_0 t^{r}$ for some $r > 0$. We denote this algorithm by $\ARSGDL$.

\begin{restatable}{proposition}{thmrsgdl}
    \label{thm:rsgdl}
    Suppose that the same assumptions as in \cref{thm:rsgdg} hold.
    Then, the iterates of $\ARSGDL$ with step sizes $\eta_{t, j} \propto \frac{1}{j + t_0}$, for some constant $t_0>0$, and $n(t) \ge n_0 t^{r}$ with a large enough $n_0$ satisfy
    \begin{align}
        \mathrm{Reg}_{T}^{\PS} (\ARSGDL)
        &\le \gO \left( \sum_{t=1}^{T-1} t^{-r/2} + \sum_{t=1}^{T-1} \Delta_{t}^{\PS} \right).
        \label{eq:rsgdl}
    \end{align}
\end{restatable}

We defer the proof of \cref{thm:rsgdl} to \cref{subsec:rsgdl}.
Summing up over $t$, we see that the first term in the regret is upper bounded by $\gO(1)$ for $r >2$, $\gO(\log T)$ for $r=2$, and $\gO (T^{1-r/2})$ for $r\in(0,2)$.

\section{Achieving optimality under partial performativity}
\label{sec:5}

We now turn our attention to studying performative optimality regret. Performative optimality is harder to achieve than stability because computing the gradients of $\PR_t(\theta)$ is rarely feasible, due to the (typically complex) dependence of $\PR_t(\theta)$ on $\theta$ through $\gD_t(\theta)$. We thus turn to zeroth-order algorithms, which only require querying $\PR_t(\theta)$, not its gradient.

Throughout, we make a couple of mild assumptions: that $|\ell(z, \theta)| \le F$ for some finite $F$, for all $z,\theta$, and that $\Theta$ contains a ball of radius $r$ centered at $0$, for some $r>0$.

To prove that the regret vanishes, we will need $\PR_t(\theta)$ to be convex or strongly convex. This is not guaranteed by the convexity of $\ell(z,\theta)$ alone. Miller et al.~\cite{miller21} identify conditions that ensure convexity of the performative risk, the most important one being a so-called mixture dominance condition. Miller et al. show that this condition is satisfied by location-scale families. In the partially performative setting, convexity is easier to ensure, since the  $\alpha_t$-fraction of the objective corresponding to the exogenous component is convex whenever $\ell(z,\theta)$ is convex. We formalize this below.

\begin{proposition}
    Suppose that $\gD(\theta)$ is $\epsilon$-sensitive \eqref{eq:sensitivity}, and $\ell(z, \theta)$ is $\mu$-strongly convex in $\theta$ \eqref{eq:strong_convexity} and $\beta_{z}$-smooth in $z$ \eqref{eq:smoothness_z}.
    Assume also the mixture dominance condition \cite{miller21}.
    Then, $\mathrm{PR}_{t}$ is $(\mu - 2 (1 - \alpha_{t}) \epsilon \beta_{z})$-convex.
\end{proposition}

If $\mu - 2 (1 - \alpha_{t}) \epsilon \beta_{z}$ is positive, then $\PR_t(\theta)$ is \emph{strongly} convex by that amount.
If $\mu - 2 (1 - \alpha_{t}) \epsilon \beta_{z}$ is negative, then $\PR_t(\theta)$ is \emph{weakly} convex by that amount, meaning that adding a quadratic regularizer of that magnitude makes $\PR_t(\theta)$ convex.

\subsection{Two-deployment algorithm}

Consider two-point zeroth-order gradient descent (ZGD2) \cite{agarwal10}: starting from an initial $\theta_1 \in \Theta$, we~set
\begin{align*}
\tag{ZGD2}
    \theta_{t+1} &= \Pi_{(1 - \rho) \Theta} \bigopen{ \theta_{t} - \eta_t \cdot \tilde{g}_{t} },\\
    \tilde{g}_{t} &= \big( \mathrm{PR}_{t} (\theta_{t} + \delta u_{t}) - \mathrm{PR}_{t} (\theta_{t} - \delta u_{t}) \big) u_{t} \frac{d}{2 \delta}, \text{ where } u_{t} \sim \mathrm{Unif} (\sS^{d-1}).
\end{align*}
Here, $\eta_t>0$ is a step size, $\delta>0$ and $\rho \in [0, 1]$ are hyperparameters, and $u_t$ is sampled uniformly on the $d$-dimensional unit sphere $\sS^{d-1}$.
In other words, we deploy two models: $\phi_t^+ = \theta_t + \delta u_t$ and $\phi_t^- = \theta_t - \delta u_t$, observe their corresponding risks $\mathrm{PR}_{t} (\phi_t^+)$ and $\mathrm{PR}_{t} (\phi_t^-)$, and use those values to perform an approximate gradient update on $\theta_t$.
Practically, one can think of the procedure as partitioning the population into two halves and deploying two different treatments $\phi_t^+$ and $\phi_t^-$ within each.
It is known that $\E [\tilde{g}_{t} | \theta_{t}] = \nabla \widehat{\mathrm{PR}}_{t} (\theta_{t})$ for
$\widehat{\mathrm{PR}}_{t} (\theta) \coloneqq \E_{v} [ \mathrm{PR}_{t} (\theta + \delta v) ]$ with $v$ sampled uniformly within the $d$-dimensional unit ball,
$v \sim \Unif (\sB^{d})$.
We use projections to $(1 - \rho) \Theta$ with $\rho \ge \frac{\delta}{r}$ to ensure that the perturbed iterates $\{\phi_t^+,\phi_t^-\}$ are in the domain, i.e., $\theta_{t} \pm \delta u_{t} \in \Theta$.
We denote the algorithm by $\AZGDtwo$.

Note that here we have applied a slight abuse of notation, since $\theta_t$ is not the actual sequence \emph{deployed} by $\AZGDtwo$; $\theta_t$ is the sequence used internally for optimization. To accurately capture optimality regret, we measure it with respect to the deployed sequence $\{\phi_t^+,\phi_t^-\}$:
\begin{align*}
    \Reg_T^\PO(\AZGDtwo)
    &= 
    \sum_{t=1}^{T}
    \left(
        \max_{\phi_t \in \{\phi_t^+,\phi_t^-\}}\E[\mathrm{PR}_t(\phi_{t})]
        -
        \mathrm{PR}_t(\theta_t^\PO)
    \right).
\end{align*}

\begin{restatable}{theorem}{thmzgdtwopoint}
    \label{thm:zgdtwopoint}
    Suppose that $\mathrm{PR}_{t}$ is $L$-Lipschitz \eqref{eq:lipschitz}. If $\PR_t$ is $\mu'$-strongly convex, then $\AZGDtwo$ with parameters $\eta_{t} = \frac{1}{\mu' t}$, $\delta = \frac{d^2 L}{6 \mu' T}$, and $\rho = \frac{\delta}{r}$ satisfies, for $T \ge \frac{d^2L}{6\mu' r}$,
    \begin{align*}
        \Reg_T^\PO(\AZGDtwo)
        &= \gO \left( d^2 \log T + \sum_{t=1}^{T-1} t \Delta_{t}^{\PO} \right).
    \end{align*}
    If $\PR_t$ is convex, then $\AZGDtwo$ with parameters $\eta_t \equiv \eta = \frac{D_{\Theta}}{d L \sqrt{T}}$, $\delta = \frac{dD_{\Theta}}{6\sqrt{T}}$, and $\rho = \frac{\delta}{r}$ satisfies, for $T \ge \frac{d^2D_\Theta^2}{36r^2}$,
    \begin{align*}
        \Reg_T^\PO(\AZGDtwo)
        &= \gO \left(  d \sqrt{T} \cdot \left( 1 + \sum_{t=1}^{T-1} \Delta_{t}^{\PO} \right) \right).
    \end{align*}
\end{restatable}

We defer the proof of \cref{thm:zgdtwopoint} to \cref{subsec:zgdtwopoint}.

If we only have access to stochastic feedback, like in \cref{sec:rsgdg_main}, then two-point zeroth-order gradient descent is not appropriate, since we cannot subject any one individual $Z_t$ to two different treatments $\phi_t^+$ and $\phi_t^-$. To address this challenge, in the following we analyze a one-point version of zeroth-order gradient descent and bound its optimality regret.


\subsection{Single-deployment algorithm}
\label{subsec:zgdmain}

Here we consider one-point zeroth-order gradient descent (ZGD) \cite{flaxman05}, which we denote by $\AZGD$:
\begin{align*}
\tag{ZGD}
\theta_{t+1} &= \Pi_{(1 - \rho) \Theta} (\theta_{t} - \eta_t  g_t),\\
 g_t &=  \frac{d}{\delta} \ell (Z_{t}, \phi_{t}) u_{t}, \text{ where }   u_{t} \sim \mathrm{Unif} (\sS^{d-1}), 
    \phi_{t} = \theta_{t} + \delta u_{t}, \text{ and }
     Z_{t} \sim \mathcal{D}_t (\phi_t),
\end{align*}
where $\eta_t >0$ is a step size, $\delta >0$ and $\rho\in[0,1]$ are hyperparameters, and $u_t$ is sampled uniformly on the $d$-dimensional sphere $\sS^{d-1}$.
Note that for one-point ZGD, we can use $\ell(Z_t, \phi_t)$ with $Z_t \sim \mathcal{D}_t (\phi_t)$ instead of computing the whole population risk $\mathrm{PR} (\phi_t)$ from stochastic feedback.

Similar to $\AZGDtwo$, it is known that $\E [g_{t}] = \frac{d}{\delta}\E_{u_t}[\mathrm{PR}_t(\theta_t+\delta u_t)u_t] = \nabla \widehat{\mathrm{PR}}_{t} (\theta_{t})$ for
$\widehat{\mathrm{PR}}_{t} (\theta) \coloneqq \E_{v} [ \mathrm{PR}_{t} (\theta + \delta v) ]$ with $v \sim \Unif (\sB^{d})$.
As for $\AZGDtwo$, we will provide upper bounds of the optimality regret in terms of the deployed sequence $\{ \phi_{t} \}_{t \in [T]}$,
\begin{align*}
    \Reg_{T}^{\PO} (\AZGD)
    &=  \sum_{t=1}^{T} \big( \E[\mathrm{PR}_{t} (\phi_{t})] - \mathrm{PR}_{t} (\theta_{t}^{\PO}) \big).
\end{align*}

\begin{restatable}{theorem}{thmzgd}
    \label{thm:zgd}
Suppose that $\mathrm{PR}_{t}$ is $L$-Lipschitz \eqref{eq:lipschitz}. If $\PR_t$ is $\mu'$-strongly convex, then $\AZGD$ with parameters $\eta_{t} = \frac{1}{\mu' t}$, $\delta \propto d^{2/3} (\frac{1 + \log T}{T})^{1/3}$, and $\rho = \frac{\delta}{r}$ satisfies the following optimality regret upper bound for $T = \tilde{\Omega} (d^2)$:
    \begin{align*}
        \Reg_{T}^{\PO} (\AZGD)
        &= \gO \left( d^{\frac{2}{3}} T^{\frac{2}{3}} (\log T)^{\frac{1}{3}} + \sum_{t=1}^{T-1} t \Delta_{t}^{\PO} \right).
    \end{align*}
    If $\mathrm{PR}_t$ is convex, then $\AZGD$ with parameters $\eta_t \equiv \eta \propto d^{-1/2} T^{-3/4}$, $\delta \propto d^{1/2} T^{-1/4}$, and $\rho = \frac{\delta}{r}$ satisfies the following optimality regret upper bound for $T = \tilde{\Omega} (d^2)$:
    \begin{align*}
        \Reg_{T}^{\PO} (\AZGD)
        &= \gO \left( d^{\frac12} T^{\frac{3}{4}} \cdot \left( 1 + \sum_{t=1}^{T-1} \Delta_{t}^{\PO} \right) \right).
    \end{align*}
\end{restatable}

We defer the proof of \cref{thm:zgd} to \cref{subsec:zgd}.
Given the weaker feedback signal, the regret bounds are naturally larger than that of ZGD2.

\begin{remark}
    \citet{chen24} similarly study how to use two deployed models to construct a zeroth-order estimate of the performative risk gradient in a fully performative setting.
    Our result gives a new analysis of the \textit{dynamic} optimality regret of the two-deployment algorithm under partial performativity.
\end{remark}

\section{Experiments}
\label{sec:6}

We conduct experiments evaluating the stability regret of the retraining algorithms considered in the paper. We include several main results in this section and defer further empirical results and additional experimental details to Appendix \ref{sec:exps}.

We use the credit scoring dataset from \cite{yeh2009comparisons}, in particular the processed version from \cite{ustun19}. The dataset contains demographic and socioeconomic features $x_i\in \R^{16}$, including age, education, and credit history, and labels $y_i$ indicating whether a person defaulted on their credit card payment, for $30,000$ individuals. The learner aims to predict the individuals' creditworthiness by minimizing a regularized logistic loss over an $\ell_2$ ball in $d=17$ dimensions (one dimension corresponding to the intercept), after standardizing the features. The performative feedback is modeled as coming from strategic responses: an individual $i$ responds to a deployed model $\theta$ by optimizing the score $\theta^\top x_i$  (equivalently, maximizing the prediction), with an additional quadratic penalty for deviating too far from their initial features $x_{i,0}$. Only certain features are modifiable (for example, age is not), and thus the individuals only manipulate those modifiable features. We model the initial features $x_0$ in the dataset as being drawn from a normal distribution $\mathcal{N}(m, \Sigma)$, and we fit $m$ and $\Sigma$ from $1500$ data points. The individuals' best response has a closed-form expression in this case:
\[x \leftarrow x_0 + A \theta,\]
where $A=a\cdot \,\operatorname{diag}(p_{\mathrm{mod}})$, $a$ is a constant that captures the individuals' utility--cost tradeoff, and $p_{\mathrm{mod}}$ indicates the modifiable coordinates. Putting everything together yields a distribution map
$\gD(\theta)=\gN(A\theta+m,\Sigma)$. We set $a=0.1$.

We introduce an exogenous component $P_t=\gN(m_t, \Sigma)$ with the same covariance matrix as the performative component, which models a time-varying population of non-strategic individuals that do not optimize their predicted score.
To capture ``hard'' and ``easy'' exogenous shifts, we consider two cases: in the first, $m_t$ is sampled independently from a bounded ball at each round; in the second, $m_t$ is sampled once at the beginning and held fixed across $t$. We study the effect of different values of $\alpha_t$; we compare polynomial schedules $\alpha_t=t^{-b}$ with $b \in \{0.25,0.5,1.0,2.0\}$ and the constant schedule $\alpha_t=1$ (or, equivalently, $b = 0$).

\begin{figure}
    \centering
    \includegraphics[width=0.99\linewidth]{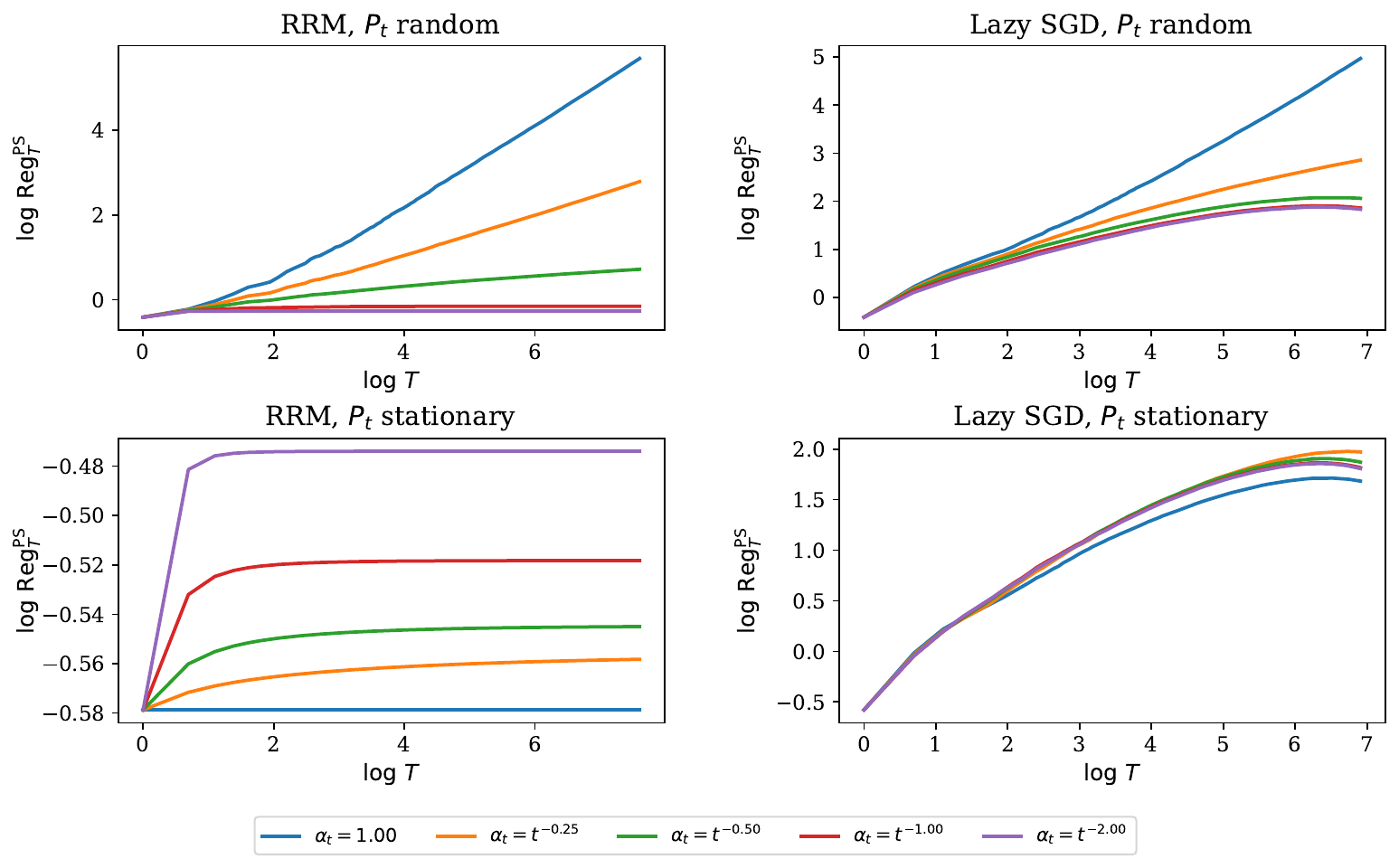}
    \caption{Performative stability regret of $\ARRM$ (left) and $\ARSGDL$ with $r = 1$ (right) in the credit scoring experiment, for varying $\alpha_t$. We consider a randomly varying (top) and a stationary (bottom) exogenous component $P_t$.
    }
    \label{fig:1}
\end{figure}

In \Cref{fig:1} we plot the stability regret of repeated risk minimization (RRM) and SGD with a lazy deployment schedule, for random and stationary $P_t$. The figure shows that the effect of the mixing schedule $\alpha_t$ depends on whether the exogenous component drifts.
With random shifts, faster decay of $\alpha_t$ suppresses the moving exogenous term and produces smaller stability regret for the displayed stability algorithms.
With fixed $P_t$, the ordering is reversed: keeping $\alpha_t$ larger for longer makes the observed distribution closer to a fixed, nonperformative component and yields smaller regret, $\alpha_t = 1$ being the most desirable scenario.
This behavior is consistent with the theoretical results, as well as the path length bound from Proposition \ref{prop:wassbound}. In the random $P_t$ case, the term $\sum_{t} \Delta_{t}^{\PS}$, which is driven by $\alpha_t \gW_1(P_t, P_{t+1})$, is the dominant term in the regret bound. Therefore, larger $\alpha_t$ leads to a larger regret. In the fixed $P_t$ case, the path length is zero, and thus the regret bound is dominated by the performative regret component, which decreases as $\alpha_t$ grows.

\section{Conclusion}
\label{sec:7}

We introduce partially performative prediction---a novel model for learning under distribution shifts that are both induced by performative feedback and exogenous variations. In this setting, we measure performance via regret relative to a sequence of performatively stable and performatively optimal points, rather than a single stable or optimal point, as in the standard performative prediction model. We prove dynamic regret bounds for several standard algorithms, including repeated risk minimization and variants of stochastic gradient descent.

Several questions remain unresolved, such as matching regret lower bounds or tighter upper bounds. As we briefly discussed, the regret rate of greedy SGD is worse than the rate of the other studied algorithms, and we do now know if this gap is fundamental. It would also be valuable to study performative stability regret and performative optimality regret beyond the presented model of partial performativity. These regret notions can be defined with respect to any time-varying map $\gD_t(\theta)$, thus allowing natural extensions to models beyond linear mixtures. Finally, our regularity assumptions on the loss (such as convexity) and the distribution map (such as sensitivity) are still fairly restrictive; finding ways to relax those assumptions remains an open question. 

\bibliographystyle{plainnat}
{\small%
	\bibliography{refs}%
}

\clearpage
\appendix


\section{Auxiliary lemmas}
\label{sec:auxlem}

\begin{restatable}{lemma}{lemrgdtelescope}
    \label{lem:rgd_telescope}
    Suppose that we have a recursive inequality of the form
    \begin{align*}
        F_{t+1} &\le \gamma_{t} F_{t} + \Delta_{t}.
    \end{align*}
    Then we have
    \begin{align}
        \sum_{t=1}^{T} (1 - \gamma_{t}) F_{t} &\le F_{1} - \gamma_{T} F_{T} + \sum_{t=1}^{T-1} \Delta_{t}.
        \label{eq:rgd_telescope}
    \end{align}
\end{restatable}

\begin{proof}
    Since we have
    \begin{align*}
        (1 - \gamma_{t}) F_{t} &\le F_{t} - F_{t+1} + \Delta_{t},
    \end{align*}
    the summation of the above yields
    \begin{align*}
        \sum_{t=1}^{T-1} (1 - \gamma_{t}) F_{t} &\le \sum_{t=1}^{T-1} \bigopen{ F_{t} - F_{t+1} } + \sum_{t=1}^{T-1} \Delta_{t} = F_{1} - F_{T} + \sum_{t=1}^{T-1} \Delta_{t}.
    \end{align*}
    We add $(1 - \gamma_{T}) F_{T}$ to both sides to obtain \cref{eq:rgd_telescope}.
\end{proof}


\begin{lemma}
    \label{lem:sublinearrecursion}
    Suppose that we have a recursive inequality of the form
    \begin{align*}
        F_{t+1} &\le \left( 1 - \frac{\alpha}{t + t_0} \right) F_{t} + H (t + t_0)^{-m-1}
    \end{align*}
    for all $t \in \sN$, with $t_0 > 0$ and $0 \le m < \alpha \le 1 + t_0$.
    Let
    \begin{align*}
        S = \max \left\{ F_1 (1 + t_0)^{m}, \, \frac{H}{\alpha - m} \right\}.
    \end{align*}
    Then we have $F_{t} \le S(t+t_0)^{-m}$ for all $t \in \sN$.
\end{lemma}

\begin{proof}
    For $t = 1$ we have $F_{1} \le S(1+t_0)^{-m}$ by definition.
    
    Suppose that $F_{t} \le S(t+t_0)^{-m}$.
    Then,
    \begin{align*}
        F_{t+1} &\le \left( 1 - \frac{\alpha}{t + t_0} \right) F_{t} + H (t + t_0)^{-m-1} \\
        &\le \left( 1 - \frac{\alpha}{t + t_0} \right) S (t + t_0)^{-m} + S (\alpha - m) (t + t_0)^{-1} \cdot (t + t_0)^{-m} \\
        &= S \left( 1 - \frac{\alpha}{t + t_0} + \frac{\alpha - m}{t + t_0} \right) (t+t_0)^{-m} \\
        &= S \left( 1 - \frac{m}{t + t_0} \right) (t+t_0)^{-m}.
    \end{align*}
    Since Bernoulli's inequality (for $m \ge 0$ and $\frac{1}{t+t_0} > -1$) implies
    \begin{align*}
        1 - \frac{m}{t + t_0} &\le \bigopen{1 + \frac{1}{t + t_0}}^{-m},
    \end{align*}
    we have
    \begin{align*}
        F_{t+1} &\le S \bigopen{1 + \frac{1}{t + t_0}}^{-m} (t+t_0)^{-m} = S (t + t_0 + 1)^{-m}
    \end{align*}
    as desired.
\end{proof}


\begin{lemma}[Lemma~B.5, \cite{mendlerdunner20}]
    \label{lem:unroll}
    For any $c \in (0, 1)$, $r > 0$, and $T \in \sN$,
    \begin{align*}
        \sum_{t=1}^{T} c^{T-t} t^{-r}
        &\le \frac{c^{T (1 - 2^{-1/r})}}{1-c} + \frac{2T^{-r}}{1-c}.
    \end{align*}
\end{lemma}

\begin{proof}
    (We rewrite the proof from \cite{mendlerdunner20} here for completeness.)
    We write
    \begin{align*}
        \sum_{t=1}^{T} c^{T-t} t^{-r} &= \sum_{t=1}^{M_{T}} c^{T-t} t^{-r} + \sum_{t=M_{T}+1}^{T} c^{T-t} t^{-r},
    \end{align*}
    where we define $M_{T} \coloneqq \max \{ m \in \sN : m^{-r} > 2 T^{-r} \}$ (or $M_T = 0$ if the set is empty for small $T$).
    Then we have
    \begin{align*}
        \sum_{t=1}^{M_{T}} c^{T-t} t^{-r} + \sum_{t=M_{T}+1}^{T} c^{T-t} t^{-r}
        &\le \sum_{t=1}^{M_{T}} c^{T-t} + 2 T^{-r} \cdot \sum_{t=M_{T}+1}^{T} c^{T-t} \\
        &\le \frac{c^{T - M_{T}}}{1 - c} + \frac{2 T^{-r}}{1 - c} \\
        &\le \frac{c^{T(1 - 2^{-1/r})}}{1 - c} + \frac{2 T^{-r}}{1 - c},
    \end{align*}
    where we use $M_{T} < 2^{-\frac{1}{r}} \cdot T$ by its definition.
\end{proof}


\begin{lemma}
    \label{lem:wassone}
    Let $a, b \in [0, 1]$ and $Q_1, Q_2, R_1, R_2 \in \gP(\gZ)$.
    Then we have
    \begin{align*}
        &\gW_1 \big( (1 - a) Q_1 + a R_1, (1 - b) Q_2 + b R_2 \big) \le (1 - a) \gW_1 (Q_1, Q_2) + a \gW_1 (R_1, R_2) + |a - b| \, \gW_1 (Q_2, R_2).
    \end{align*}
\end{lemma}

\begin{proof}
    Since $\gW_1$ is convex,
    \begin{align*}
        &\gW_1 \big( (1 - a) Q_1 + a R_1, (1 - a) Q_2 + a R_2 \big) \le (1 - a) \gW_1 (Q_1, Q_2) + a \gW_1 (R_1, R_2).
    \end{align*}
    We also have
    \begin{align*}
        \gW_1 \big( (1 - a) Q_2 + a R_2, (1 - b) Q_2 + b R_2 \big)
        &\le |b - a| \, \gW_1 (Q_2, R_2).
    \end{align*}
    Assuming $a \le b$ without loss of generality, this holds because the distributions $(1 - a) Q_2 + a R_2$ and $(1 - b) Q_2 + b R_2$ share a mass of at least $1-b$ from $Q_2$ and $a$ from $R_2$.
    Therefore the optimal transport plan from $Q_2$ to $R_2$ also yields a transport plan from $(1 - a) Q_2 + a R_2$ to $(1 - b) Q_2 + b R_2$ for the remaining mass $b-a = |b-a|$.
    Finally, applying triangle inequality for $\gW_1$ proves the given statement.
\end{proof}


\cref{lem:rgd_lislip} shows that sensitivity inherits Lipschitz properties of $\ell$ to the performative risk.

\begin{restatable}{lemma}{lemrgdlislip}
    \label{lem:rgd_lislip}
    Suppose that the map $\gD(\theta)$ is $\epsilon$-sensitive \eqref{eq:sensitivity} and the function $\ell(z,\theta)$ is $L_{\theta}$-Lipschitz in $\theta \in \Theta$ and $L_{z}$-Lipschitz in $z \in \gZ$.
    Then, $\mathrm{PR}_{t} (\theta)$ is $L_{t}$-Lipschitz, where $L_{t} = (1 - \alpha_{t}) \epsilon L_{z} + L_{\theta}$.
\end{restatable}

\begin{proof}
    (The proof is similar to that of Lemma~2 of \cite{jagadeesan22}.)
    Observe that $\mathrm{PR}_{t} (\theta) = \mathrm{R}_{\gD_{t} (\theta)} (\theta)$
    if we define
    $\mathrm{R}_{D} (\theta) \coloneqq \E_{Z \sim D} [\ell(Z, \theta)]$, and
    \begin{align*}
        &\mathrm{R}_{\gD_{t} (\theta)} (\theta) - \mathrm{R}_{\gD_{t} (\theta')} (\theta') = (1 - \alpha_{t}) \bigopen{ \mathrm{R}_{\gD (\theta)} (\theta) - \mathrm{R}_{\gD (\theta')} (\theta') } + \alpha_{t} \bigopen{ \mathrm{R}_{P_{t}} (\theta) - \mathrm{R}_{P_{t}} (\theta') },
    \end{align*}
    for any $\theta, \theta' \in \Theta$.
    Note that by Lipschitzness and sensitivity of $\gD$,
    \begin{align*}
        (1 - \alpha_{t}) \bigopen{ \mathrm{R}_{\gD (\theta)} (\theta) - \mathrm{R}_{\gD (\theta')} (\theta) } &\le (1 - \alpha_{t}) L_{z} \cdot \gW_1 (\gD (\theta), \gD (\theta')) \\
        &\le (1 - \alpha_{t}) \epsilon L_{z} \norm{\theta - \theta'},
    \end{align*}
    and also
    \begin{align*}
        &(1 - \alpha_{t}) \bigopen{ \mathrm{R}_{\gD (\theta')} (\theta) - \mathrm{R}_{\gD (\theta')} (\theta') } + \alpha_{t} \big( \mathrm{R}_{P_{t}} (\theta) - \mathrm{R}_{P_{t}} (\theta') \big) \\
        &= \bigopen{ (1 - \alpha_{t}) \mathrm{R}_{\gD (\theta')} (\theta) + \alpha_{t} \mathrm{R}_{P_{t}} (\theta) } - \bigopen{ (1 - \alpha_{t}) \mathrm{R}_{\gD (\theta')} (\theta') + \alpha_{t} \mathrm{R}_{P_{t}} (\theta') } \\
        &= \mathrm{R}_{\gD_{t} (\theta')} (\theta) - \mathrm{R}_{\gD_{t} (\theta')} (\theta') \\
        &\le L_{\theta} \norm{\theta - \theta'}.
    \end{align*}
    Taking the sum of the two and swapping $\theta, \theta'$ in the argument concludes the proof.
\end{proof}

\section{Proofs for repeated risk minimization}

We will define and frequently use the simplified notation $\gD_t^\PS \coloneqq \gD_t(\theta_t^\PS)$ throughout the remaining parts of the appendix.


\subsection{\texorpdfstring%
{Proof of \cref{lem:pislipschitz}}%
{Proof of Lemma 3.1}}
\label{subsec:pislipschitz}

Here we prove \cref{lem:pislipschitz}, restated below for the sake of readability.

\lempislipschitz*

\begin{proof}
    (The proof is similar to that of Theorem 3.5 of \cite{perdomo20}.
    For completeness, we provide the full proof adapted to our settings.)

    Define $g(\varphi) = \mathrm{R}_{\gD_{t} (\theta)} (\varphi)$ and $g'(\varphi) = \mathrm{R}_{\gD_{t} (\theta')} (\varphi)$ for fixed $\theta, \theta' \in \Theta$.
    Let $\gG (\theta) \coloneqq \argmin_{\varphi \in \Theta} g(\varphi)$ and $\gG (\theta') \coloneqq \argmin_{\varphi \in \Theta} g'(\varphi)$.
    Then the first-order optimality conditions yield
    \begin{align*}
        \nabla g(\gG(\theta))^{\top} \big( \gG(\theta') - \gG(\theta) \big) &\ge 0, \qquad 
        \nabla g'(\gG(\theta'))^{\top} \big( \gG(\theta) - \gG(\theta') \big) \ge 0.
    \end{align*}
    By strong convexity, we have
    \begin{align*}
        g(\gG(\theta)) - g(\gG(\theta'))
        &\ge \nabla g(\gG(\theta'))^{\top} \big( \gG(\theta) - \gG(\theta') \big) + \frac{\mu}{2} \| \gG(\theta) - \gG(\theta') \|^2, \\
        g(\gG(\theta')) - g(\gG(\theta))
        &\ge \nabla g(\gG(\theta))^{\top} \big( \gG(\theta') - \gG(\theta) \big) + \frac{\mu}{2} \| \gG(\theta) - \gG(\theta') \|^2 \\
        &\ge \frac{\mu}{2} \| \gG(\theta) - \gG(\theta') \|^2,
    \end{align*}
    where the last inequality comes from the first-order optimality condition of $\gG (\theta) = \argmin_{\varphi \in \Theta} g(\varphi)$.
    Therefore we have
    \begin{align*}
        - \nabla g(\gG(\theta'))^{\top} \big( \gG(\theta) - \gG(\theta') \big)
        &\ge \mu \| \gG(\theta) - \gG(\theta') \|^2.
    \end{align*}
    Note that
    \begin{align*}
        - \nabla g(\gG(\theta'))^{\top} \big( \gG(\theta) - \gG(\theta') \big) 
        &\le - \nabla g(\gG(\theta'))^{\top} \big( \gG(\theta) - \gG(\theta') \big) + \nabla g'(\gG(\theta'))^{\top} \big( \gG(\theta) - \gG(\theta') \big) \\
        &\le \| \nabla g(\gG(\theta')) - \nabla g'(\gG(\theta')) \| \cdot \| \gG(\theta) - \gG(\theta') \|,
    \end{align*}
    where the first line comes from the first-order optimality condition of $\gG (\theta') = \argmin_{\varphi \in \Theta} g'(\varphi)$.
    
    Since $\nabla g(\varphi) = \E_{Z \sim \gD_{t} (\theta)} [\nabla_{\theta} \ell(Z, \varphi)]$ and $\nabla g'(\varphi) = \E_{Z \sim \gD_{t} (\theta')} [\nabla_{\theta} \ell(Z, \varphi)]$, we can observe that $\beta_{z}$-smoothness and the Kantorovich-Rubinstein duality of the $\gW_1$ distance implies that
    \begin{align*}
        \| \nabla g(\gG(\theta')) - \nabla g'(\gG(\theta')) \| 
        &\le \norm{ \E_{Z \sim \gD_{t} (\theta)} [\nabla_\theta \ell (Z, \gG(\theta'))] - \E_{Z \sim \gD_{t} (\theta')} [\nabla_\theta \ell (Z, \gG(\theta'))] } \\
        &\le \beta_{z} \cdot \gW_1 \big( \gD_{t} (\theta), \gD_{t} (\theta') \big) 
    \end{align*}
    and by $\epsilon$-sensitivity of $\gD (\theta)$, we have
    \begin{align*}
        \gW_1 (\gD_{t} (\theta), \gD_{t} (\theta'))
        &= \gW_1 \bigopen{ (1 - \alpha_{t}) \gD (\theta) + \alpha_{t} P_{t}, (1 - \alpha_{t}) \gD (\theta') + \alpha_{t} P_{t} } \\
        &\le (1 - \alpha_{t}) \gW_1 \bigopen{ \gD (\theta), \gD (\theta') } \\
        &\le (1 - \alpha_{t}) \epsilon \norm{ \theta - \theta' }.
    \end{align*}
    Therefore we have
    \begin{align*}
        \mu \| \gG(\theta) - \gG(\theta') \|^2
        &\le - \nabla g(\gG(\theta'))^{\top} \big( \gG(\theta) - \gG(\theta') \big) \\
        &\le (1 - \alpha_{t}) \epsilon \beta_{z} \cdot \norm{ \theta - \theta' } \cdot \| \gG(\theta) - \gG(\theta') \|.
    \end{align*}
    If $\gG(\theta) = \gG(\theta')$, the desired inequality is trivial.
    Otherwise, dividing by $\| \gG(\theta) - \gG(\theta') \|$ gives
    \begin{align*}
        \| \gG(\theta) - \gG(\theta') \| &\le \frac{(1 - \alpha_{t}) \epsilon \beta_{z}}{\mu} \| \theta - \theta' \|,
    \end{align*}
    for any $\theta, \theta'$.
    This immediately implies
    \begin{align*}
        \norm{\theta_{t+1} - \theta_{t}^{\PS}} &\le \frac{(1 - \alpha_{t}) \epsilon \beta_{z}}{\mu} \norm{ \theta_{t} - \theta_{t}^{\PS} },
    \end{align*}
    as desired, as $\theta_{t+1} = \mathcal{G} (\theta_{t})$ and $\theta_{t}^{\PS}$ is defined such that $\theta_{t}^{\PS} = \mathcal{G} (\theta_{t}^{\PS})$.
\end{proof}


\subsection{\texorpdfstring%
{Proof of \cref{thm:rrm}}%
{Proof of Theorem 3.2}}
\label{subsec:rrm}

Here we prove \cref{thm:rrm}, restated below for the sake of readability.

\thmrrm*

\begin{proof}
    \cref{lem:pislipschitz} immediately implies
    \begin{align*}
        \norm{\theta_{t+1} - \theta_{t+1}^{\PS}}
        &\le \norm{\theta_{t+1} - \theta_{t}^{\PS}} + \norm{\theta_{t}^{\PS} - \theta_{t+1}^{\PS}} \le \frac{(1 - \alpha_{t}) \epsilon \beta_{z}}{\mu} \norm{ \theta_{t} - \theta_{t}^{\PS} } + \Delta_{t}^{\PS}. 
    \end{align*}
    Therefore we can use \cref{lem:rgd_telescope} as
    \begin{align*}
        \sum_{t=1}^{T} (1 - \gamma_{t}) \norm{\theta_{t} - \theta_{t}^{\PS}} &\le \norm{\theta_{1} - \theta_{1}^{\PS}} - \gamma_{T} \norm{\theta_{T} - \theta_{T}^{\PS}} + \sum_{t=1}^{T-1} \Delta_{t}^{\PS},
    \end{align*}
    with $\gamma_{t} = \frac{(1 - \alpha_{t}) \epsilon \beta_{z}}{\mu}$.

    For any choice of nonnegative weights satisfying $w_{t} \le \frac{1 - \gamma_{t}}{L}$, we can conclude that
    \begin{align*}
        \sum_{t=1}^{T} w_{t} \bigopen{ \mathrm{R}_{\gD_{t}} (\theta_{t}) - \mathrm{R}_{\gD_{t}^{\PS}} (\theta_{t}^{\PS}) } &\le \sum_{t=1}^{T} w_{t} L \norm{ \theta_{t} - \theta_{t}^{\PS} } \\
        &\le \sum_{t=1}^{T} (1 - \gamma_{t}) \norm{ \theta_{t} - \theta_{t}^{\PS} } \\
        &\le \norm{ \theta_{1} - \theta_{1}^{\PS} } + \sum_{t=1}^{T-1} \Delta_{t}^{\PS}.
    \end{align*}
    
    Finally, we can observe that $\gamma_{t}$ is \textit{decreasing} in $\alpha_{t}$, and thus the expression $\frac{1 - \gamma_{t}}{L}$ is an \textit{increasing} function of $\alpha_{t}$.
    Therefore, if we let $\alpha_{\min} = \min_{t \in [T]} \alpha_{t}$ and choose 
    \begin{align*}
        \gamma = \frac{(1 - \alpha_{\min}) \epsilon \beta_{z}}{\mu},
    \end{align*}
    then choosing a constant $w_{t} \equiv w = \frac{1 - \gamma}{L}$ satisfies $w \le \frac{1 - \gamma_{t}}{L}$ for all $t$.
    Then we have
    \begin{align*}
        \mathrm{Reg}_{T}^{ \PS } (\gA_{\mathrm{RRM}})
        &= \frac{1}{w} \sum_{t=1}^{T} w \bigopen{ \mathrm{R}_{\gD_{t}} (\theta_{t}) - \mathrm{R}_{\gD_{t}^{\PS}} (\theta_{t}^{\PS}) } \le \frac{L}{1 - \gamma} \cdot \bigopen{ \norm{ \theta_{1} - \theta_{1}^{\PS} } + \sum_{t=1}^{T-1} \Delta_{t}^{\PS} },
    \end{align*}
    as in \eqref{eq:rrm}.
\end{proof}


\subsection{\texorpdfstring%
{Proof of \cref{prop:wassbound}}%
{Proof of Proposition 3.4}}
\label{subsec:wassbound}

Here we prove \cref{prop:wassbound}, restated below for the sake of readability.

\propwassbound*

\begin{proof}
    Let us define $\theta \coloneqq \theta_{t}^{\PS}$, $\theta' \coloneqq \theta_{t+1}^{\PS}$ and $\gD_t^\PS \coloneqq \gD_t(\theta_t^\PS)$, $\gD_{t+1}^\PS \coloneqq \gD_{t+1}(\theta_{t+1}^\PS)$ for simplicity.
    By the definition of performative stability, $\theta_t^\PS$ and $\theta_{t+1}^\PS$ are the constrained minimizers of $\mathrm{R}_{\gD_t^\PS}$ and $\mathrm{R}_{\gD_{t+1}^\PS}$, respectively.
    
    The first-order optimality conditions yield
    \begin{align*}
        \nabla \mathrm{R}_{\gD_t^\PS}(\theta)^{\top} (\theta' - \theta) &\ge 0, \qquad
        \nabla \mathrm{R}_{\gD_{t+1}^\PS}(\theta')^{\top} (\theta - \theta') \ge 0.
    \end{align*}
    If $\theta=\theta'$, the desired inequality is trivial.
    Otherwise, strong monotonicity of $\nabla \mathrm{R}_{\gD_t^\PS}$ and the above inequalities imply
    \begin{align*}
        \mu \norm{\theta - \theta'}^{2}
        &\le \bigopen{\nabla \mathrm{R}_{\gD_t^\PS}(\theta') - \nabla \mathrm{R}_{\gD_t^\PS}(\theta)}^{\top} (\theta' - \theta) \\
        &= \bigopen{\nabla \mathrm{R}_{\gD_t^\PS}(\theta') - \nabla \mathrm{R}_{\gD_{t+1}^\PS}(\theta')}^{\top} (\theta' - \theta) \\
        &\phantom{={}} + \nabla \mathrm{R}_{\gD_{t+1}^\PS}(\theta')^{\top} (\theta' - \theta)
        - \nabla \mathrm{R}_{\gD_t^\PS}(\theta)^{\top} (\theta' - \theta) \\
        &\le \bigopen{\nabla \mathrm{R}_{\gD_t^\PS}(\theta') - \nabla \mathrm{R}_{\gD_{t+1}^\PS}(\theta')}^{\top} (\theta' - \theta) \\
        &\le \norm{\nabla \mathrm{R}_{\gD_t^\PS}(\theta') - \nabla \mathrm{R}_{\gD_{t+1}^\PS}(\theta')} \cdot \norm{\theta - \theta'}.
    \end{align*}
    Dividing by $\norm{\theta - \theta'}$ gives
    \begin{align*}
        \norm{\theta - \theta'} \le \frac{1}{\mu}\norm{\nabla \mathrm{R}_{\gD_t^{\PS}}(\theta') - \nabla \mathrm{R}_{\gD_{t+1}^{\PS}}(\theta')},
    \end{align*}
    and plugging back in $\theta \coloneqq \theta_{t}^{\PS}$ and $\theta' \coloneqq \theta_{t+1}^{\PS}$,
    \begin{align*}
        \Delta_t^{\PS} &= \norm{\theta_{t}^{\PS} - \theta_{t+1}^{\PS}} \le \frac{1}{\mu}\norm{\nabla \mathrm{R}_{\gD_{t}^{\mathrm{PS}}} (\theta_{t+1}^{\mathrm{PS}}) - \nabla \mathrm{R}_{\gD_{t+1}^{\mathrm{PS}}} (\theta_{t+1}^{\mathrm{PS}})}.
    \end{align*}
    By $\beta_z$-smoothness, Kantorovich-Rubinstein duality, and \cref{lem:wassone}, we have
    \begin{align*}
        &\frac{1}{\mu} \norm{\nabla \mathrm{R}_{\gD_{t}^{\mathrm{PS}}} (\theta_{t+1}^{\mathrm{PS}}) - \nabla \mathrm{R}_{\gD_{t+1}^{\mathrm{PS}}} (\theta_{t+1}^{\mathrm{PS}})} \le \frac{\beta_{z}}{\mu} \cdot \gW_1 (\gD_{t}^{\mathrm{PS}}, \gD_{t+1}^{\mathrm{PS}}) \\
        &= \frac{\beta_{z}}{\mu} \cdot \gW_1 \left( (1 - \alpha_{t}) \gD (\theta_{t}^{\mathrm{PS}}) + \alpha_{t} P_{t}, (1 - \alpha_{t+1}) \gD (\theta_{t+1}^{\mathrm{PS}}) + \alpha_{t+1} P_{t+1} \right) \\
        &\le \frac{\beta_{z}}{\mu} \big( (1 - \alpha_{t}) \gW_1 (\gD (\theta_{t}^{\mathrm{PS}}), \gD (\theta_{t+1}^{\mathrm{PS}})) + \alpha_{t} \gW_1 (P_{t}, P_{t+1}) 
        + \lvert \alpha_{t} - \alpha_{t+1} \rvert \cdot \gW_1 (\gD (\theta_{t+1}^{\mathrm{PS}}), P_{t+1}) \big).
    \end{align*}
    Now, $\gW_1 (\gD (\theta_{t}^{\mathrm{PS}}), \gD (\theta_{t+1}^{\mathrm{PS}})) \le \epsilon \| \theta_t^{\PS} - \theta_{t+1}^{\PS} \| = \epsilon \Delta_{t}^{\PS}$ yields
    \begin{align*}
        \Delta_{t}^{\PS}
        \le \frac{\beta_{z}}{\mu} \big( (1 - \alpha_{t}) \epsilon \Delta_{t}^{\PS}
        &+ \alpha_{t} \gW_1 (P_{t}, P_{t+1}) + \lvert \alpha_{t} - \alpha_{t+1} \rvert \cdot \gW_1 (\gD (\theta_{t+1}^{\mathrm{PS}}), P_{t+1}) \big).
    \end{align*}
    Rearranging terms, we have
    \begin{align*}
        &\bigopen{1 - \frac{\beta_{z}}{\mu} \cdot (1 - \alpha_{t}) \epsilon} \Delta_{t}^{\PS} \le \frac{\beta_{z}}{\mu} \bigopen{ \alpha_{t} \gW_1 (P_{t}, P_{t+1}) + \lvert \alpha_{t} - \alpha_{t+1} \rvert \cdot \gW_1 (\gD (\theta_{t+1}^{\mathrm{PS}}), P_{t+1}) },
    \end{align*}
    and $\gW_1( \gD(\theta_{t+1}^{\PS}), P_{t+1}) \le C$ yields the given statement.
\end{proof}

\section{Proofs for repeated gradient descent}


\subsection{One-step inequality for RGD}
\label{subsec:rgd_onestep}

Here we prove \cref{lem:rgd_onestep}, which provides a one-step inequality for RGD.

\begin{restatable}{lemma}{lemrgdonestep}
    \label{lem:rgd_onestep}
    Suppose that $\gD(\theta)$ is $\epsilon$-sensitive and $\ell(z, \theta)$ is $\mu$-strongly convex \eqref{eq:strong_convexity} and $\beta_{\theta}$-smooth \eqref{eq:smoothness_theta} in $\theta$ and $\beta_z$-smooth \eqref{eq:smoothness_z} in $z$.
    Then the iterates of $\gA_{\mathrm{RGD}}$ satisfy
    \begin{align}
        \norm{\theta_{t+1} - \theta_{t+1}^{\PS}} &\le \gamma_{t} \norm{\theta_{t} - \theta_{t}^{\PS}} + \Delta_{t}^{\PS},
        \label{eq:rgd_onestep}
    \end{align}
    where $\gamma_{t} \coloneqq 1 - \eta_{t} (\mu - (1 - \alpha_{t}) \epsilon \beta_{z}) + \eta_{t}^{2} (\beta_{\theta}^{2} + (1 - \alpha_{t})^{2} \epsilon^{2} \beta_{z}^{2})$.
\end{restatable}

\begin{proof}
    The proof is similar to that of Proposition~2.6 in \cite{mendlerdunner20}, but our proof relies only on the first-order optimality conditions in the definition of $\theta_t^\PS$, rather than assuming $\E \mathrm{R}_{\gD_t^\PS}(\theta_t^\PS) = 0$, which in turn requires $\theta_t^\PS$ to lie in the interior of $\Theta$, as in \cite{mendlerdunner20}.
    
    By definition of $\theta_t^\PS$, we can observe that
    \begin{align*}
        \nabla \mathrm{R}_{\gD_t^\PS}(\theta_t^\PS)^\top(\vartheta-\theta_t^\PS) \ge 0,
        \qquad \forall \vartheta\in\Theta.
    \end{align*}
    For any closed convex $\Theta$, a projection $y = \Pi_\Theta(u)$ of $u$ must satisfy
    \begin{align*}
        (y-u)^\top(\vartheta-y) \ge 0,
        \qquad \forall \vartheta \in \Theta.
    \end{align*}
    for $y=\theta_t^\PS$ and $u = \theta_t^\PS-\eta_t\nabla \mathrm{R}_{\gD_t^\PS}(\theta_t^\PS)$, this is equivalent to
    \begin{align*}
        \eta_t\nabla \mathrm{R}_{\gD_t^\PS}(\theta_t^\PS)^\top
        (\vartheta-\theta_t^\PS)
        \ge 0,
        \qquad \forall \vartheta\in\Theta,
    \end{align*}
    which is true by the definition of $\theta_t^\PS$.
    This shows that $y$ is a projection of $u$ onto $\Theta$, i.e.,
    \begin{align*}
        \theta_t^\PS &= \Pi_\Theta \bigopen{\theta_t^\PS - \eta_t \nabla \mathrm{R}_{\gD_t^\PS}(\theta_t^\PS)}
    \end{align*}
    for any $\eta_t>0$.
    Therefore, by nonexpansiveness of $\Pi_\Theta$,
    \begin{align*}
        \norm{\theta_{t+1} - \theta_{t+1}^{\PS}} &\le \norm{\theta_{t+1} - \theta_{t}^{\PS}} + \norm{\theta_{t}^{\PS} - \theta_{t+1}^{\PS}} \\
        &= \bignorm{\Pi_\Theta \bigopen{\theta_t - \eta_t \nabla \mathrm{R}_{\gD_t}(\theta_t)} - \Pi_\Theta \bigopen{\theta_t^\PS - \eta_t \nabla \mathrm{R}_{\gD_t^\PS}(\theta_t^\PS)}} + \Delta_t^{\PS} \\
        &\le \bignorm{\theta_t - \theta_t^\PS - \eta_t \bigopen{\nabla \mathrm{R}_{\gD_t}(\theta_t) - \nabla \mathrm{R}_{\gD_t^\PS}(\theta_t^\PS)}} + \Delta_t^{\PS}.
    \end{align*}
    Squaring and expanding the first norm gives
    \begin{align}
        \begin{aligned}
        \bignorm{\theta_t - \theta_t^\PS - \eta_t \bigopen{\nabla \mathrm{R}_{\gD_t}(\theta_t) - \nabla \mathrm{R}_{\gD_t^\PS}(\theta_t^\PS)}}^2 &= \norm{\theta_t - \theta_t^\PS}^2 - 2\eta_t \bigopen{\nabla \mathrm{R}_{\gD_t}(\theta_t) - \nabla \mathrm{R}_{\gD_t^\PS}(\theta_t^\PS)}^\top(\theta_t-\theta_t^\PS) \\
        &\phantom{={}} + \eta_t^2 \bignorm{\nabla \mathrm{R}_{\gD_t}(\theta_t) - \nabla \mathrm{R}_{\gD_t^\PS}(\theta_t^\PS)}^2.
        \end{aligned}
        \label{eq:rgd_basic1}
    \end{align}
    The inner product term of \eqref{eq:rgd_basic1} can be decomposed as
    \begin{align}
        \begin{aligned}
        \bigopen{\nabla \mathrm{R}_{\gD_t}(\theta_t) - \nabla \mathrm{R}_{\gD_t^\PS}(\theta_t^\PS)}^\top(\theta_t-\theta_t^\PS) &= \bigopen{ \E_{Z \sim \gD_{t}} [\nabla_{\theta} \ell (Z, \theta_{t})] - \E_{Z \sim \gD_{t}^{\PS}} [\nabla_{\theta} \ell (Z, \theta_{t}^{\PS})] }^{\top} ( \theta_{t} - \theta_{t}^{\PS} ) \\
        &= \bigopen{ \E_{Z \sim \gD_{t}} [\nabla_{\theta} \ell (Z, \theta_{t})] - \E_{Z \sim \gD_{t}^{\PS}} [\nabla_{\theta} \ell (Z, \theta_{t})] }^{\top} ( \theta_{t} - \theta_{t}^{\PS} ) \\
        &\phantom{{}\ge{}} {}+{} \bigopen{ \E_{Z \sim \gD_{t}^{\PS}} [\nabla_{\theta} \ell (Z, \theta_{t})] - \E_{Z \sim \gD_{t}^{\PS}} [\nabla_{\theta} \ell (Z, \theta_{t}^{\PS})] }^{\top} ( \theta_{t} - \theta_{t}^{\PS} ).
        \end{aligned}
        \label{eq:rgd_basic2}
    \end{align}
    Define the auxiliary function
    \begin{align*}
        \xi_{t} (z) &\coloneqq \nabla_{\theta} \ell (z, \theta_{t})^{\top} ( \theta_{t} - \theta_{t}^{\PS} ),
    \end{align*}
    which is $\beta_{z} \norm{\theta_{t} - \theta_{t}^{\PS}}$-Lipschitz in $z$.
    The first term of the RHS of \eqref{eq:rgd_basic2} satisfies
    \begin{align}
        \begin{aligned}
        \bigopen{ \E_{Z \sim \gD_{t}} [\nabla_{\theta} \ell (Z, \theta_{t})] - \E_{Z \sim \gD_{t}^{\PS}} [\nabla_{\theta} \ell (Z, \theta_{t})] }^{\top} ( \theta_{t} - \theta_{t}^{\PS} )
        &= \E_{Z \sim \gD_{t}} [\xi_{t} (Z)] - \E_{Z \sim \gD_{t}^{\PS}} [\xi_{t} (Z)] \\
        &\ge - \beta_{z} \norm{\theta_{t} - \theta_{t}^{\PS}} \cdot \gW_1 (\gD_{t}, \gD_{t}^{\PS}) \\
        &\ge - \beta_{z} \norm{\theta_{t} - \theta_{t}^{\PS}} \cdot (1 - \alpha_{t}) \epsilon \norm{\theta_{t} - \theta_{t}^{\PS}} \\
        &= - (1 - \alpha_{t}) \epsilon \beta_{z} \norm{\theta_{t} - \theta_{t}^{\PS}}^2,
        \end{aligned}
        \label{eq:rgd_basic3}
    \end{align}
    where the first inequality follows from the fact that $\xi_{t} (z)$ is Lipschitz, and the second inequality is by $(1 - \alpha_{t}) \epsilon$-sensitivity of $\gD_{t} (\theta)$.
    The second term of the RHS of \eqref{eq:rgd_basic2} satisfies
    \begin{align}
        \begin{aligned}
        \bigopen{ \E_{Z \sim \gD_{t}^{\PS}} [\nabla_{\theta} \ell (Z, \theta_{t})] - \E_{Z \sim \gD_{t}^{\PS}} [\nabla_{\theta} \ell (Z, \theta_{t}^{\PS})] }^{\top} ( \theta_{t} - \theta_{t}^{\PS} ) &\ge \mu \norm{\theta_{t} - \theta_{t}^{\PS}}^2,
        \end{aligned}
        \label{eq:rgd_basic4}
    \end{align}
    by $\mu$-strong monotonicity of the gradients.
    Plugging \eqref{eq:rgd_basic3}, \eqref{eq:rgd_basic4} into \eqref{eq:rgd_basic2}, we have
    \begin{align}
        \begin{aligned}
        \bigopen{\nabla \mathrm{R}_{\gD_t}(\theta_t) - \nabla \mathrm{R}_{\gD_t^\PS}(\theta_t^\PS)}^\top(\theta_t-\theta_t^\PS) &= \bigopen{\nabla \mathrm{R}_{\gD_t}(\theta_t) - \nabla \mathrm{R}_{\gD_t^\PS}(\theta_t)}^\top (\theta_t-\theta_t^\PS) \\
        &\phantom{={}} + \bigopen{\nabla \mathrm{R}_{\gD_t^\PS}(\theta_t) - \nabla \mathrm{R}_{\gD_t^\PS}(\theta_t^\PS)}^\top(\theta_t-\theta_t^\PS) \\
        &\ge \bigopen{\mu - (1-\alpha_t)\epsilon\beta_z}\norm{\theta_t-\theta_t^\PS}^2. \phantom{\frac{a}{a}}
        \end{aligned}
        \label{eq:rgd_basic5}
    \end{align}
    The gradient-difference norm term of \eqref{eq:rgd_basic1} satisfies
    \begin{align}
        \begin{aligned}
        \bignorm{\nabla \mathrm{R}_{\gD_t}(\theta_t) - \nabla \mathrm{R}_{\gD_t^\PS}(\theta_t^\PS)}^2
        &\le 2 \bignorm{\nabla \mathrm{R}_{\gD_{t}} (\theta_{t}) - \nabla \mathrm{R}_{\gD_{t}} (\theta_{t}^{\PS})}^{2} + 2 \bignorm{\nabla \mathrm{R}_{\gD_{t}} (\theta_{t}^{\PS}) - \nabla \mathrm{R}_{\gD_{t}^{\PS}} (\theta_{t}^{\PS})}^{2} \\
        &\le 2 \beta_{\theta}^{2} \norm{\theta_{t} - \theta_{t}^{\PS}}^2 + 2 (1 - \alpha_{t})^{2} \epsilon^{2} \beta_{z}^{2} \norm{\theta_{t} - \theta_{t}^{\PS}}^2 \\
        &= 2 (\beta_{\theta}^{2} + (1 - \alpha_{t})^{2} \epsilon^{2} \beta_{z}^{2}) \norm{\theta_{t} - \theta_{t}^{\PS}}^2,
        \end{aligned}
        \label{eq:rgd_basic6}
    \end{align}
    by smoothness and sensitivity.

    Plugging \eqref{eq:rgd_basic5} and \eqref{eq:rgd_basic6} into \eqref{eq:rgd_basic1}, we have
    \begin{align*}
        \bignorm{\theta_t - \theta_t^\PS - \eta_t \bigopen{\nabla \mathrm{R}_{\gD_t}(\theta_t) - \nabla \mathrm{R}_{\gD_t^\PS}(\theta_t^\PS)}}^2  &\le \bigopen{ 1 - 2 \eta_{t} (\mu - (1 - \alpha_{t}) \epsilon \beta_{z}) + 2 \eta_{t}^{2} (\beta_{\theta}^{2} + (1 - \alpha_{t})^{2} \epsilon^{2} \beta_{z}^{2}) } \norm{\theta_{t} - \theta_{t}^{\PS}}^2 \\
        &\le \bigopen{ 1 - \eta_{t} (\mu - (1 - \alpha_{t}) \epsilon \beta_{z}) + \eta_{t}^{2} (\beta_{\theta}^{2} + (1 - \alpha_{t})^{2} \epsilon^{2} \beta_{z}^{2}) }^2 \norm{\theta_{t} - \theta_{t}^{\PS}}^2 \\
        &= \gamma_{t}^2 \norm{\theta_{t} - \theta_{t}^{\PS}}^2,
    \end{align*}
    where the second inequality uses $1-2x \le (1-x)^2$ for simplification.
    Therefore we have
    \begin{align*}
        \norm{\theta_{t+1} - \theta_{t+1}^{\PS}} &\le \gamma_{t} \norm{\theta_{t} - \theta_{t}^{\PS}} + \Delta_{t}^{\PS},
    \end{align*}
    which is equivalent to \eqref{eq:rgd_onestep}.
\end{proof}


\subsection{\texorpdfstring%
{Proof of \cref{thm:rgd}}%
{Proof of Theorem 4.1}}
\label{subsec:rgd}

Here we prove \cref{thm:rgd}, restated below for the sake of readability.

\thmrgd*

\begin{proof}
    By \cref{lem:rgd_onestep} (proven in \cref{subsec:rgd_onestep}), we have
    \begin{align*}
        \norm{\theta_{t+1} - \theta_{t+1}^{\PS}} &\le \gamma_{t} \norm{\theta_{t} - \theta_{t}^{\PS}} + \Delta_{t}^{\PS},
    \end{align*}
    where $\gamma_{t} \coloneqq 1 - \eta_{t} (\mu - (1 - \alpha_{t}) \epsilon \beta_{z}) + \eta_{t}^{2} (\beta_{\theta}^{2} + (1 - \alpha_{t})^{2} \epsilon^{2} \beta_{z}^{2})$.
    
    \cref{lem:rgd_telescope} yields
    \begin{align*}
        \sum_{t=1}^{T} (1 - \gamma_{t}) \norm{ \theta_{t} - \theta_{t}^{\PS} }
        &\le \norm{ \theta_{1} - \theta_{1}^{\PS} } - \gamma_{T} \norm{ \theta_{T} - \theta_{T}^{\PS} } + \sum_{t=1}^{T-1} \Delta_{t}^{\PS} \\
        &\le \norm{ \theta_{1} - \theta_{1}^{\PS} } + \sum_{t=1}^{T-1} \Delta_{t}^{\PS},
    \end{align*}
    and thus
    \begin{align*}
        \sum_{t=1}^{T} w_{t} \bigopen{ \mathrm{R}_{\gD_{t}} (\theta_{t}) - \mathrm{R}_{\gD_{t}^{\PS}} (\theta_{t}^{\PS}) } &\le \sum_{t=1}^{T} w_{t} L \norm{ \theta_{t} - \theta_{t}^{\PS} } \\
        &\le \sum_{t=1}^{T} (1 - \gamma_{t}) \norm{ \theta_{t} - \theta_{t}^{\PS} } \\
        &\le \norm{ \theta_{1} - \theta_{1}^{\PS} } + \sum_{t=1}^{T-1} \Delta_{t}^{\PS},
    \end{align*}
    for any choice of weights satisfying $w_{t} \le \frac{1 - \gamma_{t}}{L}$.
    Therefore, if we set 
    \begin{align*}
        \overline{\gamma} = 1 - \frac{(\mu - (1 - \alpha_{\min}) \epsilon \beta_{z})^2}{4 ( \beta_{\theta}^{2} + (1 - \alpha_{\min})^{2} \epsilon^{2} \beta_{z}^{2} )},
    \end{align*}
    then choosing a constant $w_{t} \equiv w = \frac{1 - \overline{\gamma}}{L}$ satisfies $w \le \frac{1 - \gamma_{t}}{L}$ for all $t$.
    Then we have
    \begin{align*}
        \mathrm{Reg}_{T}^{ \PS } (\gA_{\mathrm{RGD}})
        &= \frac{1}{w} \sum_{t=1}^{T} w \bigopen{ \mathrm{R}_{\gD_{t}} (\theta_{t}) - \mathrm{R}_{\gD_{t}^{\PS}} (\theta_{t}^{\PS}) } \\
        &\le \frac{1}{w} \bigopen{\norm{ \theta_{1} - \theta_{1}^{\PS} } + \sum_{t=1}^{T-1} \Delta_{t}^{\PS}},
    \end{align*}
    which is equivalent to \eqref{eq:rgd}.
\end{proof}

\section{Proofs for greedy SGD}


\subsection{One-step inequality for greedy SGD}
\label{subsec:sgdonestep}

Here we prove \cref{lem:sgdonestep}, which provides a one-step inequality for greedy SGD.

\begin{restatable}{lemma}{lemsgdonestep}
    \label{lem:sgdonestep}
    Suppose that $\gD(\theta)$ is $\epsilon$-sensitive and $\ell(z, \theta)$ is $\mu$-strongly convex \eqref{eq:strong_convexity} and $\beta_{\theta}$-smooth \eqref{eq:smoothness_theta} in $\theta$ and $\beta_z$-smooth \eqref{eq:smoothness_z} in $z$.
    Also, suppose that we have quadratic-bounded variance (Ass.~\ref{ass:bv}).
    Then, the iterates of $\gA_{\mathrm{RSGD\text{-}G}}$ satisfy
    \begin{align}
        \begin{aligned}
        \E [\norm{ \theta_{t+1} - \theta_{t+1}^{\PS} }^2]
        \le \gamma_{t}' \E [\norm{ \theta_{t} - \theta_{t}^{\PS} }^2]
        &+ 2 \Delta_{t}^{\PS} \sqrt{ \gamma_{t}' \E [\norm{ \theta_{t} - \theta_{t}^{\PS} }^2] } + 2 \eta_{t}^2 \sigma^2 + 2 (\Delta_{t}^{\PS})^2,
        \end{aligned}
        \label{eq:sgdonestep}
    \end{align}
    where $\gamma_t' \coloneqq 1 - 2 \eta_{t} \bigopen{\mu - (1 - \alpha_{t}) \epsilon \beta_{z}} + \eta_{t}^{2} C_V^2 (1 + \frac{(1 - \alpha_{t}) \epsilon \beta_{z}}{\mu})^2$ and $\eta_{t} > 0$ is small enough so that $\gamma_t' \ge 0$.
\end{restatable}

\begin{proof}
    We have
    \begin{align}
        \begin{aligned}
        \norm{\theta_{t+1} - \theta_{t+1}^{\PS}}
        &\le \norm{\theta_{t+1} - \theta_{t}^{\PS}} + \norm{\theta_{t}^{\PS} - \theta_{t+1}^{\PS}} \\
        &= \bignorm{\Pi_{\Theta} ( \theta_{t} - \eta_{t} \nabla_{\theta} \ell (Z_{t}, \theta_{t}) ) - \Pi_{\Theta} (\theta_{t}^{\PS})} + \norm{\theta_{t}^{\PS} - \theta_{t+1}^{\PS}} \\
        &\le \norm{ \theta_{t} - \theta_{t}^{\PS} - \eta_{t} \nabla_{\theta} \ell (Z_{t}, \theta_{t}) } + \Delta_{t}^{\PS},
        \end{aligned}
        \label{eq:sgdproj}
    \end{align}
    since $\theta_t^\PS=\Pi_\Theta(\theta_t^\PS)$.
    Also, under \cref{ass:bv}, we have (conditioned on $\theta_{t}$)
    \begin{align}
        \begin{aligned}
        \E_{Z_{t} \sim \gD_{t}} [\norm{\nabla_{\theta} \ell (Z_{t}, \theta_{t})}^2] &\le \sigma^2 + C_V^2 \norm{\theta_{t} - \theta_{\gD_{t}}^{\star}}^2 \\
        &\le \sigma^2 + C_V^2 \bigopen{ \norm{\theta_{t} - \theta_{t}^{\PS}} + \norm{\theta_{t}^{\PS} - \theta_{\gD_{t}}^{\star}} }^2 \\
        &\le \sigma^2 + C_V^2 \bigopen{ 1 + \frac{(1 - \alpha_{t}) \epsilon \beta_{z}}{\mu} }^2 \norm{\theta_{t} - \theta_{t}^{\PS}}^2
        \end{aligned}
        \label{eq:sigmabound}
    \end{align}
    where the last inequality follows from \cref{lem:pislipschitz}.
    We can observe that
    \begin{align*}
        &\E_{Z_{t} \sim \gD_{t}} \norm{ \theta_{t} - \theta_{t}^{\PS} - \eta_{t} \nabla_{\theta} \ell (Z_{t}, \theta_{t}) }^2 \le \norm{\theta_{t} - \theta_{t}^{\PS}}^2 - 2 \eta_{t} \nabla \mathrm{R}_{\gD_{t}} (\theta_{t})^{\top} ( \theta_{t} - \theta_{t}^{\PS} ) + \eta_{t}^{2} \E_{Z_{t} \sim \gD_{t}} [\bignorm{\nabla_{\theta} \ell (Z_{t}, \theta_{t})}^2].
    \end{align*}
    Since we have $\nabla \mathrm{R}_{\gD_t^\PS}(\theta_t^\PS)^\top (\theta_t - \theta_t^\PS) \ge 0$ by the first-order optimality condition of $\theta_t^\PS$, we have
    \begin{align*}
        \nabla \mathrm{R}_{\gD_{t}} (\theta_{t})^{\top} ( \theta_{t} - \theta_{t}^{\PS} ) 
        &= \bigopen{\nabla \mathrm{R}_{\gD_{t}} (\theta_{t}) - \nabla R_{\gD_t^\PS}(\theta_t^\PS)}^{\top} ( \theta_{t} - \theta_{t}^{\PS} ) + \nabla R_{\gD_t^\PS}(\theta_t^\PS)^\top (\theta_t - \theta_t^\PS) \\
        &\ge \bigopen{\nabla \mathrm{R}_{\gD_{t}} (\theta_{t}) - \nabla R_{\gD_t^\PS}(\theta_t^\PS)}^{\top} ( \theta_{t} - \theta_{t}^{\PS} ) \\
        &\ge \bigopen{\mu - (1 - \alpha_{t}) \epsilon \beta_{z}} \norm{\theta_{t} - \theta_{t}^{\PS}}^2,
    \end{align*}
    by \eqref{eq:rgd_basic5} from the proof of \cref{lem:rgd_onestep}.
    We can also use \eqref{eq:sigmabound} to obtain
    \begin{align*}
        &\norm{\theta_{t} - \theta_{t}^{\PS}}^2 - 2 \eta_{t} \nabla \mathrm{R}_{\gD_{t}} (\theta_{t})^{\top} ( \theta_{t} - \theta_{t}^{\PS} ) + \eta_{t}^{2} \E_{Z_{t} \sim \gD_{t}} [\bignorm{\nabla_{\theta} \ell (Z_{t}, \theta_{t})}^2] \\
        &\le \norm{\theta_{t} - \theta_{t}^{\PS}}^2 - 2 \eta_{t} \bigopen{\mu - (1 - \alpha_{t}) \epsilon \beta_{z}} \norm{\theta_{t} - \theta_{t}^{\PS}}^2 \\
        &\phantom{{}\le{}} + \eta_{t}^{2} \bigopen{\sigma^2 + C_V^2 \bigopen{ 1 + \frac{(1 - \alpha_{t}) \epsilon \beta_{z}}{\mu} }^2 \norm{\theta_{t} - \theta_{t}^{\PS}}^2}.
    \end{align*}
    Therefore, choosing
    \begin{align*}
        \gamma_{t}' &\coloneqq 1 - 2 \eta_{t} \bigopen{\mu - (1 - \alpha_{t}) \epsilon \beta_{z}} + \eta_{t}^{2} C_V^2 \bigopen{ 1 + \frac{(1 - \alpha_{t}) \epsilon \beta_{z}}{\mu} }^2
    \end{align*}
    to simplify as
    \begin{align*}
        \E_{Z_{t} \sim \gD_{t}} [\norm{ \theta_{t} - \theta_{t}^{\PS} - \eta_{t} \nabla_{\theta} \ell (Z_{t}, \theta_{t}) }^2] 
        &\le \gamma_{t}' \norm{ \theta_{t} - \theta_{t}^{\PS} }^2 + \eta_{t}^2 \sigma^2,
    \end{align*}
    conditioned on $\theta_{t}$, or
    \begin{align}
        \E [\norm{ \theta_{t} - \theta_{t}^{\PS} - \eta_{t} \nabla_{\theta} \ell (Z_{t}, \theta_{t}) }^2] 
        &\le \gamma_{t}' \E [\norm{ \theta_{t} - \theta_{t}^{\PS} }^2] + \eta_{t}^2 \sigma^2,
        \label{eq:sgdintermed}
    \end{align}
    where the expectations are now taken on the entire history.
    Therefore, starting from \eqref{eq:sgdproj}, we have
    \begin{align*}
        \norm{\theta_{t+1}-\theta_{t+1}^\PS}^2
        &\le \left( \norm{\theta_t-\theta_t^\PS-\eta_t\nabla\ell(Z_t,\theta_t)}+\Delta_t^\PS \right)^2,
    \end{align*}
    and therefore
    \begin{align*}
        \E [\norm{\theta_{t+1} - \theta_{t+1}^{\PS}}^2] 
        &\le \E [\norm{ \theta_{t} - \theta_{t}^{\PS} - \eta_{t} \nabla_{\theta} \ell (Z_{t}, \theta_{t}) }^2]  + 2 \Delta_{t}^{\PS} \cdot \E [\norm{ \theta_{t} - \theta_{t}^{\PS} - \eta_{t} \nabla_{\theta} \ell (Z_{t}, \theta_{t}) }] + (\Delta_{t}^{\PS})^2 \\
        &\le \E [\norm{ \theta_{t} - \theta_{t}^{\PS} - \eta_{t} \nabla_{\theta} \ell (Z_{t}, \theta_{t}) }^2]  + 2 \Delta_{t}^{\PS} \cdot \sqrt{\E [\norm{ \theta_{t} - \theta_{t}^{\PS} - \eta_{t} \nabla_{\theta} \ell (Z_{t}, \theta_{t}) }^2]} + (\Delta_{t}^{\PS})^2.
    \end{align*}
    Finally, we can use \eqref{eq:sgdintermed}, $\sqrt{a+b}\le\sqrt{a}+\sqrt{b}$, and AM-GM to obtain
    \begin{align*}
        &\E [\norm{ \theta_{t} - \theta_{t}^{\PS} - \eta_{t} \nabla_{\theta} \ell (Z_{t}, \theta_{t}) }^2] + 2 \Delta_{t}^{\PS} \cdot \sqrt{\E [\norm{ \theta_{t} - \theta_{t}^{\PS} - \eta_{t} \nabla_{\theta} \ell (Z_{t}, \theta_{t}) }^2]} + (\Delta_{t}^{\PS})^2 \\
        &\le \gamma_{t}' \E [\norm{ \theta_{t} - \theta_{t}^{\PS} }^2] + \eta_{t}^2 \sigma^2 + 2 \Delta_{t}^{\PS} \sqrt{ \gamma_{t}' \E [\norm{ \theta_{t} - \theta_{t}^{\PS} }^2] + \eta_{t}^2 \sigma^2 } + (\Delta_{t}^{\PS})^2 \\
        &\le \gamma_{t}' \E [\norm{ \theta_{t} - \theta_{t}^{\PS} }^2] + \eta_{t}^2 \sigma^2 + 2 \Delta_{t}^{\PS} \sqrt{ \gamma_{t}' \E [\norm{ \theta_{t} - \theta_{t}^{\PS} }^2] } + 2 \eta_{t} \sigma \Delta_{t}^{\PS} + (\Delta_{t}^{\PS})^2 \\
        &\le \gamma_{t}' \E [\norm{ \theta_{t} - \theta_{t}^{\PS} }^2] + 2 \Delta_{t}^{\PS} \sqrt{ \gamma_{t}' \E [\norm{ \theta_{t} - \theta_{t}^{\PS} }^2] } + 2 \eta_{t}^2 \sigma^2 + 2 (\Delta_{t}^{\PS})^2,
    \end{align*}
    which proves the given statement.
\end{proof}


\subsection{\texorpdfstring%
{Upper bound on $\E \| \theta_{T} - \theta_{T}^{\PS} \|^2$ for greedy SGD}%
{Upper bounds of squared iterate norms for greedy SGD}}
\label{subsec:rsgd_wreg}

Here we prove \cref{lem:rsgd_wreg}, which provides upper bounds of $\E \| \theta_{T} - \theta_{T}^{\PS} \|^2$ for a fixed $T$.

\begin{restatable}{lemma}{lemrsgdwreg}
    \label{lem:rsgd_wreg}
    Suppose that $\gD \colon \Theta \rightarrow \gP (\gZ)$ is $\epsilon$-sensitive and $\ell(z, \theta)$ is $\mu$-strongly convex \eqref{eq:strong_convexity} and $\beta_{\theta}$-smooth \eqref{eq:smoothness_theta} in $\theta$ and $\beta_z$-smooth \eqref{eq:smoothness_z} in $z$
    Also, suppose that $\epsilon < \frac{\mu}{(1 - \alpha_{\min}) \beta_{z}}$ and we have quadratic-bounded variance (Ass.~\ref{ass:bv}).
    Then, $\gA_{\mathrm{RSGD\text{-}G}}$ with step sizes
    \begin{align}
        \eta_{t} &= \frac{4}{A(t + t_0)}, \quad t_0 = \max \left\{ 1, \left\lceil \frac{8B}{A^2} \right\rceil, \left\lceil \frac{8\mu}{A} \right\rceil \right\}
        \label{eq:sgdstepsize}
    \end{align}
    for some constants $A, B > 0$ defined as 
    \begin{align*}
        A \coloneqq 2 \bigopen{\mu - (1 - \alpha_{\min}) \epsilon \beta_{z}},
        \qquad
        B \coloneqq C_V^2 \bigopen{ 1 + \frac{(1 - \alpha_{\min}) \epsilon \beta_{z}}{\mu} }^2
    \end{align*}
    satisfies the following inequality for $T \ge 0$,
    \begin{align*}
        &\E \| \theta_{T+1} - \theta_{T+1}^{\PS} \|^2 \le \left( \frac{t_0 + 1}{T + t_0 + 1} \right)^{3/2} \| \theta_{1} - \theta_{1}^{\PS} \|^2 + \frac{64 \sigma^2}{A^2 (T+t_0)}  + \frac{2}{(T + t_0 + 1)^{3/2}} \sum_{t=1}^{T} (t + t_0 + 1)^{5/2} (\Delta_{t}^{\PS})^2.
    \end{align*}
\end{restatable}

\begin{proof}
    First, note that the case $T=0$ is trivial, and that we have $A > 0$ since $\epsilon < \frac{\mu}{(1 - \alpha_{\min}) \beta_{z}}$.
    
    Let us define $e_{t} = \theta_{t} - \theta_{t}^{\PS}$ for simplicity.
    We also simplify $\gamma_{t}'$ using
    \begin{align}
        \gamma_{t}' &= 1 - 2 \eta_{t} \bigopen{\mu - (1 - \alpha_{t}) \epsilon \beta_{z}} + \eta_{t}^{2} C_V^2 \bigopen{ 1 + \frac{(1 - \alpha_{t}) \epsilon \beta_{z}}{\mu} }^2 \notag \\
        &\le 1 - \eta_{t} \cdot \underbrace{ 2 \bigopen{\mu - (1 - \alpha_{\min}) \epsilon \beta_{z}} }_{\triangleq A} + \eta_{t}^{2} \cdot \underbrace{ C_V^2 \bigopen{ 1 + \frac{(1 - \alpha_{\min}) \epsilon \beta_{z}}{\mu} }^2 } _{\triangleq B} \eqqcolon \overline{\gamma}_{t}',
        \label{eq:aandb}
    \end{align}
    where $\alpha_{\min} = \min_{t \in [T]} \alpha_{t}$.
    Recall that we use step sizes
    \begin{align*}
        \eta_{t} &= \frac{4}{A(t + t_0)}, \quad t_0 = \max \left\{ 1, \left\lceil \frac{8B}{A^2} \right\rceil, \left\lceil \frac{8\mu}{A} \right\rceil \right\},
    \end{align*}
    for time-independent constants $A, B$ defined in \eqref{eq:aandb}.
    Since $1+t_0 > \frac{8\mu}{A}$, we have $\frac{4}{A(1+t_0)} < \frac{1}{2\mu}$ and therefore $\eta_{t} = \frac{4}{A(t+t_0)} < \frac{1}{2\mu}$ for all $t \in [T]$, which ensures that $\overline{\gamma}_{t}' \ge \gamma_{t}' \ge 0$ for all $t$.

    By \eqref{eq:sgdonestep} of \cref{lem:sgdonestep}, we have
    \begin{align*}
        \E \norm{ e_{t+1} }^2
        &\le \overline{\gamma}_{t}' \E \norm{ e_{t} }^2 + 2 \Delta_{t}^{\PS} \sqrt{ \overline{\gamma}_{t}' \E \norm{ e_{t} }^2 } + 2 \eta_{t}^2 \sigma^2 + 2 (\Delta_{t}^{\PS})^2,
    \end{align*}
    where $\overline{\gamma}_{t}' = 1 - \eta_{t} A + \eta_{t}^2 B$.
    Note that
    \begin{align*}
        \overline{\gamma}_{t}' &= 1 - \eta_{t} A + \eta_{t}^2 B \\
        &= 1 - \frac{4}{t + t_0} + \frac{16 B}{A^2 (t + t_0)^2} \\
        &\le 1 - \frac{4}{t + t_0} + \frac{8 B}{A^2} \cdot \frac{2}{t_0 (t + t_0)} \\
        &\le 1 - \frac{4}{t + t_0} + \frac{2}{t + t_0} \\
        &= 1 - \frac{2}{t + t_0}.
    \end{align*}
    Then we have
    \begin{align*}
        \E \norm{ e_{t+1} }^2
        &\le \overline{\gamma}_{t}' \E \norm{ e_{t} }^2 + 2 \Delta_{t}^{\PS} \sqrt{ \overline{\gamma}_{t}' \E \norm{ e_{t} }^2 } + 2 \eta_{t}^2 \sigma^2 + 2 (\Delta_{t}^{\PS})^2 \\
        &\le \overline{\gamma}_{t}' \E \norm{ e_{t} }^2 + 2 \Delta_{t}^{\PS} \sqrt{ \E \norm{ e_{t} }^2 } + 2 \eta_{t}^2 \sigma^2 + 2 (\Delta_{t}^{\PS})^2 \\
        &\le \left( 1 - \frac{2}{t + t_0} \right) \E \norm{ e_{t} }^2 + \frac{32 \sigma^2}{A^2 (t + t_0)^2} + 2 (\Delta_{t}^{\PS})^2 \\
        &\phantom{{}\le{}} + \left( \frac{1}{2 (t + t_0)} \E \norm{ e_{t} }^2 + 2(t + t_0) (\Delta_{t}^{\PS})^2 \right) \\
        &= \left( 1 - \frac{3/2}{t + t_0} \right) \E \norm{ e_{t} }^2 + \frac{32 \sigma^2}{A^2 (t + t_0)^2} + 2 (t + t_0 + 1) (\Delta_{t}^{\PS})^2.
    \end{align*}
    For simplicity, let us define
    \begin{align*}
        F_{t} &\coloneqq \E \norm{ e_{t} }^2, \quad \quad \quad
        a_{t} \coloneqq \left( 1 - \frac{3/2}{t + t_0} \right), \\
        b_{t} &\coloneqq \frac{32 \sigma^2}{A^2 (t + t_0)^2}, \quad
        d_{t} \coloneqq 2 (t + t_0 + 1) (\Delta_{t}^{\PS})^2,
    \end{align*}
    and write
    \begin{align*}
        F_{t+1} &\le a_{t} F_{t} + b_{t} + d_{t}.
    \end{align*}
    Since $(1 - \frac{3/2}{s+t_{0}}) \le \exp (- \frac{3/2}{s+t_{0}})$, we have
    \begin{align*}
        \prod_{s=t+1}^{T} a_{s} &\le \prod_{s=t+1}^{T} \exp \left( - \frac{3/2}{s+t_{0}} \right) \\
        &= \exp \left( - \sum_{s=t+1}^{T} \frac{3/2}{s+t_{0}} \right) \\
        &\le \exp \left( - \frac{3}{2} \cdot \log \frac{T+t_{0}+1}{t+t_{0}+1} \right) \\
        &= \left( \frac{t+t_{0}+1}{T+t_{0}+1} \right)^{3/2}.
    \end{align*}
    Also, we can observe that
    \begin{align*}
        \frac{64 \sigma^2}{A^2(t + t_0)} - \left( 1 - \frac{3/2}{t + t_0} \right) \frac{64 \sigma^2}{A^2(t + t_0 - 1)} &= \frac{32 \sigma^2}{A^2(t + t_0)(t + t_0 - 1)} \\
        &\ge \frac{32 \sigma^2}{A^2(t + t_0)^2},
    \end{align*}
    and thus $N_{t+1} \ge a_{t} N_{t} + b_{t}$ for $N_{t} = \frac{64 \sigma^2}{A^2(t + t_0 - 1)}$. This implies that
    \begin{align*}
        F_{t+1} - N_{t+1} &\le a_{t} (F_{t} - N_{t}) + d_{t},
    \end{align*}
    and therefore
    \begin{align*}
        F_{T+1} &\le N_{T+1} + \left( \prod_{t=1}^{T} a_{t} \right) (F_{1} - N_{1}) + \sum_{t=1}^{T} d_{t} \left( \prod_{s=t+1}^{T} a_{s} \right) \\
        &\le \frac{64 \sigma^2}{A^2(T + t_0)} + \left( \frac{t_{0}+1}{T+t_{0}+1} \right)^{3/2} \cdot F_{1} \\
        &\phantom{{}\le{}} + 2 \sum_{t=1}^{T} (t + t_0 + 1) (\Delta_{t}^{\PS})^2 \cdot \left( \frac{t+t_{0}+1}{T+t_{0}+1} \right)^{3/2},
    \end{align*}
    which is equivalent to the given statement.
\end{proof}


\subsection{\texorpdfstring%
{Proof of \cref{thm:rsgdg}}%
{Proof of Proposition 4.2}}
\label{subsec:rsgdg}

Here we prove \cref{thm:rsgdg}, restated below for the sake of readability.

\thmrsgdg*

\begin{proof}
    Using the same choice of learning rates $\eta_t$ and the same auxiliary definitions (e.g., $e_t, F_t$) as in \cref{lem:rsgd_wreg}, we start from the result of \cref{lem:rsgd_wreg} with $T \leftarrow t-1$,
    \begin{align*}
        \E \norm{e_{t}}^2
        &\le \frac{64 \sigma^2}{A^2(t + t_0 - 1)} + \left( \frac{t_{0}+1}{t+t_{0}} \right)^{3/2} \!\!\!\!\cdot \| \theta_1 - \theta_1^{\PS} \|^2 \\
        &\phantom{{}\le{}} + \frac{2}{(t+t_{0})^{3/2}} \sum_{s=1}^{t-1} (s + t_0 + 1)^{5/2} (\Delta_{s}^{\PS})^2,
    \end{align*}
    which implies
    \begin{align*}
        \sqrt{\E \norm{e_{t}}^2}
        &\le \frac{8 \sigma}{A(t + t_0 - 1)^{1/2}} + \left( \frac{t_{0}+1}{t+t_{0}} \right)^{3/4} \cdot \| \theta_{1} - \theta_{1}^{\PS} \| \\
        &\phantom{{}\le{}} + \frac{2}{(t+t_{0})^{3/4}} \left( \sum_{s=1}^{t-1} (s + t_0 + 1)^{5/2} (\Delta_{s}^{\PS})^2 \right)^{\!\!1/2},
    \end{align*}
    since $\sqrt{a + b + c} \le \sqrt{a} + \sqrt{b} + \sqrt{c}$ for $a, b, c \ge 0$.
    We can observe that
    \begin{align*}
        \sum_{t=1}^{T} \frac{8 \sigma}{A(t + t_0 - 1)^{1/2}}
        &\le \frac{16 \sigma}{A} (T + t_0)^{1/2}, \\
        \sum_{t=1}^{T} \left( \frac{t_{0}+1}{t+t_{0}} \right)^{3/4} \cdot \| \theta_{1} - \theta_{1}^{\PS} \|
        &\le 4 (t_{0}+1)^{3/4} \| \theta_{1} - \theta_{1}^{\PS} \| \cdot (T + t_0)^{1/4},
    \end{align*}
    and we can substitute the last term with an upper bound
    \begin{align*}
        \sum_{s=1}^{t-1} (s + t_0 + 1)^{5/2} (\Delta_{s}^{\PS})^2
        &\le \sum_{s=1}^{T-1} (s + t_0 + 1)^{5/2} (\Delta_{s}^{\PS})^2,
    \end{align*}
    so that
    \begin{align*}
        &\sum_{t=1}^{T} \frac{2}{(t+t_{0})^{3/4}} \left( \sum_{s=1}^{t-1} (s + t_0 + 1)^{5/2} (\Delta_{s}^{\PS})^2 \right)^{\!\!1/2} \\
        &\le 8 (T + t_0)^{1/4} \left( \sum_{s=1}^{T-1} (s + t_0 + 1)^{5/2} (\Delta_{s}^{\PS})^2 \right)^{\!\!1/2}
    \end{align*}
    
    Finally, since $\PR_{t}$ is $L$-Lipschitz, we have
    \begin{align*}
        \mathrm{Reg}_{T}^{\PS} (\ARSGDG) 
        &\le L \sum_{t=1}^{T} \mathbb{E} \norm{ e_{t} } \\
        &\le L \sum_{t=1}^{T} \sqrt{ \E \| e_{t} \|^2 } \\
        &\le 8 L (T + t_0)^{1/4} \left( (t_{0}+1)^{3/4} \| \theta_{1} - \theta_{1}^{\PS} \| + \left( \sum_{s=1}^{T-1} (s + t_0 + 1)^{5/2} (\Delta_{s}^{\PS})^2 \right)^{\!\!1/2} \, \right)  + \frac{16 \sigma L}{A} (T + t_0)^{1/2},
    \end{align*}
    and therefore
    \begin{align*}
        \mathrm{Reg}_{T}^{\PS} (\ARSGDG)
        &\le \gO \left( T^{1/2} + T^{1/4} \left( \sum_{t=1}^{T-1} (t + t_0 + 1)^{5/2} (\Delta_{t}^{\PS})^2 \right)^{\!\!1/2} \, \right),
    \end{align*}
    as desired.
\end{proof}

\section{Proofs for lazy SGD}


\subsection{One-step inequality for lazy SGD}
\label{subsec:rsgdl_inner}

Here we prove \cref{lem:rsgdl_inner}, which provides a one-step inequality for the iterations of lazy SGD within step $t$, based on the standard SGD analysis under a fixed distribution $\gD_t$.

\begin{restatable}{lemma}{rsgdlinner}
    \label{lem:rsgdl_inner}
    Suppose that $\ell(z, \theta)$ is $\mu$-strongly convex in $\theta$ and the stochastic gradient oracle has quadratic-bounded variance \eqref{ass:bv}.
    Then, one iteration of $\ARSGDL$ satisfies
    \begin{align}
        \E \| \varphi_{t, j+1} - \theta_{\gD_t} \|^2
        &\le \left( 1 - 2 \eta_{t, j} \mu + \eta_{t, j}^{2} C_V^2 \right) \E \norm{\varphi_{t, j} - \theta_{\gD_t}}^2 + \eta_{t, j}^{2} \sigma^2,
        \label{eq:rsgdl_inner_onestep}
    \end{align}
    where $\theta_{\gD_t} \coloneqq \argmin_{\theta} \mathrm{R}_{\gD_t} (\theta)$.
\end{restatable}

\begin{proof}
    Similarly as in the previous proofs, we have
    \begin{align*}
        \E \| \varphi_{t, j+1} - \theta_{\gD_t} \|^2 
        &\le \E \norm{ \varphi_{t, j} - \theta_{\gD_t} - \eta_{t, j} \nabla_{\theta} \ell (Z_{t, j}, \varphi_{t, j}) }^2 \\
        &\le \norm{\varphi_{t, j} - \theta_{\gD_t}}^2 - 2 \eta_{t, j} \nabla \mathrm{R}_{\gD_t} (\varphi_{t, j})^{\top} ( \varphi_{t, j} - \theta_{\gD_t} ) + \eta_{t, j}^{2} \E [\bignorm{\nabla_{\theta} \ell (Z_{t, j}, \varphi_{t, j})}^2]
    \end{align*}
    where the expectation here is over $Z_{t, j} \sim \gD_t$ conditioned on $\varphi_{t, j}$.
    
    Since $\theta_{\gD_t} = \argmin_{\theta} \mathrm{R}_{\gD_t} (\theta)$, the first-order optimality condition of $\theta_{\gD_t}$ and $\mu$-strong monotonicity of $\nabla \mathrm{R}_{\gD_t} (\theta)$ yields
    \begin{align*}
        \nabla \mathrm{R}_{\gD_t} (\varphi_{t, j})^{\top} ( \varphi_{t, j} - \theta_{\gD_t} )
        &\ge \left( \nabla \mathrm{R}_{\gD_t} (\varphi_{t, j}) - \nabla \mathrm{R}_{\gD_t} (\theta_{\gD_t}) \right)^{\top} ( \varphi_{t, j} - \theta_{\gD_t} ) \\
        &\ge \mu \| \varphi_{t, j} - \theta_{\gD_t} \|^2,
    \end{align*}
    and by \cref{ass:bv}, we have
    \begin{align*}
        \E [\bignorm{\nabla_{\theta} \ell (Z_{t, j}, \varphi_{t, j})}^2] &\le \sigma^2 + C_V^2 \| \varphi_{t, j} - \theta_{\gD_t} \|^2.
    \end{align*}
    Therefore we have
    \begin{align*}
        \E \| \varphi_{t, j+1} - \theta_{\gD_t} \|^2
        &\le \left( 1 - 2 \eta_{t, j} \mu + \eta_{t, j}^{2} C_V^2 \right) \norm{\varphi_{t, j} - \theta_{\gD_t}}^2 + \eta_{t, j}^{2} \sigma^2,
    \end{align*}
    conditioned on $\varphi_{t, j}$, and thus
    \eqref{eq:rsgdl_inner_onestep}
    by taking expectations over the entire history.
\end{proof}


\subsection{\texorpdfstring%
{Upper bounds of $\E \| \theta_{T} - \theta_{T}^{\PS} \|^2$ for lazy SGD}%
{Upper bounds of squared iterate norms for lazy SGD}}
\label{subsec:lem_rsgd_lazy}

Here we prove \cref{lem:rsgd_lazy}, which provides upper bounds of $\E \| \theta_{T} - \theta_{T}^{\PS} \|^2$ for a fixed $T$.

\begin{restatable}{lemma}{lemrsgdlazy}
    \label{lem:rsgd_lazy}
    Suppose that $\gD \colon \Theta \rightarrow \gP (\gZ)$ is $\epsilon$-sensitive, $z \mapsto \ell(z, \theta)$ is $\beta_{z}$-smooth for all $\theta \in \Theta$, and $\mathrm{R}_{D} (\theta)$ is $\mu$-strongly convex and $\beta_{\theta}$-smooth for all $D \in \gC$.
    Also, suppose that $\epsilon < \frac{\mu}{(1 - \alpha_{\min}) \beta_{z}}$ and we have quadratic-bounded variance (Ass.~\ref{ass:bv}).
    Then $\ARSGDL$ with step sizes
    \begin{align}
        \eta_{t, j} = \frac{4}{A(j+t_0)}, \qquad t_0 = \max \left\{ 1, \left\lceil \frac{8B}{A^2} \right\rceil, \left\lceil \frac{8\mu}{A} \right\rceil \right\},
        \label{eq:rsgdl_stepsize}
    \end{align}
    for constants $A = 2 \mu$ and $B = C_V^2$, and $n(t) \ge n_0 t^{r}$ samples per deployment with large enough $n_0$ satisfies
    \begin{align*}
        \E \| \theta_{T+1} - \theta_{T+1}^{\PS} \|^2
        &\le \tilde{\gamma}^{T} \| \theta_{1} - \theta_{1}^{\PS} \|^2 + J \!\cdot\! \left( \tilde{\gamma}^{T (1 - 2^{-\frac{1}{r}})} + 2T^{-r} \right) \\
        &\phantom{{}\le{}} + (1 + \tau) \sum_{t=1}^{T} \tilde{\gamma}^{T-t} (\Delta_{t}^{\PS})^2,
    \end{align*}
    where $\tau, J > 0$ and $0 < \tilde{\gamma} < 1$ are chosen constants independent of $T$.
\end{restatable}

\begin{proof}
First, note that the case $T=0$ is trivial.

Suppose that $T > 0$.
We start with \cref{lem:rsgdl_inner} (proved in \cref{subsec:rsgdl_inner}), which shows
\begin{align*}
    \E \| \varphi_{t, j+1} - \theta_{\gD_t} \|^2
    &\le \left( 1 - 2 \eta_{t, j} \mu + \eta_{t, j}^{2} C_V^2 \right) \E \norm{\varphi_{t, j} - \theta_{\gD_t}}^2 + \eta_{t, j}^{2} \sigma^2.
\end{align*}
Let $F_{t,j} = \E \norm{\varphi_{t, j} - \theta_{\gD_t}}^2$.
Then we can rewrite as
\begin{align*}
    F_{t,j+1}
    &\le ( 1 - \eta_{t, j} A + \eta_{t, j}^{2} B ) F_{t,j} + \eta_{t, j}^{2} \sigma^2,
\end{align*}
with $A = 2 \mu$ and $B = C_V^2$.
Similarly to the proof of \cref{lem:rsgd_wreg} in \cref{subsec:rsgd_wreg}, the choice of
\begin{align*}
    \eta_{t, j} = \frac{4}{A(j+t_0)}, \qquad t_0 = \max \left\{ 1, \left\lceil \frac{8B}{A^2} \right\rceil, \left\lceil \frac{8\mu}{A} \right\rceil \right\}
\end{align*}
yields
\begin{align*}
    F_{t,j+1}
    &\le \left( 1 - \frac{2}{j+t_0} \right) F_{t,j} + \frac{4 \sigma^2}{\mu^2 (j+t_0)^2}.
\end{align*}

By \cref{lem:sublinearrecursion} with $m = 1$ and $\alpha = 2$, we have
\begin{align*}
    F_{t, j+1} &\le S (j + t_0)^{-1},
\end{align*}
for $t_0 = \max \{ 1, \lceil \frac{8B}{A^2} \rceil, \lceil \frac{8\mu}{A} \rceil \} = \max \{ 1, \lceil \frac{2C_V^2}{\mu^2} \rceil, \lceil \frac{8\mu}{A} \rceil \}$ and
\begin{align*}
    S &= \max \left\{ F_1 (1 + t_0), \frac{H}{\alpha - m} \right\} \\
    &= \max \left\{ (1 + t_0) \E \norm{\theta_{t} - \theta_{\gD_t}}^2, \frac{4 \sigma^2}{\mu^2} \right\} \\
    &\le (1 + t_0) \E \norm{\theta_{t} - \theta_{\gD_t}}^2 + \frac{4 \sigma^2}{\mu^2}.
\end{align*}

In particular, for $j = n(t)$ we have
\begin{align}
    \E \norm{\theta_{t+1} - \theta_{\gD_t}}^2
    &\le \frac{1}{n(t) + t_0} \cdot \left( (1 + t_0) \E \norm{\theta_{t} - \theta_{\gD_t}}^2 + \frac{4 \sigma^2}{\mu^2} \right).
    \label{eq:rsgdl_1}
\end{align}
Observe that
\begin{align*}
    \| \theta_{\gD_t} - \theta_{t}^{\PS} \|
    &\le \frac{(1 - \alpha_{t}) \epsilon \beta_{z}}{\mu} \| \theta_{t} - \theta_{t}^{\PS} \|,
\end{align*}
and also,
\begin{align*}
    \| \theta_{t} - \theta_{\gD_t} \|
    &\le \| \theta_{t} - \theta_{t}^{\PS} \| + \| \theta_{\gD_t} - \theta_{t}^{\PS} \|
    \le \left( 1 + \frac{(1 - \alpha_{t}) \epsilon \beta_{z}}{\mu} \right) \| \theta_{t} - \theta_{t}^{\PS} \|.
\end{align*}
Therefore, from each inequality we can deduce that
\begin{align}
    \E \| \theta_{\gD_t} - \theta_{t}^{\PS} \|^2
    &\le \left( \frac{(1 - \alpha_{t}) \epsilon \beta_{z}}{\mu} \right)^2 \E \| \theta_{t} - \theta_{t}^{\PS} \|^2
    \label{eq:rsgdl_2}
\end{align}
and
\begin{align}
    \E \| \theta_{t+1} - \theta_{\gD_t} \|^2 
    &\le \frac{1}{n(t) + t_0} \cdot \left( (1 + t_0) \E \norm{\theta_{t} - \theta_{\gD_t}}^2 + \frac{4 \sigma^2}{\mu^2} \right) \notag \\
    &\le \frac{1}{n(t) + t_0} \cdot \left( (1 + t_0) \left( 1 + \frac{(1 - \alpha_{t}) \epsilon \beta_{z}}{\mu} \right)^2 \E \| \theta_{t} - \theta_{t}^{\PS} \|^2 + \frac{4 \sigma^2}{\mu^2} \right) \notag \\
    &\le \frac{1}{n(t) + t_0} \cdot \left( 4 (1 + t_0) \E \| \theta_{t} - \theta_{t}^{\PS} \|^2 + \frac{4 \sigma^2}{\mu^2} \right).
    \label{eq:rsgdl_3}
\end{align}
We can compute
\begin{align*}
    \E \| \theta_{t+1} - \theta_{t}^{\PS} \|^2
    &= \E \| \theta_{t+1} - \theta_{\gD_t} \|^2 + 2 \E \left[ (\theta_{t+1} - \theta_{\gD_t})^{\top} (\theta_{\gD_t} - \theta_{t}^{\PS}) \right] + \E \| \theta_{\gD_t} - \theta_{t}^{\PS} \|^2 \\
    &\le \E \| \theta_{t+1} - \theta_{\gD_t} \|^2 + 2 \E \left[ \| \theta_{t+1} - \theta_{\gD_t} \| \cdot \| \theta_{\gD_t} - \theta_{t}^{\PS} \| \right] + \E \| \theta_{\gD_t} - \theta_{t}^{\PS} \|^2 \\
    &\le \E \| \theta_{t+1} - \theta_{\gD_t} \|^2 + 2 \sqrt{\E \| \theta_{t+1} - \theta_{\gD_t} \|^2 \cdot \E \| \theta_{\gD_t} - \theta_{t}^{\PS} \|^2 } + \E \| \theta_{\gD_t} - \theta_{t}^{\PS} \|^2 \\
    &\le (1 + \zeta) \E \| \theta_{t+1} - \theta_{\gD_t} \|^2 + \left( 1 + \frac{1}{\zeta} \right) \E \| \theta_{\gD_t} - \theta_{t}^{\PS} \|^2,
\end{align*}
for any $\zeta > 0$, which we choose later.

By \cref{eq:rsgdl_1,eq:rsgdl_2,eq:rsgdl_3}, we have
\begin{align*}
    \E \| \theta_{t+1} - \theta_{t}^{\PS} \|^2
    &\le \left( \frac{4 (1 + t_0) (1 + \zeta)}{n(t) + t_0} + \left( 1 + \frac{1}{\zeta} \right) \left( \frac{(1 - \alpha_{t}) \epsilon \beta_{z}}{\mu} \right)^2 \right) \E \| \theta_{t} - \theta_{t}^{\PS} \|^2
    + \frac{1 + \zeta}{n(t) + t_0} \cdot \frac{4 \sigma^2}{\mu^2} \\
    &\le \left( \frac{4 (1 + t_0) (1 + \zeta)}{n_0} + \left( 1 + \frac{1}{\zeta} \right) \left(  \frac{(1 - \alpha_{\min}) \epsilon \beta_{z}}{\mu} \right)^2 \right) \E \| \theta_{t} - \theta_{t}^{\PS} \|^2 
    + \frac{1 + \zeta}{n(t) + t_0} \cdot \frac{4 \sigma^2}{\mu^2},
\end{align*}
where the last inequality is by $\epsilon < \frac{\mu}{(1 - \alpha_{\min}) \beta_{z}}$.
Substituting
\begin{align*}
    c = \frac{4 (1 + t_0) (1 + \zeta)}{n_0} + \left( 1 + \frac{1}{\zeta} \right) \left( \frac{(1 - \alpha_{\min}) \epsilon \beta_{z}}{\mu} \right)^2,
\end{align*}
we can write
\begin{align*}
    \E \| \theta_{t+1} - \theta_{t}^{\PS} \|^2
    &\le c \E \| \theta_{t} - \theta_{t}^{\PS} \|^2 + \frac{1 + \zeta}{n(t) + t_0} \cdot \frac{4 \sigma^2}{\mu^2}.
\end{align*}

Recall that $\Delta_{t}^{\PS} = \| \theta_{t}^{\PS} - \theta_{t+1}^{\PS} \|$.
By the AM-GM inequality, we have
\begin{align*}
    \E \| \theta_{t+1} - \theta_{t+1}^{\PS} \|^2
    &\le \E \| \theta_{t+1} - \theta_{t}^{\PS} \|^2 + 2 \Delta_{t}^{\PS} \sqrt{\E \| \theta_{t+1} - \theta_{t}^{\PS} \|^2} + (\Delta_{t}^{\PS})^2 \\
    &\le \left( 1 + \frac{1}{\tau} \right) \E \| \theta_{t+1} - \theta_{t}^{\PS} \|^2 + (1 + \tau) (\Delta_{t}^{\PS})^2 \\
    &\le c \left( 1 + \frac{1}{\tau} \right) \E \| \theta_{t} - \theta_{t}^{\PS} \|^2 + \frac{(1 + \frac{1}{\tau}) (1 + \zeta)}{n(t) + t_0} \cdot \frac{4 \sigma^2}{\mu^2} \\
    &\phantom{{}\le{}} + (1 + \tau) (\Delta_{t}^{\PS})^2
\end{align*}
for any $\tau > 0$, which we choose later.

Let $\tilde{\gamma} = c (1 + \frac{1}{\tau})$.
Now we choose constants $\zeta, \tau > 0$ that satisfy
\begin{align*}
    \left( 1 + \frac{1}{\tau} \right) \left( 1 + \frac{1}{\zeta} \right) \left( \frac{(1 - \alpha_{\min}) \epsilon \beta_{z}}{\mu} \right)^2 < 1,
\end{align*}
and $n_0$ as a large enough constant so that
\begin{align*}
    \tilde{\gamma} = \left( 1 + \frac{1}{\tau} \right) \left( \frac{4 (1 + t_0) (1 + \zeta)}{n_0} + \left( 1 + \frac{1}{\zeta} \right) \left( \frac{(1 - \alpha_{\min}) \epsilon \beta_{z}}{\mu} \right)^2 \right) < 1,
\end{align*}
which is possible as we assume $\epsilon < \frac{\mu}{(1 - \alpha_{\min}) \beta_{z}}$.
Then, we have
\begin{align*}
    \E \| \theta_{t+1} - \theta_{t+1}^{\PS} \|^2
    &\le \tilde{\gamma} \cdot \E \| \theta_{t} - \theta_{t}^{\PS} \|^2 + \frac{(1 + \frac{1}{\tau}) (1 + \zeta)}{n(t) + t_0} \cdot \frac{4 \sigma^2}{\mu^2} + (1 + \tau) (\Delta_{t}^{\PS})^2 \\
    &\le \tilde{\gamma} \cdot \E \| \theta_{t} - \theta_{t}^{\PS} \|^2 + \frac{4 \sigma^2 (1 + \frac{1}{\tau}) (1 + \zeta)}{n_0 \mu^2} t^{-r} + (1 + \tau) (\Delta_{t}^{\PS})^2,
\end{align*}
since $n_0 t^r \le n(t) \le n(t) + t_0$.

Finally, by \cref{lem:unroll}, we have
\begin{align*}
    \E \| \theta_{T+1} - \theta_{T+1}^{\PS} \|^2
    &\le \tilde{\gamma}^{T} \| \theta_{1} - \theta_{1}^{\PS} \|^2  + \frac{4 \sigma^2 (1 + \frac{1}{\tau}) (1 + \zeta)}{n_0 \mu^2} \sum_{t=1}^{T} \tilde{\gamma}^{T-t} t^{-r} + (1 + \tau) \sum_{t=1}^{T} \tilde{\gamma}^{T-t} (\Delta_{t}^{\PS})^2 \\
    &\le \tilde{\gamma}^{T} \| \theta_{1} - \theta_{1}^{\PS} \|^2 + \frac{4 \sigma^2 (1 + \frac{1}{\tau}) (1 + \zeta)}{n_0 \mu^2 (1-\tilde{\gamma})} \left( \tilde{\gamma}^{T (1 - 2^{-\frac{1}{r}})} + 2T^{-r} \right)  + (1 + \tau) \sum_{t=1}^{T} \tilde{\gamma}^{T-t} (\Delta_{t}^{\PS})^2,
\end{align*}
which proves the given statement for $J =  \frac{4 \sigma^2 (1 + \frac{1}{\tau}) (1 + \zeta)}{n_0 \mu^2 (1-\tilde{\gamma})}$.
\end{proof}


\subsection{\texorpdfstring%
{Proof of \cref{thm:rsgdl}}%
{Proof of Proposition 4.4}}
\label{subsec:rsgdl}

Here we prove \cref{thm:rsgdl}, restated below for the sake of readability.

\thmrsgdl*

\begin{proof}
    Using the same step sizes $\eta_{t, j}$ as in \cref{lem:rsgd_lazy}, we start from the result of \cref{lem:rsgd_lazy} with $T \leftarrow t-1$,
    \begin{align*}
        \E \| e_t \|^2
        &\le \tilde{\gamma}^{t-1} \| e_1 \|^2 + J \!\cdot\! \left( \tilde{\gamma}^{\beta (t-1)} + 2(t-1)^{-r} \right) + (1 + \tau) \sum_{s=1}^{t-1} \tilde{\gamma}^{t-s-1} (\Delta_{s}^{\PS})^2,
    \end{align*}
    where we define $e_{t} = \theta_{t} - \theta_{t}^{\PS}$ and $\beta = 1 - 2^{-1/r}$ for simplicity.
    This implies
    \begin{align*}
        \sqrt{\E \| e_t \|^2}
        &\le \tilde{\gamma}^{\frac{t-1}{2}} \| e_1 \| + \sqrt{J} \cdot \tilde{\gamma}^{\frac{\beta (t-1)}{2}} + \sqrt{2J} \cdot (t-1)^{-r/2} + \sqrt{1 + \tau} \left( \sum_{s=1}^{t-1} \tilde{\gamma}^{t-s-1} (\Delta_{s}^{\PS})^2 \right)^{\!\!1/2},
    \end{align*}
    since $\sqrt{a + b + c + d} \le \sqrt{a} + \sqrt{b} + \sqrt{c} + \sqrt{d}$ for $a, b, c, d \ge 0$.
    We can immediately observe that
    \begin{align*}
        \sum_{t=1}^{T} \tilde{\gamma}^{\frac{t-1}{2}} \| e_1 \|
        &\le \frac{\| e_1 \|}{1 - \tilde{\gamma}^{1/2}}, \qquad 
        \sum_{t=1}^{T} \sqrt{J} \cdot \tilde{\gamma}^{\frac{\beta (t-1)}{2}}
        \le \frac{\sqrt{J}}{1 - \tilde{\gamma}^{\beta/2}},
    \end{align*}
    and the last term can be upper-bounded using the inequality
    \begin{align*}
        \sum_{s=1}^{t-1} \tilde{\gamma}^{t-s-1} (\Delta_{s}^{\PS})^2
        &\le \left( \sum_{s=1}^{t-1} \tilde{\gamma}^{\frac{t-s-1}{2}} \Delta_{s}^{\PS} \right)^2,
    \end{align*}
    which implies
    \begin{align*}
        \sum_{t=1}^{T} \left( \sum_{s=1}^{t-1} \tilde{\gamma}^{t-s-1} (\Delta_{s}^{\PS})^2 \right)^{\!\!1/2}
        &\le \sum_{t=1}^{T} \sum_{s=1}^{t-1} \tilde{\gamma}^{\frac{t-s-1}{2}} \Delta_{s}^{\PS} \\
        &= \sum_{s=1}^{T-1} \Delta_{s}^{\PS} \left( \sum_{t=s+1}^{T} \tilde{\gamma}^{\frac{t-s-1}{2}} \right) \\
        &\le \frac{1}{1 - \tilde{\gamma}^{1/2}} \cdot \sum_{s=1}^{T-1} \Delta_{s}^{\PS}.
    \end{align*}
    Therefore (recalling that $e_1 = \theta_1 - \theta_1^{\PS}$) we have  
    \begin{align*}
        \mathrm{Reg}_{T}^{\PS} (\ARSGDL) 
        &\le L \sum_{t=1}^{T} \mathbb{E} \norm{ e_{t} }
        \le L \sum_{t=1}^{T} \sqrt{ \E \| e_{t} \|^2 } \\
        &\le L \left( \frac{\| \theta_{1} - \theta_{1}^{\PS} \|}{1 - \tilde{\gamma}^{1/2}} + \frac{\sqrt{J}}{1 - \tilde{\gamma}^{\beta/2}} + \sqrt{2J} \sum_{t=1}^{T-1} t^{-r/2} + \frac{\sqrt{1 + \tau}}{1 - \tilde{\gamma}^{1/2}} \sum_{t=1}^{T-1} \Delta_{t}^{\PS} \right).
    \end{align*}
    Eliminating all $\gO(1)$ terms, we have
    \begin{align*}
        \mathrm{Reg}_{T}^{\PS} (\ARSGDL)
        &\le \gO \left( \sum_{t=1}^{T-1} t^{-r/2} + \sum_{t=1}^{T-1} \Delta_{t}^{\PS} \right),
    \end{align*}
    as desired.
\end{proof}

\section{Proofs for zeroth-order gradient descent}
\label{sec:f}

As we state in \cref{sec:5}, we make the following mild assumptions throughout \cref{sec:f}: $|\ell(z, \theta)| \le F$ for all $z,\theta$ for some $F < \infty$, and $\Theta$ contains a ball of radius $r>0$ centered at $0$.


\subsection{\texorpdfstring%
{Proof of \cref{thm:zgdtwopoint}}%
{Proof of Theorem 5.2}}
\label{subsec:zgdtwopoint}

Here we prove \cref{thm:zgdtwopoint}, restated below for the sake of readability.

\thmzgdtwopoint*

\begin{proof}
Since we assume $\Theta$ is a bounded (convex) set, we denote the (finite) diameter of $\Theta$ by $D_{\Theta}$, i.e., $D_{\Theta} \coloneqq \sup \{ \| \theta - \theta' \| : \theta, \theta' \in \Theta \}$.

We first prove the strongly convex case, and then prove the convex case.

\paragraph{Strongly convex case.}
Recall that for two-point ZGD, we use
\begin{align*}
    \tilde{g}_{t} &= \big( \mathrm{PR}_{t} (\theta_{t} + \delta u_{t}) - \mathrm{PR}_{t} (\theta_{t} - \delta u_{t}) \big) u_{t} \cdot \frac{d}{2 \delta}.
\end{align*}
Then, we have $\E[\tilde{g}_{t}] = \nabla \widehat{\mathrm{PR}}_t(\theta_t)$ and
\begin{align}
    \E [\norm{\tilde{g}_{t}}^2] &\le d^2 L^2,
    \label{eq:zgd_newgradbound}
\end{align}
since $\mathrm{PR}_{t}$ and $\widehat{\mathrm{PR}}_{t}$ are both $L$-Lipschitz.

Then, $\mu'$-strong convexity of $\widehat{\mathrm{PR}}_{t} (\theta)$ implies
\begin{align}
    \begin{aligned}
    \E [\norm{\theta_{t+1} - \theta_{t}^{u}}^2]
    &= \E [\norm{\Pi_{(1 - \rho) \Theta} (\theta_{t} - \eta_{t} \tilde{g}_{t}) - \theta_{t}^{u}}^2] \\
    &\le \E [\norm{\theta_{t} - \theta_{t}^{u} - \eta_{t} \tilde{g}_{t}}^2] \\
    &= \E [\norm{\theta_{t} - \theta_{t}^{u}}^2 - 2 \eta_{t} \tilde{g}_{t}^{\top} ( \theta_{t} - \theta_{t}^{u} ) + \eta_{t}^{2} \norm{\tilde{g}_{t}}^2] \\
    &= \norm{\theta_{t} - \theta_{t}^{u}}^2 - 2 \eta_{t} \nabla \widehat{\mathrm{PR}}_{t} (\theta_{t})^{\top} ( \theta_{t} - \theta_{t}^{u} ) + \eta_{t}^{2} \E [\norm{\tilde{g}_{t}}^2],
    \end{aligned}
    \label{eq:zgd_basic1same}
\end{align}
for any comparator $\theta_{t}^{u} \in (1 - \rho) \Theta$ which we choose later, since $\E[\tilde{g}_{t}] = \nabla \widehat{\mathrm{PR}}_t(\theta_t)$.
(The expectations here are conditioned on $\theta_{t}$.)

By $\mu'$-strong convexity of $\widehat{\mathrm{PR}}_{t}$, \eqref{eq:zgd_newgradbound}, and \eqref{eq:zgd_basic1same}, we have
\begin{align}
    \begin{aligned}
    \widehat{\mathrm{PR}}_{t} (\theta_{t}) - \widehat{\mathrm{PR}}_{t} (\theta_{t}^{u})
    &\le \nabla \widehat{\mathrm{PR}}_{t} (\theta_{t})^{\top} ( \theta_{t} - \theta_{t}^{u} ) - \frac{\mu'}{2} \lVert \theta_{t} - \theta_{t}^{u} \rVert^2 \\
    &\le \left( \frac{1}{2 \eta_{t}} - \frac{\mu'}{2} \right) \lVert \theta_{t} - \theta_{t}^{u} \rVert^2
    - \frac{1}{2 \eta_{t}} \E [\lVert \theta_{t+1} - \theta_{t}^{u} \rVert^2] + \eta_{t} \frac{d^2 L^2}{2},
    \end{aligned}
    \label{eq:aaadonta}
\end{align}
for any choice of comparators $\theta_{t}^{u}, \theta_{t+1}^{u} \in (1 - \rho) \Theta$.
By triangle inequality, we have
\begin{align*}
    \lVert \theta_{t+1} - \theta_{t}^{u} \rVert
    &\ge \lVert \theta_{t+1} - \theta_{t+1}^{u} \rVert - \lVert \theta_{t}^{u} - \theta_{t+1}^{u} \rVert,
\end{align*}
and since $\lVert \theta_{t+1} - \theta_{t+1}^{u} \rVert \le D_{\Theta}$, we can use
\begin{align*}
    \lVert \theta_{t+1} - \theta_{t}^{u} \rVert^2
    &\ge \lVert \theta_{t+1} - \theta_{t+1}^{u} \rVert^2 - 2(\theta_{t+1} - \theta_{t+1}^{u})^{\top} (\theta_{t}^{u} - \theta_{t+1}^{u}) + \lVert \theta_{t}^{u} - \theta_{t+1}^{u} \rVert^2 \\
    &\ge \lVert \theta_{t+1} - \theta_{t+1}^{u} \rVert^2 - 2 D_{\Theta} \lVert \theta_{t}^{u} - \theta_{t+1}^{u} \rVert.
\end{align*}
Plugging in $\eta_{t} = \frac{1}{\mu' t}$ and noting that $\frac{1}{2 \eta_1} - \frac{\mu'}{2} = 0$, we have
\begin{align*}
    \widehat{\mathrm{PR}}_{t} (\theta_{t}) - \widehat{\mathrm{PR}}_{t} (\theta_{t}^{u}) \le \frac{\mu' (t-1)}{2} \lVert \theta_{t} - \theta_{t}^{u} \rVert^2
    - \frac{\mu' t}{2} \E [\lVert \theta_{t+1} - \theta_{t+1}^{u} \rVert^2] + \mu' D_{\Theta} t \lVert \theta_{t}^{u} - \theta_{t+1}^{u} \rVert + \eta_{t} \frac{d^2 L^2}{2},
\end{align*}
and taking the sum and expectations over the entire history,
\begin{align}
    \begin{aligned}
    \E \left[ \sum_{t=1}^{T} \big( \widehat{\mathrm{PR}}_{t} (\theta_{t}) - \widehat{\mathrm{PR}}_{t} (\theta_{t}^{u}) \big) \right]
    &\le \frac{d^2 L^2}{2 \mu'} \sum_{t=1}^{T} \frac{1}{t} + \mu' D_{\Theta} \sum_{t=1}^{T} t \lVert \theta_{t}^{u} - \theta_{t+1}^{u} \rVert \\
    &\le \frac{d^2 L^2}{2 \mu'} (1 + \log T) + \mu' D_{\Theta} \sum_{t=1}^{T} t \lVert \theta_{t}^{u} - \theta_{t+1}^{u} \rVert.
    \end{aligned}
    \label{eq:zgd_basic2same}
\end{align}

Recall that $\phi_{t}^{+} = \theta_{t} + \delta u_{t}$ and $\widehat{\mathrm{PR}}_{t} (\theta) = \E_{v \sim \Unif(\sB^{d})} [\mathrm{PR}_{t}(\theta + \delta v)]$. 
Since $\mathrm{PR}_{t}$ and $\widehat{\mathrm{PR}}_{t}$ are both $L$-Lipschitz,
we have
\begin{align*}
    \mathrm{PR}_{t} (\phi_{t}^{+}) - \mathrm{PR}_{t} (\theta_{t})
    &\le \delta L, \qquad
    \mathrm{PR}_{t} (\theta_{t}) - \widehat{\mathrm{PR}}_{t} (\theta_{t})
    \le \delta L, \qquad
    - \mathrm{PR}_{t} (\theta_{t}^{u}) + \widehat{\mathrm{PR}}_{t} (\theta_{t}^{u})
    \le \delta L.
\end{align*}
The same arguments hold when we replace $\phi_{t}^{+}$ with $\phi_{t}^{-}$.
Therefore, for any choice of $\phi_{t} \in \{\phi_{t}^{+}, \phi_{t}^{-}\}$, we have
\begin{align*}
    \E \left[ \sum_{t=1}^{T} \big( \mathrm{PR}_{t} (\phi_{t}) - \mathrm{PR}_{t} (\theta_{t}^{u}) \big) \right] 
    \le \frac{d^2 L^2}{2 \mu'} (1 + \log T) + 3 \delta L T + \mu' D_{\Theta} \sum_{t=1}^{T} t \lVert \theta_{t}^{u} - \theta_{t+1}^{u} \rVert.
\end{align*}

Now choose $\theta_{t}^{u} = (1 - \rho) \theta_{t}^{\PO}$ for all $t \in [T]$, where we can check $\theta_{t}^{u} \in (1 - \rho) \Theta$, and $\theta_{T+1}^{u} = \theta_{T}^{u}$.
Then, by convexity of $\mathrm{PR}_{t}$ and $0 \in \Theta$, we have
\begin{align}
    \begin{aligned}
    \mathrm{PR}_{t} (\theta_{t}^{u})
    &\le (1 - \rho) \mathrm{PR}_{t} (\theta_{t}^{\PO}) + \rho \mathrm{PR}_{t} (0) = \mathrm{PR}_{t} (\theta_{t}^{\PO}) + \rho \big( \mathrm{PR}_{t} (0) - \mathrm{PR}_{t} (\theta_{t}^{\PO}) \big) 
    \le \mathrm{PR}_{t} (\theta_{t}^{\PO}) + 2 \rho F.
    \end{aligned}
\end{align}
Therefore, taking the maximum over all possible choices of $\phi_{t} \in \{\phi_{t}^{+}, \phi_{t}^{-}\}$, we have
\begin{align*}
    \Reg_{T}^{\PO} (\AZGDtwo)
    &= \sum_{t=1}^{T}
    \left(
    \max_{\phi_t \in \{\phi_t^+,\phi_t^-\}} \E[\mathrm{PR}_t(\phi_t)]
    - \mathrm{PR}_t(\theta_t^{\PO})
    \right) \\
    &\le \frac{d^2 L^2}{2 \mu'} (1 + \log T) + 3 \delta L T + 2 \rho F T + \mu' D_{\Theta} (1 - \rho) \sum_{t=1}^{T-1} t \Delta_{t}^{\PO},
\end{align*}
where the $t = T$ term is eliminated by choosing $\theta_{T+1}^{u} = \theta_{T}^{u}$.

Setting $\rho = \frac{\delta}{r}$ and $\delta = \frac{d^2 L}{6 \mu' T}$, we have\footnote{We can improve the leading constant a bit by choosing $\delta = \frac{d^2 L^2}{2 \mu'} (3 L + 2 \frac{F}{r})^{-1} \cdot \frac{1 + \log T}{T}$, but we choose the above $\delta$ for simplicity.}
\begin{align*}
    \Reg_{T}^{\PO} (\AZGDtwo) &\le \frac{d^2 L^2}{2 \mu'} \left( 2 + \log T + \frac{2 F}{3Lr} \right) + \mu' D_{\Theta} (1 - \rho) \sum_{t=1}^{T-1} t \Delta_{t}^{\PO} = \gO \left( d^2 \log T + \sum_{t=1}^{T-1} t \Delta_{t}^{\PO} \right),
\end{align*}
and $T \ge \frac{d^2L}{6\mu' r}$ ensures that $\rho = \frac{\delta}{r} \le 1$, which proves the given statement.

\paragraph{Convex case.}

Suppose that $\mathrm{PR}_{t}$ and thus $\widehat{\mathrm{PR}}_{t}$ are convex and Lipschitz.
We use $\eta_{t} \equiv \eta$, which we choose later.
Since this case corresponds to $\mu' = 0$, from \eqref{eq:aaadonta} and the triangle inequality argument, we have
\begin{align*}
    &\widehat{\mathrm{PR}}_{t} (\theta_{t}) - \widehat{\mathrm{PR}}_{t} (\theta_{t}^{u}) \le \frac{1}{2 \eta} \lVert \theta_{t} - \theta_{t}^{u} \rVert^2
    - \frac{1}{2 \eta} \E [\lVert \theta_{t+1} - \theta_{t+1}^{u} \rVert^2] + \frac{D_{\Theta}}{\eta} \lVert \theta_{t}^{u} - \theta_{t+1}^{u} \rVert + \eta \frac{d^2 L^2}{2},
\end{align*}
and therefore
\begin{align*}
    \E \left[ \sum_{t=1}^{T} \big( \widehat{\mathrm{PR}}_{t} (\theta_{t}) - \widehat{\mathrm{PR}}_{t} (\theta_{t}^{u}) \big) \right] 
    &\le \frac{1}{2 \eta} \lVert \theta_{1} - \theta_{1}^{u} \rVert^2 + \frac{\eta d^2 L^2 T}{2} + \frac{D_{\Theta}}{\eta} \sum_{t=1}^{T} \lVert \theta_{t}^{u} - \theta_{t+1}^{u} \rVert \\
    &\le \frac{D_{\Theta}^2}{2 \eta} + \frac{\eta d^2 L^2 T}{2} + \frac{D_{\Theta}}{\eta} \sum_{t=1}^{T} \lVert \theta_{t}^{u} - \theta_{t+1}^{u} \rVert.
\end{align*}
Following the same steps as in the strongly convex case, we have
\begin{align*}
    \Reg_{T}^{\PO} (\AZGDtwo)
    &= \sum_{t=1}^{T}
    \left(
    \max_{\phi_t \in \{\phi_t^+,\phi_t^-\}} \E[\mathrm{PR}_t(\phi_t)]
    - \mathrm{PR}_t(\theta_t^{\PO})
    \right) \\
    &\le \frac{D_{\Theta}^2}{2 \eta} + \frac{\eta d^2 L^2 T}{2} + 3 \delta L T + 2 \rho F T + \frac{D_{\Theta} (1 - \rho)}{\eta} \sum_{t=1}^{T-1} \Delta_{t}^{\PO}.
\end{align*}
If we choose $\rho = \frac{\delta}{r}$, for which we have
\begin{align*}
    \Reg_{T}^{\PO} (\AZGDtwo)
    &\le \frac{D_{\Theta}^2}{2 \eta} + \frac{\eta d^2 L^2 T}{2} + \left( 3 L + 2 \frac{F}{r} \right) \delta T + \frac{D_{\Theta}}{\eta} \left( 1 - \frac{\delta}{r} \right) \sum_{t=1}^{T-1} \Delta_{t}^{\PO},
\end{align*}
and then choose\footnote{We can improve the leading constant a bit by choosing $(3 L + 2 \frac{F}{r})^{-1} \frac{dD_{\Theta}L}{2\sqrt{T}}$, but we choose the above $\delta$ for  simplicity.}
\begin{align*}
    \eta = \frac{D_{\Theta}}{d L \sqrt{T}},
    \qquad
    \delta = \frac{dD_{\Theta}}{6\sqrt{T}},
\end{align*}
then we have
\begin{align*}
    \Reg_{T}^{\PO} (\AZGDtwo)
    &\le \frac{3 dD_{\Theta}L\sqrt{T}}{2} + \frac{dD_{\Theta}F\sqrt{T}}{3r} + d L \sqrt{T} \left( 1 - \frac{dD_{\Theta}}{6r\sqrt{T}} \right) \sum_{t=1}^{T-1} \Delta_{t}^{\PO} = \gO \left( d \sqrt{T} \cdot \left( 1 + \sum_{t=1}^{T-1} \Delta_{t}^{\PO} \right) \right),
\end{align*}
and $T \ge \frac{d^2 D_{\Theta}^2}{36 r^2}$ ensures that $\rho = \frac{\delta}{r} \le 1$, which proves the given statement.
\end{proof}


\subsection{\texorpdfstring%
{Proof of \cref{thm:zgd}}%
{Proof of Theorem 5.3}}
\label{subsec:zgd}

Here we prove \cref{thm:zgd}, restated below for the sake of readability.

\thmzgd*

\begin{proof}

We first prove the strongly convex case, and then prove the convex case.

\paragraph{Strongly convex case.}
For any comparator $\theta_{t}^{u} \in (1 - \rho) \Theta$ (which we choose later), $\mu'$-strong convexity of $\widehat{\mathrm{PR}}_{t} (\theta)$ implies
\begin{align}
    \begin{aligned}
    \E [\norm{\theta_{t+1} - \theta_{t}^{u}}^2]
    &= \E [\norm{\Pi_{(1 - \rho) \Theta} (\theta_{t} - \eta_{t} g_{t}) - \theta_{t}^{u}}^2] \\
    &\le \E [\norm{\theta_{t} - \theta_{t}^{u} - \eta_{t} g_{t}}^2] \\
    &= \E [\norm{\theta_{t} - \theta_{t}^{u}}^2 - 2 \eta_{t} g_{t}^{\top} ( \theta_{t} - \theta_{t}^{u} ) + \eta_{t}^{2} \norm{g_{t}}^2] \\
    &= \norm{\theta_{t} - \theta_{t}^{u}}^2 - 2 \eta_{t} \nabla \widehat{\mathrm{PR}}_{t} (\theta_{t})^{\top} ( \theta_{t} - \theta_{t}^{u} ) + \eta_{t}^{2} \E [\norm{g_{t}}^2],
    \end{aligned}
    \label{eq:zgd_basic1}
\end{align}
since $\E[g_{t}] = \nabla \widehat{\mathrm{PR}}_t(\theta_t)$.
Also, since $g_{t} = \frac{d}{\delta} \cdot \ell (Z_{t}, \theta_{t} + \delta u_{t}) u_{t}$ with $|\ell(z, \theta)| \le F$ and $\lVert u_{t} \rVert = 1$,
\begin{align*}
    \E [\norm{g_{t}}^2] &\le \frac{d^2 F^2}{\delta^2}.
\end{align*}
By $\mu'$-strong convexity of $\widehat{\mathrm{PR}}_{t}$ and \eqref{eq:zgd_basic1}, we have
\begin{align}
    \begin{aligned}
    \widehat{\mathrm{PR}}_{t} (\theta_{t}) - \widehat{\mathrm{PR}}_{t} (\theta_{t}^{u})
    &\le \nabla \widehat{\mathrm{PR}}_{t} (\theta_{t})^{\top} ( \theta_{t} - \theta_{t}^{u} ) - \frac{\mu'}{2} \lVert \theta_{t} - \theta_{t}^{u} \rVert^2 \\
    &\le \left( \frac{1}{2 \eta_{t}} - \frac{\mu'}{2} \right) \lVert \theta_{t} - \theta_{t}^{u} \rVert^2
    - \frac{1}{2 \eta_{t}} \E [\lVert \theta_{t+1} - \theta_{t}^{u} \rVert^2] + \eta_{t} \frac{d^2 F^2}{2 \delta^2},
    \end{aligned}
    \label{eq:zgdonestep}
\end{align}
for any choice of comparators $\theta_{t}^{u}, \theta_{t+1}^{u} \in (1 - \rho) \Theta$.
By triangle inequality, we have
\begin{align*}
    \lVert \theta_{t+1} - \theta_{t}^{u} \rVert
    &\ge \lVert \theta_{t+1} - \theta_{t+1}^{u} \rVert - \lVert \theta_{t}^{u} - \theta_{t+1}^{u} \rVert,
\end{align*}
and since $\lVert \theta_{t+1} - \theta_{t+1}^{u} \rVert \le D_{\Theta}$, we can use
\begin{align*}
    \lVert \theta_{t+1} - \theta_{t}^{u} \rVert^2
    &\ge \lVert \theta_{t+1} - \theta_{t+1}^{u} \rVert^2 - 2(\theta_{t+1} - \theta_{t+1}^{u})^{\top} (\theta_{t}^{u} - \theta_{t+1}^{u}) + \lVert \theta_{t}^{u} - \theta_{t+1}^{u} \rVert^2 \\
    &\ge \lVert \theta_{t+1} - \theta_{t+1}^{u} \rVert^2 - 2 D_{\Theta} \lVert \theta_{t}^{u} - \theta_{t+1}^{u} \rVert.
\end{align*}
Plugging in $\eta_{t} = \frac{1}{\mu' t}$ and noting that $\frac{1}{2 \eta_1} - \frac{\mu'}{2} = 0$, we have
\begin{align*}
    &\widehat{\mathrm{PR}}_{t} (\theta_{t}) - \widehat{\mathrm{PR}}_{t} (\theta_{t}^{u}) \le \frac{\mu' (t-1)}{2} \lVert \theta_{t} - \theta_{t}^{u} \rVert^2
    - \frac{\mu' t}{2} \E [\lVert \theta_{t+1} - \theta_{t+1}^{u} \rVert^2] + \mu' D_{\Theta} t \lVert \theta_{t}^{u} - \theta_{t+1}^{u} \rVert + \eta_{t} \frac{d^2 F^2}{2 \delta^2},
\end{align*}
and taking the sum and expectations over the entire history,
\begin{align}
    \begin{aligned}
    \E \left[ \sum_{t=1}^{T} \big( \widehat{\mathrm{PR}}_{t} (\theta_{t}) - \widehat{\mathrm{PR}}_{t} (\theta_{t}^{u}) \big) \right] 
    &\le \frac{d^2 F^2}{2 \delta^2 \mu'} \sum_{t=1}^{T} \frac{1}{t} + \mu' D_{\Theta} \sum_{t=1}^{T} t \lVert \theta_{t}^{u} - \theta_{t+1}^{u} \rVert \\
    &\le \frac{d^2 F^2}{2 \delta^2 \mu'} (1 + \log T) + \mu' D_{\Theta} \sum_{t=1}^{T} t \lVert \theta_{t}^{u} - \theta_{t+1}^{u} \rVert.
    \end{aligned}
    \label{eq:zgd_basic2}
\end{align}

Recall that $\phi_{t} = \theta_{t} + \delta u_{t}$ and $\widehat{\mathrm{PR}}_{t} (\theta) = \E_{v \sim \Unif(\sB^{d})} [\mathrm{PR}_{t}(\theta + \delta v)]$. 
Since $\mathrm{PR}_{t}$ and $\widehat{\mathrm{PR}}_{t}$ are both $L$-Lipschitz, we have
\begin{align*}
    \mathrm{PR}_{t} (\phi_{t}) - \mathrm{PR}_{t} (\theta_{t})
    &\le \delta L, \\
    \mathrm{PR}_{t} (\theta_{t}) - \widehat{\mathrm{PR}}_{t} (\theta_{t})
    &\le \delta L, \\
    - \mathrm{PR}_{t} (\theta_{t}^{u}) + \widehat{\mathrm{PR}}_{t} (\theta_{t}^{u})
    &\le \delta L,
\end{align*}
and therefore
\begin{align*}
    \E \left[ \sum_{t=1}^{T} \big( \mathrm{PR}_{t} (\phi_{t}) - \mathrm{PR}_{t} (\theta_{t}^{u}) \big) \right] 
    &\le \frac{d^2 F^2}{2 \delta^2 \mu'} (1 + \log T) + 3 \delta L T + \mu' D_{\Theta} \sum_{t=1}^{T} t \lVert \theta_{t}^{u} - \theta_{t+1}^{u} \rVert.
\end{align*}

Now we choose $\theta_{t}^{u} = (1 - \rho) \theta_{t}^{\PO}$ for all $t \in [T]$, where we can check $\theta_{t}^{u} \in (1 - \rho) \Theta$, and $\theta_{T+1}^{u} = \theta_{T}^{u}$.
Then, by convexity of $\mathrm{PR}_{t}$ and $0 \in \Theta$, we have
\begin{align}
    \begin{aligned}
    \mathrm{PR}_{t} (\theta_{t}^{u})
    &\le (1 - \rho) \mathrm{PR}_{t} (\theta_{t}^{\PO}) + \rho \mathrm{PR}_{t} (0) \\
    &= \mathrm{PR}_{t} (\theta_{t}^{\PO}) + \rho \big( \mathrm{PR}_{t} (0) - \mathrm{PR}_{t} (\theta_{t}^{\PO}) \big) \\
    &\le \mathrm{PR}_{t} (\theta_{t}^{\PO}) + 2 \rho F.
    \end{aligned}
    \label{eq:zgd_basic3}
\end{align}
Therefore we have
\begin{align*}
    \Reg_{T}^{\PO} (\AZGD) &= \E \left[ \sum_{t=1}^{T} \big( \mathrm{PR}_{t} (\phi_{t}) - \mathrm{PR}_{t} (\theta_{t}^{\PO}) \big) \right] \\
    &\le \frac{d^2 F^2}{2 \delta^2 \mu'} (1 + \log T) + 3 \delta L T + 2 \rho F T + \mu' D_{\Theta} (1 - \rho) \sum_{t=1}^{T-1} t \Delta_{t}^{\PO}.
\end{align*}
If we choose $\rho = \frac{\delta}{r}$ and
\begin{align*}
    \delta = \left( \frac{d^2 F^2}{2 \mu'} (1 + \log T) \right)^{\frac{1}{3}} \cdot \left( \left( 3 L + \frac{2 F}{r} \right) T \right)^{-\frac{1}{3}},
\end{align*}
we have
\begin{align*}
    \Reg_{T}^{\PO} (\AZGD)
    &\le 2 \left( \frac{d^2 F^2}{2 \mu'} (1 + \log T) \right)^{\frac{1}{3}} \cdot \left( \left( 3 L + \frac{2 F}{r} \right) T \right)^{\frac{2}{3}} + \mu' D_{\Theta} \sum_{t=1}^{T-1} t \Delta_{t}^{\PO} \\
    &= \gO \left( d^{\frac{2}{3}} T^{\frac{2}{3}} (\log T)^{\frac{1}{3}} + \sum_{t=1}^{T-1} t \Delta_{t}^{\PO} \right),
\end{align*}
and $\frac{T}{1 + \log T} \ge \frac{d^2 F^2}{2 \mu' r^3 \left( 3L + \frac{2F}{r} \right)}$ ensures that $\rho = \frac{\delta}{r} \le 1$, which proves the given statement.


\paragraph{Convex case.}

For the case when $\widehat{\mathrm{PR}}_{t}$ is convex, we use constant $\eta_{t} \equiv \eta$ which we choose later.
Since this case corresponds to $\mu' = 0$, from \eqref{eq:zgdonestep} and the triangle inequality argument, we have
\begin{align*}
    &\widehat{\mathrm{PR}}_{t} (\theta_{t}) - \widehat{\mathrm{PR}}_{t} (\theta_{t}^{u}) \le \frac{1}{2 \eta} \lVert \theta_{t} - \theta_{t}^{u} \rVert^2
    - \frac{1}{2 \eta} \E [\lVert \theta_{t+1} - \theta_{t+1}^{u} \rVert^2] + \frac{D_{\Theta}}{\eta} \lVert \theta_{t}^{u} - \theta_{t+1}^{u} \rVert + \eta \frac{d^2 F^2}{2 \delta^2},
\end{align*}
and therefore
\begin{align*}
    \E \left[ \sum_{t=1}^{T} \big( \widehat{\mathrm{PR}}_{t} (\theta_{t}) - \widehat{\mathrm{PR}}_{t} (\theta_{t}^{u}) \big) \right]
    &\le \frac{1}{2 \eta} \lVert \theta_{1} - \theta_{1}^{u} \rVert^2 + \frac{\eta d^2 F^2 T}{2 \delta^2} + \frac{D_{\Theta}}{\eta} \sum_{t=1}^{T} \lVert \theta_{t}^{u} - \theta_{t+1}^{u} \rVert \\
    &\le \frac{D_{\Theta}^2}{2 \eta} + \frac{\eta d^2 F^2 T}{2 \delta^2} + \frac{D_{\Theta}}{\eta} \sum_{t=1}^{T} \lVert \theta_{t}^{u} - \theta_{t+1}^{u} \rVert.
\end{align*}
Following the same steps as in the strongly convex case, we have
\begin{align*}
    \Reg_{T}^{\PO} (\AZGD) &= \E \left[ \sum_{t=1}^{T} \big( \mathrm{PR}_{t} (\phi_{t}) - \mathrm{PR}_{t} (\theta_{t}^{\PO}) \big) \right] \\
    &\le \frac{D_{\Theta}^2}{2 \eta} + \frac{\eta d^2 F^2 T}{2 \delta^2} + 3 \delta L T + 2 \rho F T + \frac{D_{\Theta} (1 - \rho)}{\eta} \sum_{t=1}^{T-1} \Delta_{t}^{\PO}.
\end{align*}
If we choose $\rho = \frac{\delta}{r}$, for which we have
\begin{align*}
    \Reg_{T}^{\PO} (\AZGD)
    &\le \frac{D_{\Theta}^2}{2 \eta} + \frac{\eta d^2 F^2 T}{2 \delta^2} + \left( 3 L + 2 \frac{F}{r} \right) \delta T + \frac{D_{\Theta}}{\eta} \left( 1 - \frac{\delta}{r} \right) \sum_{t=1}^{T-1} \Delta_{t}^{\PO},
\end{align*}
and then choose
\begin{align*}
    \eta =
    \sqrt{\frac{D_{\Theta}^3}{dF} \cdot \frac{1}{3L + \frac{2F}{r}}}\, T^{-3/4},
    \qquad
    \delta =
    \sqrt{\frac{D_{\Theta}dF}{3L + \frac{2F}{r}}}\, T^{-1/4},
\end{align*}
so that
\begin{align*}
    &\frac{D_{\Theta}^2}{2 \eta}
    = \frac{\eta d^2 F^2 T}{2 \delta^2}
    = \frac{1}{2} \sqrt{D_{\Theta}dF \left( 3L + \frac{2F}{r} \right)}\, T^{3/4}, \\
    &\left( 3 L + 2 \frac{F}{r} \right) \delta T
    = \sqrt{D_{\Theta}dF \left( 3L + \frac{2F}{r} \right)}\, T^{3/4},
\end{align*}
then we have
\begin{align*}
    \Reg_{T}^{\PO} (\AZGD) &\le 2 \sqrt{D_{\Theta}dF \left( 3L + \frac{2F}{r} \right)}\, T^{3/4} + \sqrt{\frac{dF}{D_{\Theta}} \left( 3L + \frac{2F}{r} \right)}\, T^{3/4} \left( 1 - \frac{\delta}{r} \right) \sum_{t=1}^{T-1} \Delta_{t}^{\PO} \\
    &= \gO \left( d^{\frac12} T^{\frac{3}{4}} \cdot \left( 1 + \sum_{t=1}^{T-1} \Delta_{t}^{\PO} \right) \right),
\end{align*}
and $T \ge (\frac{D_{\Theta}dF}{r^2 \left( 3L + \frac{2F}{r} \right)})^2$ ensures that $\rho = \frac{\delta}{r} \le 1$, which proves the given statement.
\end{proof}

\clearpage{}
\section{Experimental details and further results}
\label{sec:exps}

\subsection{Credit scoring experiment}
\label{sec:exps-credit}

\paragraph{Setup.}
Here we provide a more detailed description of the credit scoring experiment summarized in \Cref{sec:6}.
We minimize the regularized logistic loss,
\begin{align*}
    \ell(z, \theta)
    = \log \left( 1 + e^{-\theta^\top z} \right)
    + \frac{\lambda}{2} \|\theta\|^2
\end{align*}
over the Euclidean ball $\Theta = \{\theta : \|\theta\|_2 \le 1\}$ with $\lambda = 1$, where $z = y x$, where $y\in \{-1, +1\}$, for features $x$ and labels $y$.
The environment is built from the processed credit data used in the actionable-recourse experiments, available as the \texttt{credit} dataset from \url{https://github.com/ustunb/actionable-recourse}.
We first standardize each of the $17$ credit features to have zero mean and unit variance. 
The Gaussian base distribution is also fitted to $1500$ signed feature vectors $xy$ randomly sampled without replacement, from which we compute the empirical mean $m$ and covariance $\Sigma$.
The first six feature coordinates, corresponding to marital-status and age indicators, are treated as non-modifiable.
The remaining eleven coordinates, corresponding to education, balance, payment, spending, and overdue-history features, are treated as modifiable.
This modifiable features can be altered strategically.

As we discuss in \Cref{sec:6}, the performative component can be written as
\begin{align*}
    \gD(\theta) = \gN(A\theta+m, \Sigma),
\end{align*}
where $\Sigma$ is the empirical covariance and $A = a \operatorname{diag} (p_{\mathrm{mod}})$ with $a = 0.1$.
The vector $p_{\mathrm{mod}}$ is zero on the non-modifiable coordinates and one on the modifiable coordinates, which implies that $\|A\|_{\mathrm{op}} = a = 0.1$.
The exogenous component is normal with the same covariance matrix $\Sigma$ as the performative component,
\begin{align*}
    P_t = \gN(m_t, \Sigma).
\end{align*}
For the exogenous sequence, we consider both an i.i.d. randomly varying sequence for $m_t$ in the unit Euclidean ball and the stationary sequence $m_t\equiv m_0$, where $m_0$ is drawn once from the same bounded set.
The experiments use $d = 17$, $T_{\mathrm{stab}} = 2000$ for the algorithms $\ARRM, \ARGD$, and $\ARSGDG$, and $T_{\mathrm{lazy}} = 1000$ for $\ARSGDL$.
The $\alpha_t$ grid consists of $\alpha_t = 1$ and polynomial schedules $\alpha_t = \alpha_0 t^{-b}$ with $\alpha_0 = 1$ and $b \in \{0.25, 0.5, 1.0, 2.0\}$.
For algorithms $\ARSGDG$ and $\ARSGDL$ with stochasticity, we average over $10$ runs.

The logistic loss is $\lambda$-strongly convex in $\theta$, and 
if $\| z \| \le R_z$ and $\| \theta \| \le R_{\theta}$ for some finite $R_z, R_{\theta}$, then $\ell$ is $\lambda_z + \frac{R_z^2}{4}$-smooth in $z$ and $\lambda_z + \frac{R_z R_{\theta}}{4}$-smooth in $\theta$.
Since Gaussian distributions technically have unbounded support, when plotting the theoretical upper bounds later we use slightly conservative constants as a high-probability surrogate for both the smoothness constants above and the Lipschitz constants.

We can also compute
\begin{align*}
    \gW_1\bigl( \gD(\theta), \gD(\theta') \bigr)
    = \|A(\theta - \theta')\|
    \le 0.1 \|\theta - \theta'\|.
\end{align*}
Consequently, the sensitivity condition is satisfied.

\paragraph{Computing the stable points.}
To approximate $\theta_t^\PS$ empirically, we solve the fixed-point problem
\begin{align*}
    \theta_t^\PS
    = \argmin_{\theta' \in \Theta}
    \E_{Z \sim \gD_t (\theta_t^\PS)} [\ell(Z, \theta')]
\end{align*}
by a bilevel fixed-point iteration heuristic.
At each outer step, the inner fixed-distribution risk minimization is approximated by projected gradient descent using common random numbers.
The experiment uses $20$ outer iterations, $60$ inner gradient steps, and inner step size $0.2$ for the fixed point iteration to compute $\theta_t^\PS$, and use $512$ Monte Carlo samples to estimate the gradient.
Risk evaluations for plotting use $16384$ common random samples, and these are shared across algorithms and comparator evaluations for a consistent comparison.

\begin{figure}
    \centering
    \includegraphics[width=0.95\linewidth]{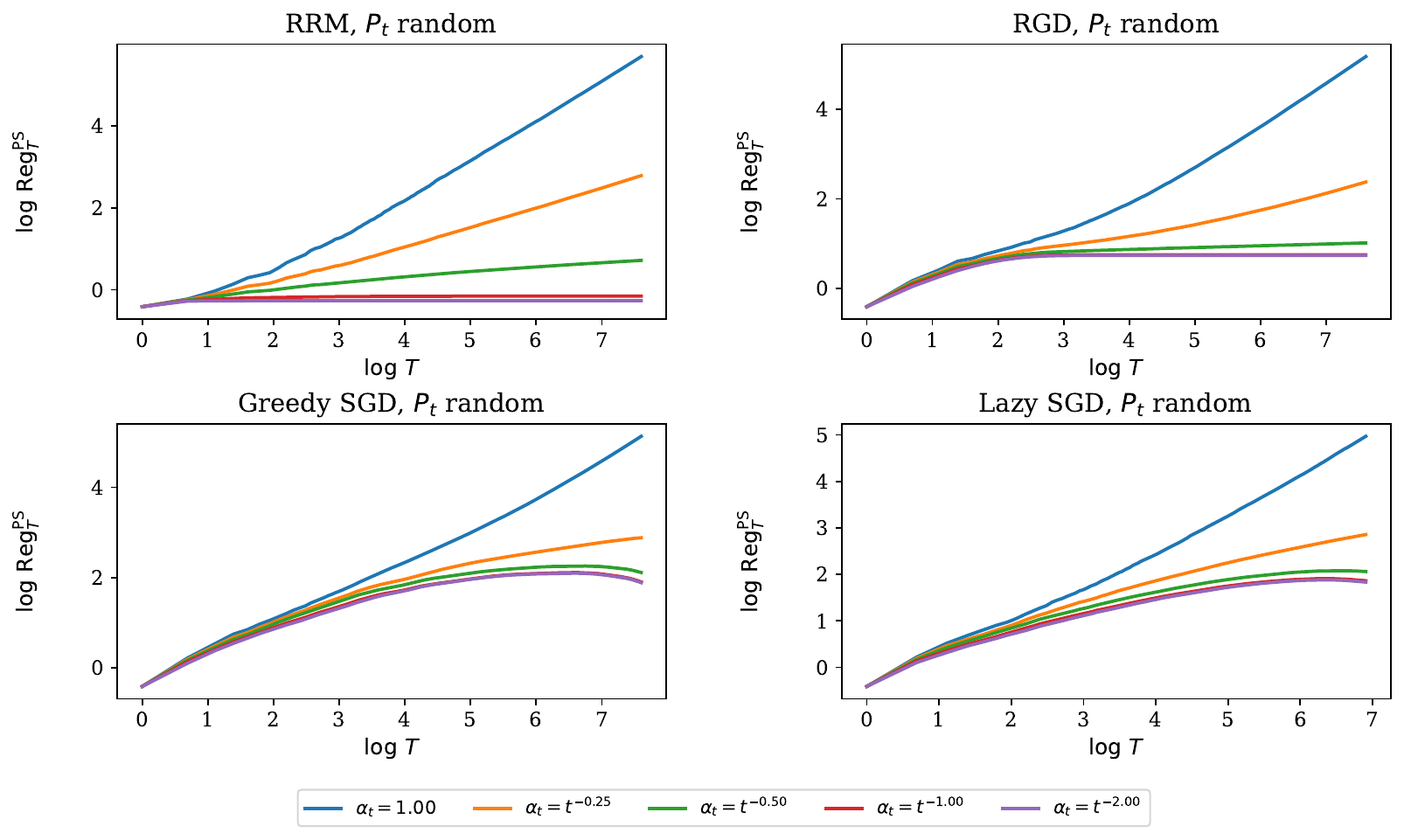}
    \caption{Performative stability regret of $\ARRM$, $\ARGD$, $\ARSGDG$, and $\ARSGDL$ with $r = 1$ in the credit scoring experiment, for varying $\alpha_t$. We consider a randomly varying exogenous component $P_t$.
    }
    \label{fig:2}
\end{figure}
\begin{figure}
    \centering
    \includegraphics[width=0.95\linewidth]{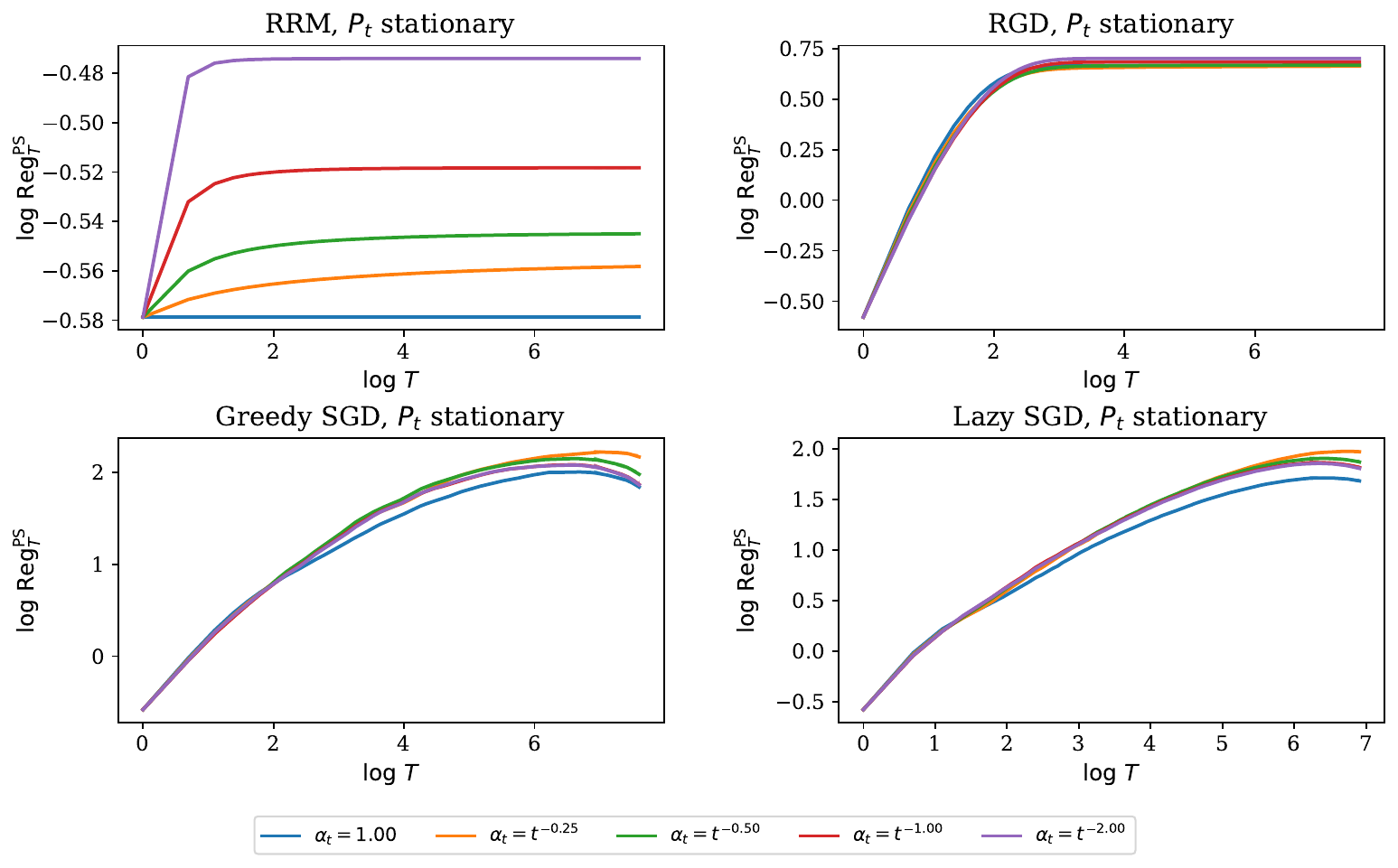}
    \caption{Performative stability regret of $\ARRM$, $\ARGD$, $\ARSGDG$, and $\ARSGDL$ with $r = 1$ in the credit scoring experiment, for varying $\alpha_t$. We consider a stationary exogenous component $P_t$.}
    \label{fig:3}
\end{figure}

\paragraph{Results.}
\Cref{fig:2} shows that for the randomly varying exogenous shift, $\ARRM$ and $\ARGD$ have a clear ordering with respect to the $\alpha_t$ schedule, with $\alpha_t\propto t^{-1.0}$ producing smaller regret than $\alpha_t\propto t^{-0.5}$, which matches the trend we can expect from the path-length dependence in \Cref{thm:rrm,thm:rgd} and our discussions in \Cref{sec:6}.
The stochastic algorithms have less separation among the decaying schedules, yet we can observe that the constant schedule remains visibly worse with a slope near $1$ (i.e., linear regret) for $\ARSGDG$ and $\ARSGDL$ as well.

\Cref{fig:3} shows that for the stationary case, we have the opposite trend compared to that of the randomly varying case in general, with the constant schedule performing slightly better than the decaying schedules.
The distinction is clear for $\ARRM$ and less prominent for gradient-based methods like $\ARGD$, $\ARSGDG$, and $\ARSGDL$.
Note that because the reference for stability regret $\theta_t^{\PS}$ may be different from the minimizer of the performative risk $\theta_t^{\PO}$, it is possible for the regret to decrease when the iterates $\theta_t$ coincidentally get close to $\theta_t^{\PO}$.

\Cref{fig:4} compares the empirical stability regret of $\ARRM$ and $\ARGD$ to the finite upper bounds we derive in \Cref{thm:rrm,thm:rgd}.
In the randomly varying case, \Cref{fig:4} again shows that $\alpha_t\propto t^{-1.0}$ produces smaller regret than $\alpha_t\propto t^{-0.5}$ for both algorithms, matching the path-length dependence in \Cref{thm:rrm,thm:rgd}.
We can also observe that, even though Gaussian distributions are unbounded, the finite upper bounds are still qualitatively correct in their ordering and level, which suggests that the conservative high-probability envelope is a reasonable surrogate for the Lipschitz constants in this case.
Moreover, for decaying $\alpha_t$, the slopes of the log-log plots near large $T$ of the empirical regret curves are slightly smaller than the rates predicted by the upper bounds, which leaves room for both an improvement in the path-length dependence of the upper bounds or the existence of a more carefully constructed adversarial lower bound, which we leave for future investigation.

\begin{figure}[t]
    \centering
    \includegraphics[width=0.9\linewidth]{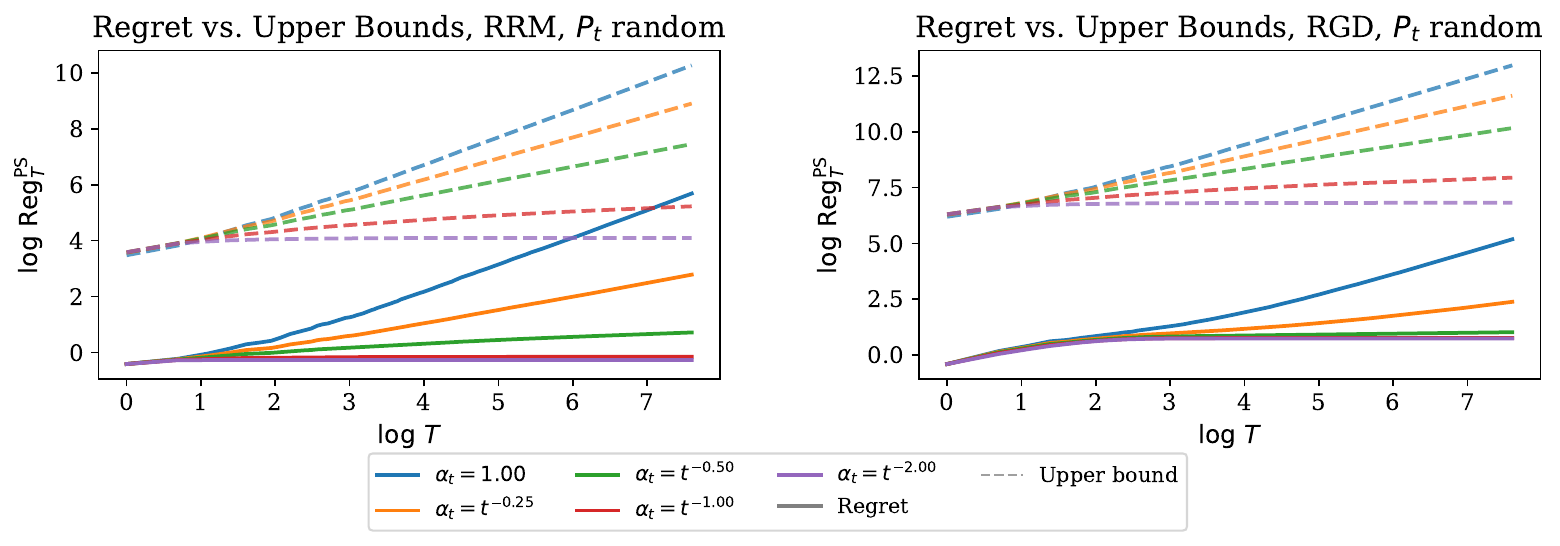}
    \caption{Performative stability regret vs. theoretical upper bounds of $\ARRM$ (left) and $\ARGD$ (right) in the credit scoring experiment, for varying $\alpha_t$. We consider a randomly varying exogenous component $P_t$.}
    \label{fig:4}
\end{figure}


\subsection{Further results on synthetic data}
\label{sec:exps-quad}

\paragraph{Setup.}
We also report an experiment on synthetic Gaussian data.
We minimize the squared loss $\ell(z, \theta) = \frac12 \|z-\theta\|^2$ over $\Theta = [-B, B]^d$ with $B=5$.
We use a partially performative setting similar to that of \Cref{sec:exps-credit}:
\begin{align*}
    Z_t
    &\sim \gD_t(\theta_t)
    = (1-\alpha_t) \gD(\theta_t) + \alpha_t P_t,\\
    \gD(\theta)
    &= \gN(A\theta+m, \Sigma),
    \qquad
    P_t = \gN(m_t, \Sigma_P).
\end{align*}
We use $\|A\|_{\mathrm{op}} = \epsilon = 0.3$, $\Sigma = \sigma^2I$, $\Sigma_P = \sigma_P^2 I$, and $\sigma = \sigma_P = 0.5$.
The vector $m$ is sampled once from $[-1,1]^d$, and the matrix $A$ is sampled once and rescaled to have operator norm $\epsilon$.
The $\alpha_t$ grid consists of the constant schedule $\alpha_t = 1$ and polynomial decays $\alpha_t \propto t^{-b}$ with $b\in\{0.5, 0.75, 1.0, 1.5, 2.0\}$.
For the exogenous sequence, we report both a shared i.i.d. randomly varying sequence in the unit Euclidean ball and the stationary sequence $m_t\equiv m_0$, where $m_0$ is drawn once from the same bounded construction.
The experiments use $d \in \{2, 10\}$, $T_{\mathrm{stab}} = 100000$ for the algorithms $\ARRM, \ARGD$, and $\ARSGDG$ and $T_{\mathrm{lazy}} = 1000$ for $\ARSGDL$.

The quadratic loss is $1$-strongly convex in $\theta$ and $1$-smooth in both $\theta$ and $z$.
Since $\gD(\theta)$ and $\gD(\theta')$ have the same covariance for any $\theta, \theta'$, we have
\begin{align*}
    \gW_1 \bigl( \gD(\theta), \gD(\theta') \bigr)
    &=
    \gW_1 \bigl( \gN(A \theta + m, \Sigma), \gN(A \theta + m', \Sigma) \bigr) \\
    &= \| A (\theta - \theta') \|
    \le
    \epsilon \| \theta - \theta' \|.
\end{align*}
Thus $\gD$ is $\epsilon$-sensitive, and the partially performative map has sensitivity at most $(1-\alpha_t) \epsilon$.
Since $\epsilon = 0.3 < 1 = \mu/L_z$, the stability condition is satisfied.

As in the previous experiment, because Gaussian distributions have unbounded support, we use slightly conservative constants as a high-probability surrogate for the Lipschitz constants that are only used to compute the theoretical upper bounds.

\paragraph{Computing the stable points.}
For the stable point, the minimizer of the squared loss under a fixed distribution is its mean.
Hence $\theta_t^\PS$ solves
\begin{align*}
    \theta_t^\PS
    &= (1-\alpha_t) (A \theta_t^\PS + m) + \alpha_t m_t,
\end{align*}
and therefore can be evaluated in closed form as the solution to the linear system:
\begin{align*}
    \theta_t^\PS
    = \bigl( I - (1-\alpha_t) A \bigr)^{-1}
    \bigl( (1-\alpha_t)m + \alpha_t m_t \bigr).
\end{align*}

\paragraph{Results.}
\Cref{fig:8} compares the empirical stability regret of $\ARRM$ and $\ARGD$ to the finite upper bounds under randomly varying shifts for $\alpha_t \propto t^{-0.5}$ and $\alpha_t \propto t^{-1.0}$ for the Gaussian experiment with dimension $d = 10$.
We have similar trends as in \Cref{fig:4}, and the difference between slopes of the empirical and upper bound curves (indicating the looseness of the upper bounds) are more pronounced in this case.

\Cref{fig:9} shows the optimization iterates in the experiment in two dimensions.
For the randomly varying sequence with $\alpha_t \propto t^{-1.0}$, we can visually observe that the algorithm iterates $\theta_t$ track the moving target $\theta_t^{\PS}$ with a slight lag.
The log-distance panels on the right show that this tracking error contracts over time but is also repeatedly perturbed by the movement of the reference.

\begin{figure}[t]
    \centering
    \includegraphics[width=0.91\linewidth]{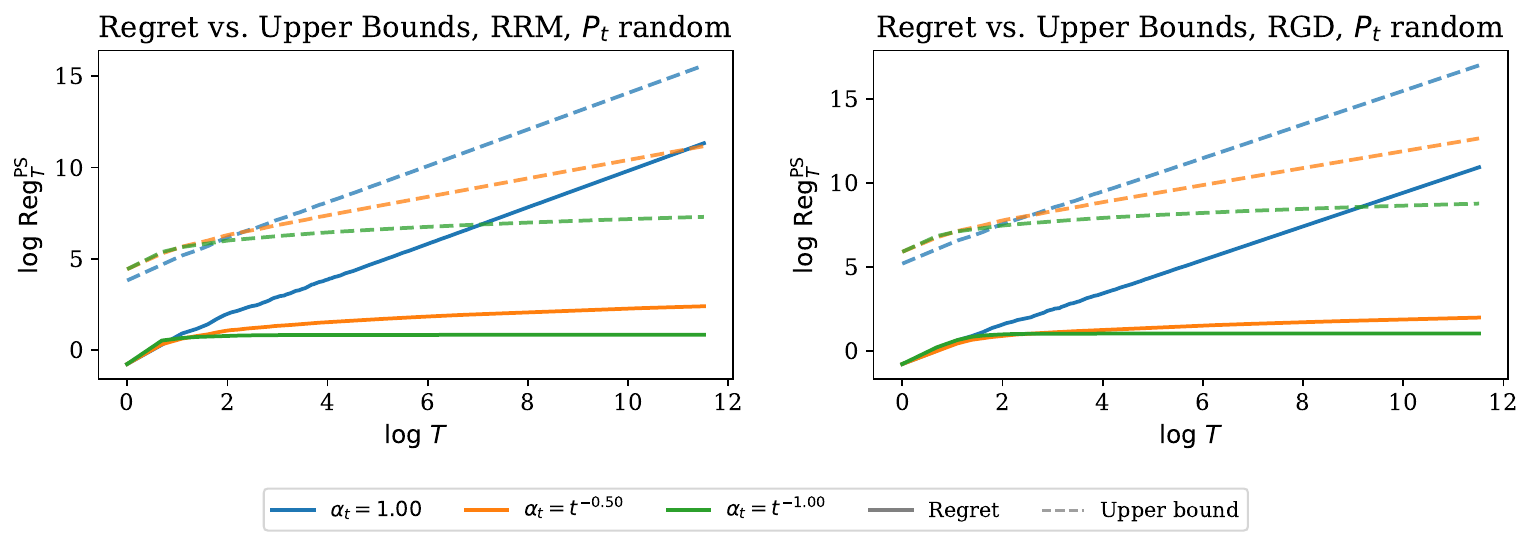}
    \caption{Performative stability regret vs. theoretical upper bounds of $\ARRM$ (left) and $\ARGD$ (right) in the quadratic-Gaussian experiment, for varying $\alpha_t$. We consider a randomly varying exogenous component $P_t$.}
    \label{fig:8}
\end{figure}

\begin{figure}[t]
    \centering
    \includegraphics[height=0.88\textheight]{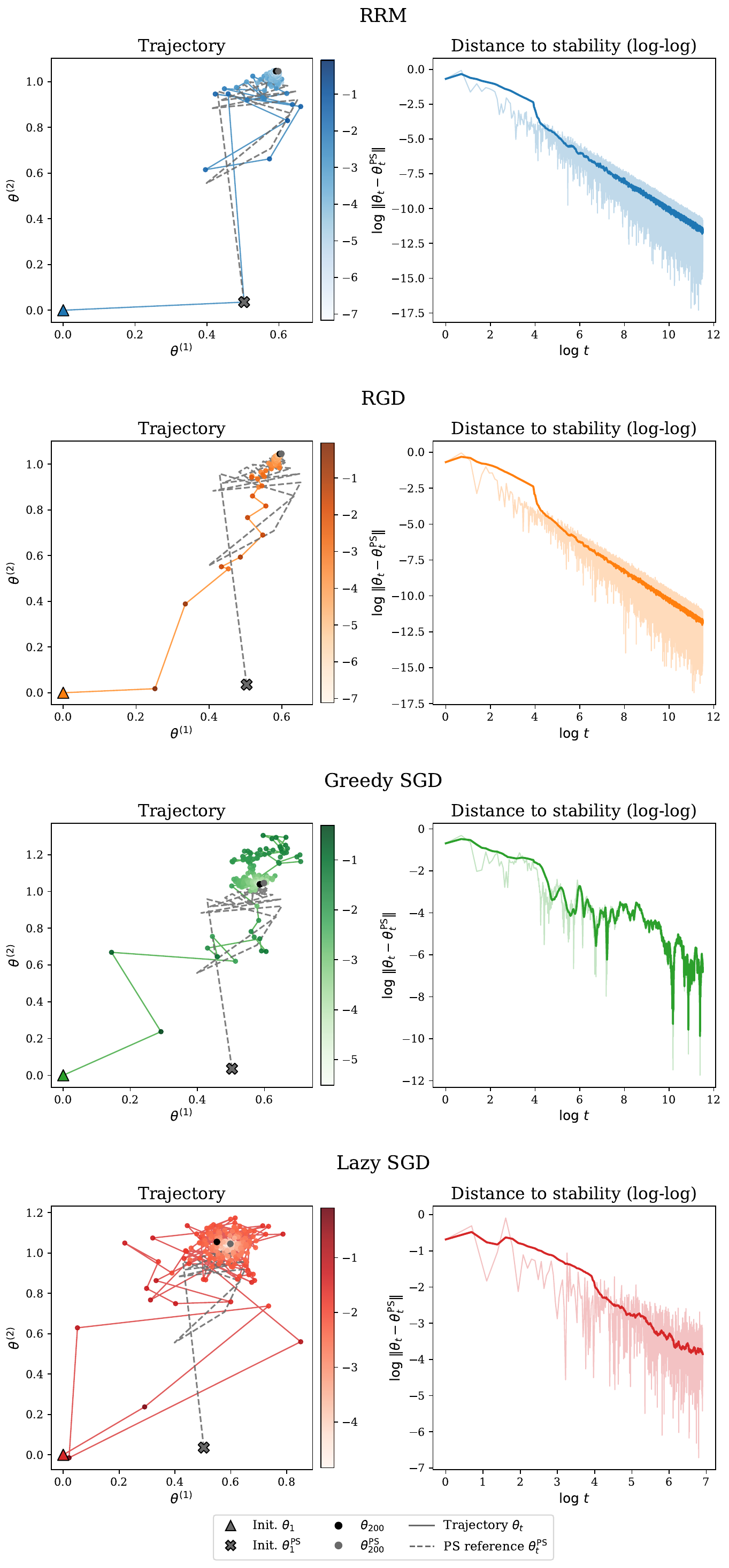}
    \caption{Representative two-dimensional iterate vs. stability trajectories of $\ARRM$, $\ARGD$, $\ARSGDG$, and $\ARSGDL$ with $r = 1$ in the 2D quadratic-Gaussian experiment, for $\alpha_t \propto t^{-1.0}$. We consider a randomly varying exogenous component $P_t$.
    The left subfigures show the actual movement of $\theta_t$ and $\theta_t^{\PS}$ in the 2D space for the initial $200$ steps, and the right subfigures show the log-log plots of the distance $\| \theta_t - \theta_t^{\PS} \|$.
    We consider both the raw trajectories and the moving average of the last $50$ iterations in the right log-log plot.
    }
    \label{fig:9}
\end{figure}

\end{document}